%% file: main.tex
\documentclass{article}

 \PassOptionsToPackage{numbers, compress}{natbib}

\usepackage[final]{neurips_2021}

\usepackage{url}            % simple URL typesetting
\usepackage{booktabs}       % professional-quality tables

\usepackage{nicefrac}       % compact symbols for 1/2, etc.
\usepackage{microtype}      % microtypography
\usepackage{adjustbox}
\usepackage[flushleft]{threeparttable}
\usepackage{tcolorbox}
\usepackage[page, header]{appendix}

\author{%
  Kuba\\
  Muning\\
  Linghui\\
  Yaodong\\
  Shangding
}

\usepackage{subcaption}
\usepackage[utf8]{inputenc} % allow utf-8 input
\usepackage[T1]{fontenc}    % use 8-bit T1 fonts
\usepackage{url}            % simple URL typesetting
\usepackage{booktabs}       % professional-quality tables
\usepackage{nicefrac}       % compact symbols for 1/2, etc.
\usepackage{microtype}      % microtypography
\usepackage{xcolor}
\usepackage{graphicx}
\usepackage{booktabs} % for professional tables
\input{input_math}

\usepackage{float}
\usepackage{algorithmic}
\usepackage{algorithm}
\usepackage{stackengine}
\usepackage{listings}

\usepackage{amssymb}
\usepackage{comment}
\usepackage{apxproof}
\usepackage{thmtools, thm-restate}
\usepackage{wrapfig}

\declaretheorem{corollary}
\declaretheorem{remark}

\DeclareMathAlphabet\mathbfcal{OMS}{cmsy}{b}{n}
\usepackage{bm}

\makeatletter
\newenvironment{smalleralign}[1][\small]
 {\par\nopagebreak\leavevmode\vspace*{-\baselineskip}%
  \skip0=\abovedisplayskip
  #1%
  \def\maketag@@@##1{\hbox{\m@th\normalfont\normalsize##1}}%
  \abovedisplayskip=\skip0
  \align}
 {\endalign\ignorespacesafterend}
\makeatother
\definecolor{codegreen}{rgb}{0,0.6,0}
\definecolor{codegray}{rgb}{0.5,0.5,0.5}
\definecolor{codepurple}{rgb}{0.58,0,0.82}
\definecolor{backcolour}{rgb}{0.95,0.95,0.92}

\renewcommand*{\thefootnote}{(\arabic{footnote})}

\lstdefinestyle{mystyle}{
    backgroundcolor=\color{backcolour},   
    commentstyle=\color{codegreen},
    keywordstyle=\color{magenta},
    numberstyle=\tiny\color{codegray},
    stringstyle=\color{codepurple},
    basicstyle=\ttfamily\footnotesize,
    breakatwhitespace=false,         
    breaklines=true,                 
    captionpos=b,                    
    keepspaces=true,                 
    numbers=left,                    
    numbersep=5pt,                  
    showspaces=false,                
    showstringspaces=false,
    showtabs=false,                  
    tabsize=2
}

\usepackage{titletoc}
\usepackage{lipsum}
\usepackage{hyperref}
\hypersetup{
    colorlinks=true,
    linkcolor=red,
    citecolor=blue,
    filecolor=magenta,      
    urlcolor=blue,
linktocpage}
\newcommand\blfootnote[1]{%
\begingroup
\renewcommand\thefootnote{}\footnote{#1}%
\addtocounter{footnote}{-1}%
\endgroup
}
\title{Settling the Variance of Multi-Agent Policy Gradients}

\author{%
Jakub Grudzien Kuba$^{*,1, 2}$, Muning Wen$^{*,3}$, Linghui Meng$^{4}$, Shangding Gu$^{4}$, \\ \textbf{Haifeng Zhang$^{4}$, David Henry Mguni$^{2}$, Jun Wang$^{5}$, Yaodong Yang$^{\dag,6}$} \\
$^1$Imperial College London, $^2$Huawei R\&D UK,  $^3$Shanghai Jiao Tong University,  \\ $^4$Institute of Automation, Chinese Academy of Science, \\  $^5$University College London, $^6$Institute for AI, Peking University. % \\ $^\dag$Corresponding author  <yaodong.yang@kcl.ac.uk>.
}

\begin{document}

\maketitle

\begin{abstract}
Policy gradient (PG) methods are popular reinforcement learning (RL) methods where a baseline is often applied to reduce the variance of gradient estimates.  
In multi-agent RL (MARL), although the PG theorem can be naturally extended, the effectiveness of multi-agent PG (MAPG)  methods degrades as the variance of gradient estimates  increases rapidly with the number of agents.  
In this paper \blfootnote{$^*$Equal contribution.  $^\dag$Corresponding author  <yaodong.yang@pku.edu.cn>. }, we offer a rigorous analysis of MAPG methods by, firstly, quantifying the contributions  of the number of agents and agents' explorations  to  the variance of MAPG estimators. 
Based on this analysis, we derive the optimal baseline (OB) that achieves the minimal variance. 
In comparison to the OB, we measure the excess variance of existing MARL algorithms such as vanilla MAPG and COMA. 
Considering using deep neural networks,  we also propose a surrogate version of OB, which can be seamlessly plugged into any existing PG methods in MARL.   
On benchmarks of Multi-Agent MuJoCo and StarCraft challenges, our OB technique effectively stabilises training and improves the performance of multi-agent PPO  and COMA algorithms  by a significant margin.  Code is released at  \url{https://github.com/morning9393/Optimal-Baseline-for-Multi-agent-Policy-Gradients}.   %\footnote{\url{https://github.com/morning9393/Optimal-Baseline-for-Multi-agent-Policy-Gradients}}.
\end{abstract}

%\vspace{-5pt}
\section{Introduction}
%\vspace{-5pt}
Policy gradient (PG) methods refer to the category of  reinforcement learning (RL) algorithms where the  parameters of a stochastic policy are 
optimised with respect to the expected reward through gradient ascent. 
%learned by following an estimate of the expected total reward gradient.
Since the earliest embodiment of REINFORCE  \cite{williams1992simple}, 
PG methods, empowered by deep neural networks \cite{sutton:nips12}, are among the most effective model-free RL algorithms on various kinds of tasks \cite{baxter2001infinite, duan2016benchmarking, sac}. 
However,  the performance of PG methods is greatly affected by the variance of the PG estimator  \cite{tucker2018mirage,weaver2001optimal}. 
Since the RL agent behaves in an unknown stochastic environment, which is often considered as a black box, 
the randomness of expected reward can easily become  very large with increasing sizes  of the state and action spaces;  
this renders PG estimators with high variance, which concequently leads to low sample efficiency and  unsuccessful trainings \cite{gu2017q, tucker2018mirage}.  

To address the large variance issue and improve the PG estimation, different variance reduction methods were developed    \cite{gu2017q, peters2006policy,  tucker2018mirage, sugiyama}. 
%A famous technique is the  \emph{additive control variate} \cite{greensmith2004variance}, which adds a random variable to the PG estimator, and then subtracts its expectation from the resulting estimate, which preserves the estimator's expectation, but can decrease its variance. 
One of the most successfully applied and extensively studied methods is the  control variate   subtraction \cite{greensmith2004variance, gu2017q, weaver2001optimal}, also known as the  \emph{baseline} trick.  
A baseline is a scalar random variable, which can be subtracted from the state-action value samples in the PG  estimates so as to decrease the variance, meanwhile   introducing no bias to its expectation.
Baselines can be implemented through  a constant value \cite{greensmith2004variance, kimura1997reinforcement} or  a value that is dependent on the state \cite{greensmith2004variance, gu2017q, ciosek2020, sutton:nips12} such as the state value function, which results in the zero-centered advantage function  and  the  advantage actor-critic algorithm  \cite{mnih2016asynchronous}. 
State-action dependent  baselines can also be applied  \cite{gu2017q,wu2018variance}, although they are reported to have no advantages  over state-dependent baselines  \cite{tucker2018mirage}. 

When it comes to multi-agent reinforcement learning (MARL) \cite{yang2020overview}, PG methods can naturally be applied. 
One naive approach is to make each agent disregard its opponents and model them as part of the environment. In such a setting, single-agent PG methods can be applied in a fully decentralised way (hereafter referred as \textsl{decentralised training} (DT)). 
Although DT has numerous drawbacks, for example the non-stationarity issue and poor convergence guarantees \cite{claus1998dynamics, ctde}, it demonstrates excellent empirical performance in certain tasks \cite{papoudakiscomparative, schroeder2020independent}.     
%and learn to solve the resulting RL problem by estimating its policy gradient \cite{ctde}. We refer to this approach as \textsl{decentralised training} (DT). The DT approach has numerous drawbacks, among which are non-stationarity of the learning targets of the critics, which results in poor convergence guarantees \cite{ctde}, as well as forces the agents' critics to use relatively limited information \cite{ctde, maddpg}. 
A more rigorous treatment is to extend the PG theorem to the multi-agent policy gradient (MAPG) theorem, which  induces a learning  paradigm known as \textsl{centralised training with decentralised execution} (CTDE) \cite{foerster2018counterfactual, yang2018mean, mappo, zhang2020bi}. In CTDE, each agent during training maintains a centralised critic which takes joint information as input, for example COMA \cite{foerster2018counterfactual} takes joint state-action pair; meanwhile, it learns a decentralised policy that only depends on the local state, for execution.  
Learning the centralised critic helps address the non-stationarity issue encountered in DT \cite{foerster2018counterfactual, ctde}; this makes CTDE  an effective framework for implementing MAPG and 
successful applications have been achieved in many real-world tasks \cite{li2019efficient,spg, peng1703multiagent, wen2018probabilistic,gr2, zhou2020smarts}.  
%CTDE establishes the stationarity of the critic's target, reduces the number of neural networks that must be trained, and tends to stimulate cooperative behaviour among agents \cite{maddpg, ctde}. These advantages make CTDE the most attractice choice of MARL training in many problems . 

Unfortunately, compared to single-agent PG methods,  MAPG methods suffer more from the large variance issue. 
This is because in multi-agent settings, the randomness comes not only from each agent's own interactions with the environment but also other agents' explorations.   
In other words, an agent would not be able to tell if  an improved outcome is due to its own behaviour change or other agents' actions. 
Such a \textsl{credit assignment} problem \cite{foerster2018counterfactual, wang2020off} is believed to be one of the main reasons behind the large variance of CTDE methods \cite{ctde}; yet,  
despite the intuition being built, there is still a lack of mathematical treatment for understanding the contributing factors to  the variance of MAPG estimators.   
As a result, addressing the large variance issue in MARL is still challenging.  
One relevant baseline trick in MARL is  the application of a \textsl{counterfactual baseline}, introduced in  COMA   \cite{foerster2018counterfactual}; however, COMA  still suffers from the large variance issue empirically  \cite{papoudakiscomparative}.

%this prevents us from addressing the large variance issue in MARL and designing effective MAPG  methods.

%comes with its own notorious issue, which is that of \textsl{credit assignment} \cite{foerster2018counterfactual, wang2020off}. The problem is that even if the reward is improved, an agent still cannot tell whether that is due to its own behaviour change or other agents' explorations.
%This, however, is a rather intuitive formulation of the problem, and we choose its more rigorous version, which is
%that the joint action value is just an estimate of the individual agent's action value. Thus, this causes the CTDE PG estimator have extra variance over the DT estimator, induced by other agents' explorations \cite{ctde}. 
%However, there is still a lack of mathematical treatment for quantifying such difference, which essentially prevents us from addressing the large variance issue in MARL and designing effective CTDE methods. So far, the most effective variance reduction technique used in MARL is subtraction of a \textsl{counterfactual baseline}, introduced in COMA \cite{foerster2018counterfactual}. However, due to the above mentioned knowledge gap, the method's effectiveness has not yet been properly understood.

In this work, we analyse the variance of MAPG estimates mathematically. 
Specifically, we try to quantify the contributions of the number of agents and the effect of multi-agent explorations to the variance of MAPG estimators.  
One natural outcome of our analysis is the optimal baseline (OB), which achieves the minimal variance for MAPG estimators. Our OB technique can be seamlessly plugged into any existing MAPG methods. We incorporate it in COMA \cite{foerster2018counterfactual} and a multi-agent version of  PPO \cite{ppo},  and demonstrate its effectiveness by evaluating the resulting algorithms against the state-of-the-art algorithms.  Our main  contributions are summarised as follows: 
%\shangding{The following points may not provide our contributions very well, it can be summarized from the perspectives of what problems you formulated, what points you concerned, what methods, theory and algorithms you proposed, and what performance you achieved, etc.}
%\vspace{-10pt}
\vspace{-5pt}
\begin{enumerate}
    \item We rigorously quantify the excess variance of the  CTDE MAPG estimator to that of the DT one and prove that the order of such excess depends linearly on the number of agents, and quadratically on agents' exploration terms (i.e., the local advantages).  
    \item We demonstrate that the counterfactual baseline of COMA reduces the noise induced by other agents, but COMA still faces the large variance due to  agent's own exploration.
    \item We derive that there exists an optimal baseline (OB), which minimises the variance of an MAPG estimator,   and  introduce a surrogate version of OB that can be easily implemented  in any MAPG algorithms with deep neural networks. 
    \item We show by experiments that OB can effectively decrease the variance of MAPG estimates in COMA and multi-agent PPO, stabilise and accelerate training in StarCraftII and Multi-Agent MuJoCo environments. 
\end{enumerate}

%\vspace{-5pt}
\section{Preliminaries \& Background}
\label{section:preliminaries}
%\vspace{-5pt}
In this section, we provide the prelimiaries for the MAPG methods in MARL. 
We introduce notations and problem formulations in Section \ref{subsec:marl-problem}, present the MAPG theorem in Section \ref{subsec:mapg-theorem}, and finally, review two existing MAPG methods that our OB can be applied to in Section \ref{subsec:existing-methods}. 
%Subsection \ref{subsec:marl-problem} provides basic definitions, including those that we introduce in this paper for the first time. 
%Subsection \ref{subsec:mapg-theorem} provides basic results, including the MAPG theorem, and techniques to approximate MAPG. Subsection \ref{subsec:existing-methods} briefly describes some state-of-the-art MAPG-based methods. 
%\vspace{-5pt}
\subsection{Multi-Agent Reinforcement Learning Problem}
\label{subsec:marl-problem}
%\vspace{-5pt}
We formulate a MARL problem as a Markov  game \cite{littman1994markov}, represented as a tuple $\langle \mathcal{N}, \mathcal{S}, \mathbfcal{A}, \mathcal{P}, r, \gamma\rangle$. Here, $\mathcal{N} = \{1, \dots, n\}$ is the set of agents, $\mathcal{S}$ is the state space, $\mathbfcal{A} = \prod_{i=1}^{n}\mathcal{A}^{i}$ is the product of action spaces of the $n$ agents, known as the joint action space, $\mathcal{P} : \mathcal{S} \times \mathbfcal{A}\times \mathcal{S}\xrightarrow[]{}\mathbb{R}$ is the transition probability kernel, $r: \mathcal{S}\times \mathbfcal{A} \xrightarrow[]{}\mathbb{R}$ is the reward function (with $|r(s, \va)|\leq \beta, \ \forall s\in\mathcal{S}, \va\in\mathbfcal{A}$), and $\gamma \in [0, 1)$ is the discount factor. Each agent $i$ possesses a parameter vector $\theta^{i}$, which concatenated with parameters of other agents, gives the joint parameter vector $\vtheta$.
At a time step $t\in\mathbb{N}$, the agents are at state $\rs_{t} \in\mathcal{S}$. An agent $i$ takes an action $\ra^{i}_{t} \in\mathcal{A}^{i}$ drawn from its stochastic policy $\pi^{i}_{\vtheta}(\cdot | \rs_{t})$ parametrised by $\theta^{i}$ \footnote{It could have been more clear to write $\pi^{i}_{\theta^{i}}$ to highlight that $\pi^{i}$ depends only on $\theta^{i}$ and no other parts of $\vtheta$. Our notation, however, allows for more convenience in the later algebraic manipulations.} , simultaneously with other agents, which together gives a joint action $\rva_{t} =  (\ra^{1}_{t}, \dots, \ra^{n}_{t})\in\mathbfcal{A}$, drawn from the joint policy $\boldsymbol{\pi}_{\vtheta}(\rva_{t}|\rs_{t}) = \prod_{i=1}^{n}\pi^{i}_{\vtheta}(\ra^{i}_{t}|\rs_{t})$. The system moves to state $\rs_{t+1} \in\mathcal{S}$ with probability mass/density $\mathcal{P}(\rs_{t+1}|\rs_{t}, \rva_{t})$. 
A trajectory is a sequence $\tau = \langle \rs_{t}, \rva_{t}, \rr_{t} \rangle_{t=0}^{\infty}$ of states visited, actions taken, and rewards received by the agents in an interaction with the environment. The joint policy $\boldsymbol{\pi}_{\vtheta}$, the transition kernel $\mathcal{P}$, and the initial state distribution $d^{0}$, induce the marginal state distributions at time $t$, i.e., $d_{\vtheta}^{t}(s)$, which is a probability mass when $\mathcal{S}$ is discrete, and a density function when $\mathcal{S}$ is continuous. The total reward at time $t\in\mathbb{N}$  is defined as $R_{t} \triangleq \sum_{k=0}^{\infty}\gamma^{k}\rr_{t+k}$. 
The state value function $V_{\vtheta}$ and the state-action value function $Q_{\vtheta}$ are given by 
\vspace{-0pt}
\begin{align}
%\vspace{-10pt}
&V_{\vtheta}(s) \triangleq \E_{\rva_{0:\infty}\sim\boldsymbol{\pi}_{\vtheta}, \rs_{1:\infty}\sim\mathcal{P}}\left[ R_{0} \Big| \ \rs_{0}=s\right] \ , \ \ \ \ 
Q_{\vtheta}(s, \va) \triangleq \E_{\rs_{1:\infty}\sim\mathcal{P}, \rva_{1:\infty}\sim\boldsymbol{\pi}_{\vtheta}}\left[ R_{0}\Big| \ \rs_{0}=s, \  \ra_{0}= \rva
\right].
\nonumber
\end{align}
The advantage function is defined as $A_{\vtheta}(s, \va) \triangleq Q_{\vtheta}(s, \va) - V_{\vtheta}(s)$. The goal of the agents is to maximise the expected total reward $\mathcal{J}(\vtheta) \triangleq \E_{\rs\sim d^{0}}\big[ V_{\vtheta}(\rs) \big]$. 
 
In this paper, we write $-(i_{1}, \dots, i_{k})$ to denote the set of all agents excluding $i_{1}, \dots, i_{k}$ (we drop the bracket when $k=1$). 
We define the multi-agent state-action value function for agents $i_{1}, \dots, i_{k}$ as $ Q_{\vtheta}^{i_{1}, \dots, i_{k}}\left(s, \va^{(i_{1}, \dots, i_{k})}\right)
    \triangleq
    \E_{\rva^{-(i_{1}, \dots, i_{k})}\sim\boldsymbol{\pi}_{\vtheta}^{-(i_{1}, \dots, i_{k})} }
    \left[ Q_{\vtheta}\left(s, \va^{(i_{1}, \dots, i_{k})}, \rva^{-(i_{1}, \dots, i_{k})}\right)\right]$\footnote{
The notation "$\rva$" and "$(\rva^{-i}, \ra^{i})$", as well as "$\rva\sim\boldsymbol{\pi}_{\vtheta}$" and "$\rva^{-i}\sim\boldsymbol{\pi}^{-i}_{\vtheta},  \ra^{i}\sim\pi^{i}_{\vtheta}$", are equivalent.  We write $a^{i}$ and $\va$ when we refer to the action and joint action as to values,  and $\ra^{i}$ and $\rva$ as to random variables.},
which is the expected total reward once agents $i_{1}, \dots, i_{k}$ have taken their actions. Note that for $k=0$, this becomes the state value function, and for $k=n$, this is the usual state-action value function. As such, we can define the multi-agent advantage function as
\vspace{-0pt}
\begin{align}
    \label{eq:ma-advantage-definition}
    A_{\vtheta}^{i_{1}, \dots, i_{k}} & \left( s, \va^{(j_{1}, \dots, j_{m})}, \va^{(i_{1}, \dots, i_{k})} \right) 
    \nonumber\\
    &\triangleq
    Q_{\vtheta}^{j_{1},\dots, j_{m}, i_{1}, \dots, i_{k}}\left(s, \va^{(j_{1}, \dots, j_{m}, i_{1}, \dots, i_{k})} \right) - Q_{\vtheta}^{j_{1},\dots, j_{m}}\left(s, \va^{(j_{1}, \dots, j_{m})}\right),  
\end{align}
which is the advantage of agents $i_{1}, \dots, i_{k}$ playing $a^{i_{1}}, \dots, a^{i_{k}}$, given $a^{j_{1}}, \dots, a^{j_{m}}$. When $m=n-1$ and $k=1$, this is often referred to as the \textsl{local advantage} of agent $i$ \cite{foerster2018counterfactual}.

%\vspace{-5pt}
\subsection{The Multi-Agent Policy Gradient Theorem}
\label{subsec:mapg-theorem}
%\vspace{-5pt}
The Multi-Agent Policy Gradient Theorem \cite{ foerster2018counterfactual, marl} is an extension of the Policy Gradient Theorem \cite{sutton:nips12} from RL to MARL, and provides the gradient of $\mathcal{J}(\vtheta)$ with respect to agent $i$'s parameter, $\theta^{i}$, as
\vspace{-5pt}
\begin{align}
    &\nabla_{\theta^{i}}\mathcal{J}(\vtheta) = 
    \E_{\rs_{0:\infty}\sim d^{0:\infty}_{\vtheta}, 
    \rva^{-i}_{0:\infty}\sim \boldsymbol{\pi}_{\vtheta}^{-i}, 
    \ra^{i}_{0:\infty}\sim \pi^{i}_{\vtheta}}\left[
    \sum_{t=0}^{\infty}\gamma^{t}
    Q_{\vtheta}(\rs_{t}, \rva^{-i}_{t}, \ra^{i}_{t})
    \nabla_{\theta^{i}}\log\pi^{i}_{\vtheta}(\ra^{i}_{t}|\rs_{t})
    \right]\nonumber
\end{align}
From this theorem we can derive two MAPG estimators, one for CTDE (i.e., $\rvg^{i}_{\text{C}}$), where learners can query for the joint state-action value function, and one for DT (i.e., $\rvg^{i}_{\text{D}}$), where every agent can only query for its own state-action value function.  These estimators, respectively, are given by
\vspace{-5pt}
\begin{align}
    &\rvg^{i}_{\text{C}} \ = \ \sum_{t=0}^{\infty}\gamma^{t}\hat{Q}\left(\rs_{t}, \rva^{-i}_{t}, \ra^{i}_{t} \right)\nabla_{\theta^{i}}\log \pi_{\vtheta}^{i}\left(\ra^{i}_{t}\big|\rs_{t}\right) 
    \ \ \ \text{and} \ \ \
    \rvg^{i}_{\text{D}} \ = \ \sum_{t=0}^{\infty}\gamma^{t}\hat{Q}^{i}\left(\rs_{t}, \ra^{i}_{t}\right)\nabla_{\theta^{i}}\log \pi^{i}_{\vtheta}\left(\ra^{i}_{t}\big|\rs_{t}\right),\nonumber
\end{align}
where $\rs_{t}\sim d_{\vtheta}^{t}$, $\ra^{i}_{t}\sim\pi^{i}_{\vtheta}(\cdot|\rs_{t})$, $\rva^{-i}_{t}\sim\boldsymbol{\pi}^{-i}_{\vtheta}(\cdot|\rs_{t})$. 
Here, $\hat{Q}$ is a (joint) critic which agents query for values of $Q_{\vtheta}(\rs, \rva)$. Similarly, $\hat{Q}^{i}$ is a critic providing values of $Q^{i}_{\vtheta}(\rs, \ra^{i})$. The roles of the critics are, in practice, played by  neural networks that can be trained with TD-learning \cite{foerster2018counterfactual, sutton2018reinforcement}. For the purpose of this paper, we assume they give exact values. 
In CTDE, a \textsl{baseline} \cite{foerster2018counterfactual} is any  function  $b(\rs, \rva^{-i})$. For any such function, we can easily prove that
\vspace{-5pt}
\begin{align}
    \E_{\rs\sim d^{t}_{\vtheta}, \rva^{-i}\sim\boldsymbol{\pi^{-i}_{\vtheta}},
    \ra^{i}\sim\pi^{i}_{\vtheta}}\left[ b(\rs, \rva^{-i})\nabla_{\theta^{i}}\log\pi^{i}_{\vtheta}(\ra^{i}\big|\rs)\right] = \mathbf{0},\nonumber
\end{align}
(for proof see Appendix \ref{proof:baseline-no-bias} or \cite{foerster2018counterfactual}) which allows augmenting the CTDE estimator as follows
\vspace{-5pt}
\begin{align}
    \label{eq:grad-with-baseline}
    \rvg^{i}_{\text{C}}(b) \ = \ \sum_{t=0}^{\infty}\gamma^{t}\left[\hat{Q}(\rs_{t}, \rva^{-i}_{t}, \ra^{i}_{t}) - b(\rs_{t}, \rva^{-i}_{t})\right]\nabla_{\theta^{i}}\log \pi_{\vtheta}^{i}(\ra^{i}_{t}|\rs_{t}),
\end{align}
and it has exactly the same expectation as $\rvg^{i}_{\text{C}} = \rvg^{i}_{\text{C}}(0)$, but can lead to different variance properties. In this paper, we study total variance, which is the sum of variances of all components $\rg^{i}_{\text{C}, j}(b)$ of a vector estimator $\rvg^{i}_{\text{C}}(b)$. We note that one could consider the variance of every component of parameters, and for each of them choose a tuned baseline \cite{peters2006policy}. This, however, in light of neural networks with overwhelming parameter sizes used in deep MARL  seems not to have practical applications.

%\vspace{-5pt}
\subsection{Existing CTDE Methods}
\label{subsec:existing-methods}
%\vspace{-5pt}
The first stream of CTDE methods uses the collected experience in order to approximate MAPG and apply stochastic gradient ascent \cite{sutton2018reinforcement} to optimise the policy parameters.

\textbf{COMA} \cite{foerster2018counterfactual} is one of the most successful examples of these. It employs a centralised critic, which it adopts to compute a counterfactual baseline $b(\rs, \rva^{-i}) = \hat{Q}^{-i}(\rs, \rva^{-i})$. Together with Equations \ref{eq:ma-advantage-definition} \& \ref{eq:grad-with-baseline}, COMA gives the following MAPG estimator
\vspace{-5pt}
\begin{align}
    \label{eq:coma-estimator}
    \rvg^{i}_{\text{COMA}} = \sum_{t=0}^{\infty}\gamma^{t}\hat{A}^{i}(\rs_{t}, \rva_{t}^{-i},  \ra^{i}_{t})\nabla_{\theta^{i}}\log \pi_{\vtheta}^{i}(\ra^{i}_{t}|\rs_{t}) 
\end{align}
Another stream is the one of trust-region methods, started in RL by TRPO \cite{trpo} in which at every iteration, the algorithm aims to maximise the total reward with a policy in proximity of the current policy. It achieves it by maximising the objective
\begin{align}
    \label{eq:trpo-objective}
    \E_{\rs\sim\rho_{\theta_{\text{old}}}, \ra\sim\pi_{\theta_{\text{old}}}}
    \left[
    \frac{\pi_{\theta}(\ra|\rs)}{\pi_{\theta_{\text{old}}}(\ra|\rs)}
    \hat{A}(\rs, \ra)
    \right], \ \ \ \ \  \text{subject to }
    \E_{\rs\sim\rho_{\theta_{\text{old}}}}\Big[
    D_{\text{KL}}\big(  \pi_{\theta_{\text{old}}}\left(\cdot|\rs \right) \big\|
    \pi_{\theta}\left(\cdot|\rs\right)
    \big)
    \Big] \leq \delta .
\end{align}
PPO  \cite{ppo} has been developed as a trust-region method that is friendly for implementation; it  approximates TRPO by implementing the constrained optimisation by means of the PPO-clip objective.

\textbf{Multi-Agent PPO.}   PPO methods can be naturally extended to the MARL setting by leveraging the CTDE framework to train a shared policy for each agent via maximising the sum of their PPO-clip objectives, written as 
\begin{align}
    \label{eq:mappo}
    \sum_{i=1}^{n}\E_{\rs\sim\rho_{\vtheta_{\text{old}}}, \rva\sim\boldsymbol{\pi}_{\vtheta_{\text{old}}}}
    \left[ 
    \min\left(\frac{\pi_{\vtheta}(\ra^{i}|\rs)}{\pi_{\vtheta_{\text{old}}}(\ra^{i}|\rs)  }\hat{A}(\rs, \rva), \ \text{clip}\left(
    \frac{\pi_{\vtheta}(\ra^{i}|\rs)}{\pi_{\vtheta_{\text{old}}}(\ra^{i}|\rs)  }, 1-\epsilon, 1+ \epsilon
    \right)\hat{A}(\rs, \rva)\right)
    \right].
\end{align}
The clip operator replaces the ratio $\frac{\pi_{\vtheta}(\ra^{i}|\rs)}{\pi_{\vtheta_{\text{old}}}(\ra^{i}|\rs)}$ with, $1-\epsilon$ when its value is lower, or $1+\epsilon$ when its value is higher, to prevent large policy updates. Existing implementations of Equation \ref{eq:mappo} have been mentioned by \cite{de2020independent, mappo}. 

In addition to these, there are other streams of CTDE methods in MARL, such as MADDPG \cite{maddpg}, which follow the idea of  deterministic policy gradient method  \cite{silver2014deterministic}, and QMIX \cite{rashid2018qmix}, Q-DPP \cite{yang2020multi} and FQL \cite{zhou2019factorized}, which focus on the value function decomposition. These methods  learn either a deterministic policy or a value function, thus are not in the scope of stochastic MAPG methods. 

%\vspace{-5pt}
\section{Analysis and Improvement of Multi-agent  Policy Gradient Estimates }
\label{sec:analysis}
%\vspace{-5pt}
In this section, we provide a detailed analysis of the variance of the MAPG estimator, and propose a method for its reduction. Throughout the whole section we rely on the following two assumptions.
\begin{restatable}{assumption}{assumespace}
 The state space $\mathcal{S}$, and every agent $i$'s action space $\mathcal{A}^{i}$ is either discrete and finite, or continuous and compact. 
\end{restatable}
   
\begin{restatable}{assumption}{assumedifferen}
    For all $i\in\mathcal{N}$, $s\in\mathcal{S}$, $a^{i}\in\mathcal{A}^{i}$, the map $\theta^{i}\mapsto \pi^{i}_{\vtheta}(a^{i}|s)$ is continuously differentiable.
\end{restatable}

%\vspace{-5pt}
\subsection{Analysis of MAPG Variance}
\label{subsec:analysis}
%\vspace{-5pt}
The goal of this subsection is to demonstrate how an agent's MAPG estimator's variance is influenced by other agents. In particular, we show how the presence of other agents makes the MARL problem different from single-agent RL problems (e.g., the DT framework). %, whose role here is played by the DT framework. 
We start our analysis by studying the variance of the  component in CTDE estimator that depends on other agents, which is the joint $Q$-function. Since 
$
    \mathbf{Var}_{\rva\sim\boldsymbol{\pi}_{\vtheta}}\left[ Q_{\vtheta}(s, \rva) \right] =
    \mathbf{Var}_{\rva\sim\boldsymbol{\pi}_{\theta}}\left[ V_{\vtheta}(s) + A_{\vtheta}(s, \rva) \right] =
    \mathbf{Var}_{\rva\sim\boldsymbol{\pi}_{\vtheta}}\left[ A_{\vtheta}(s, \rva) \right]
$, 
we can focus our analysis on the advantage function, which has more interesting algebraic properties.
We start by presenting a simple lemma which, in addition to being a premise for the main result, offers some insights about the relationship between RL and MARL problems. %Moreover, it builds a relationship between RL and MARL, and so we think of it as generally appealing to the MARL community.
\begin{restatable}[Multi-agent advantage decomposition]{lemma}{maadlemma}
\label{lemma:ma-advantage-decomp}
For any state $s\in\mathcal{S}$, the following equation holds for any subset of $m$ agents and any permutation of their labels, 
%\vspace{-5pt}
\begin{align}
    A_{\vtheta}^{1, \dots, m}\left(s, \va^{(1, \dots, m)}\right) = \sum_{i=1}^{m}A_{\vtheta}^{i}\left(s, \va^{(1, \dots, i-1)}, a^{i}\right)\nonumber. 
\end{align}
\end{restatable}
For proof see Appendix \ref{proof:ma-advantage-decomp}.
The statement of this lemma is that the joint advantage of agents' joint action is the sum of sequentially unfolding multi-agent advantages of individual agents' actions. It  suggests that a MARL problem can be considered as a sum of $n$ RL problems.
The intuition from Lemma \ref{lemma:ma-advantage-decomp} leads to an idea of decomposing the variance of the total advantage into variances of multi-agent advantages of individual agents. Leveraging the proof of Lemma \ref{lemma:ma-advantage-decomp}, we further prove that
\begin{restatable}{lemma}{avdlemma}
\label{lemma:advantage-variance-decomp}
For any state $s\in\mathcal{S}$, we have 
\vspace{-5pt}
\begin{align}
    \mathbf{Var}_{\rva\sim\boldsymbol{\pi}_{\vtheta}}\big[ A_{\vtheta}(s, \rva) \big] =
    \sum_{i=1}^{n}\E_{\ra^{1}\sim\pi^{1}_{\vtheta}, \dots, \ra^{i-1}\sim\pi_{\vtheta}^{i-1}}\left[ \mathbf{Var}_{\ra^{i}\sim\pi^{i}_{\vtheta}}\left[ A_{\vtheta}^{i}\left(s, \rva^{(1, \dots, i-1)}, \ra^{i}\right) \right] \right] \nonumber. 
\end{align}
\end{restatable}
For proof see Appendix \ref{proof:advantage-variance-decomp}.
The above result reveals that the variance of the total advantage takes  a  sequential and additive structure of the advantage that is presented in Lemma \ref{lemma:ma-advantage-decomp}. This hints that a similar additive relation can hold once we loose  the sequential structure of the multi-agent advantages. %, trading off strictness for generality.
Indeed, the next lemma, which follows naturally from Lemma \ref{lemma:advantage-variance-decomp}, provides an upper bound for the joint advantage's variance in terms of local advantages, and establishes a  notion of  additivity of variance in MARL.
\begin{restatable}{lemma}{avublemma}
\label{lemma:advantage-var-upper-bound}
    For any state $s\in\mathcal{S}$, we have
    \vspace{-5pt}
    \begin{align}
        &\mathbf{Var}_{\rva\sim\boldsymbol{\pi}_{\vtheta}}
        \left[ A_{\vtheta}(s, \rva) \right] \ \leq \ \sum_{i=1}^{n}\mathbf{Var}_{\rva^{-i}\sim\boldsymbol{\pi}^{-i}_{\vtheta}, \ra^{i}\sim\pi^{i}_{\vtheta}}\left[ A^{i}_{\vtheta}(s, \rva^{-i}, \ra^{i}) \right].
        \nonumber
    \end{align}
\end{restatable}
For proof see Appendix \ref{proof:advantage-variance-decomp}.
Upon these lemmas we derive the main theoretical result of this subsection. The following theorem describes the order of excess variance that the centralised policy gradient estimator has over the decentralised one.
\begin{restatable}{theorem}{evtheorem}
    \label{theorem:excess-variance}
    The CTDE and DT estimators of MAPG satisfy 
    %\vspace{-5pt}
    \begin{align}
        &\mathbf{Var}_{\rs_{0:\infty}\sim d_{\vtheta}^{0:\infty}, \rva_{0:\infty}\sim\boldsymbol{\pi}_{\vtheta}}
        \big[ \rvg^{i}_{\text{C}} \big] -
        \mathbf{Var}_{\rs_{0:\infty} \sim d_{\vtheta}^{0:\infty}, \rva_{0:\infty}\sim\boldsymbol{\pi}_{\vtheta} }
        \big[ \rvg^{i}_{\text{D}} \big] \
        \leq \ \frac{B_{i}^{2}}{1-\gamma^{2}}\sum_{j\neq i}\epsilon_{j}^{2} \
        \leq \ (n-1)\frac{(\epsilon B_{i})^{2}}{1-\gamma^{2}}\nonumber
    \end{align}
    where
    {\small 
    $B_{i} = \sup_{s, \va}\left|\left|\nabla_{\theta^{i}}\log\pi^{i}_{\vtheta}\left( \ra^{i}|\rs \right) \right|\right| , \
    \epsilon_{i} = \sup_{s, \va^{-i}, a^{i}}\left|A^{i}_{\vtheta}(s, \va^{-i}, a^{i})\right| , \ \text{and} \ 
    \epsilon = \max_{i}\epsilon_{i} \ .
    $}
\end{restatable}
\begin{proofsketch}
(For the full proof see Appendix \ref{proof:grad-var-bound}.) 
    We start the proof by fixing a state $s$ and considering the difference of 
    {\small$
        \mathbf{Var}_{\rva\sim\boldsymbol{\pi}_{\vtheta}}[
        \hat{Q}(s, \rva^{-i}, \ra^{i})
        \nabla_{\theta^{i}}\log\pi^{i}_{\vtheta}(\ra^{i}|s)
        ] - 
        \mathbf{Var}_{\ra^{i}\sim\pi^{i}_{\vtheta}}[
        \hat{Q}^{i}(s, \ra^{i})
        \nabla_{\theta^{i}}\log\pi^{i}_{\vtheta}(\ra^{i}|s)
        ]
    $}. The goal of our proof strategy was to collate the terms {\small $\hat{Q}(s, \rva^{-i}, \ra^{i})$} and {\small $\hat{Q}^{i}(s, \ra^{i})$} because such an expression could be related to the  above lemmas about multi-agent advantage, given that the latter quantity is the expected value of the former when $\rva^{-i}\sim\pi^{-i}_{\vtheta}$.
    Based on the fact that these two estimators are unbiased, we transform the considered difference into 
    {\small$
    \E_{\rva\sim\boldsymbol{\pi}_{\vtheta}}\left[ \left|\left|
    \nabla_{\theta^{i}}\log\pi^{i}_{\vtheta}(\ra^{i}|s)
    \right|\right|^{2}
    \left(\hat{Q}(s, \rva)
    - \hat{Q}^{i}(s, \ra^{i})
    \right)^{2} 
    \right]
    $}.
    Using the upper bound $B_{i}$, the fact that {\small 
    $ \hat{A}^{-i}(s, \ra^{i}, \rva^{-i}) =  \hat{Q}(s, \ra^{i}, \rva^{-i}) - \hat{Q}^{i}(s, \ra^{i})$}, and the  result of Lemma \ref{lemma:advantage-var-upper-bound}, we bound this expectation by 
     {\small
        $B_{i}^{2}\sum_{j\neq i}
        \E_{\ra^{i}\sim\pi^{i}_{\vtheta}}\left[
        \mathbf{Var}_{\rva^{-i}\sim\pi^{-i}_{\vtheta}}\left[ \hat{A}^{j}(s, \rva^{-j}, \ra^{j}) \right]
        \right]
        $},
    which we then  rewrite to bound it by {\small $B_{i}^{2}\sum_{j\neq i}\epsilon_{j}^{2}$}. 
    As such, the first inequality in the theorem follows from summing, with discounting, over all time steps $t$,  and the second one is its trivial upper bound.
\end{proofsketch}

The result in Theorem \ref{theorem:excess-variance} exposes the level of difference between MAPG and PG estimation which, measured by variance, is not only non-negative, as shown in \cite{ctde}, but can grow \textbf{linearly} with the number of agents. More precisely, a CTDE learner's gradient estimator comes with an extra price of variance coming from other agents' local advantages (i.e., explorations). This further suggests that we shall search for variance reduction techniques which augment the state-action value signal. In RL, such a well-studied technique is baseline-subtraction, where many successful baselines are state-dependent. In CTDE, we can employ baselines taking state and other agents' actions into account. We demonstrate their strength on the example of a counterfactual baseline  of COMA \cite{foerster2018counterfactual}.

\begin{restatable}{theorem}{comadttheorem}
    \label{theorem:coma-dt}
     The COMA and DT estimators of MAPG satisfy
     \vspace{-5pt}
     \begin{align}
         \mathbf{Var}_{
         \rs_{0:\infty} \sim d_{\vtheta}^{0:\infty}, \rva_{0:\infty}\sim\boldsymbol{\pi}_{\vtheta} }\left[ \rvg^{i}_{\text{COMA}} \right] 
        - 
        \mathbf{Var}_{
        \rs_{0:\infty} \sim d_{\vtheta}^{0:\infty}, \rva_{0:\infty}\sim\boldsymbol{\pi}_{\vtheta}
        }\left[ \rvg^{i}_{\text{D}} \right] 
        \ \leq \ \frac{\left(\epsilon_{i}B_{i}\right)^{2}}{1-\gamma^{2}} 
        \nonumber
     \end{align}
\end{restatable}

For proof see Appendix \ref{proof:coma-de}. The above theorem discloses the effectiveness of the counterfactual baseline. COMA baseline essentially allows to drop the number of agents from the order of excess variance of CTDE, thus potentially binding it closely to the single-agent one. Yet, such binding is not exact, since it still contains the dependence on the local advantage,  which can be very large in scenarios when, for example, a single agent has a chance to revert its collaborators' errors with its own single action. Based on such insights, in the following subsections, we study the method of optimal  baselines, and derive a solution to these issues above.

%\vspace{-5pt}
\subsection{The Optimal Baseline for MAPG}
\label{subsec:ob}
%\vspace{-5pt}
In order to search for the optimal baseline, we first demonstrate how it can impact the variance of an MAPG estimator. % and discuss where the following result suggests us to focus while seeking good baselines. 
To achieve that, we decompose the variance at an arbitrary time step $t\geq0$, and separate the terms which are subject to possible reduction from those unchangable ones. Let us denote $\rvg_{\text{C}, t}^{i}(b) = [ \hat{Q}(\rs_{t}, \rva_{t}) - b]\nabla_{\theta^{i}}\log \pi^{i}_{\vtheta}(\ra^{i}_{t}|\rs_{t})$, sampled with $\rs_{t}\sim d_{\vtheta}^{t}, \rva_{t}\sim\boldsymbol{\pi}_{\vtheta}(\cdot|\rs_{t})$, which is essentially the $t^{\text{th}}$ summand of the gradient estimator given by Equation \ref{eq:grad-with-baseline}.
Specifically, we have 
\vspace{-5pt}
\begin{smalleralign}[\footnotesize]
    \label{eq:grad-var-decomp}
    &\mathbf{Var}_{\rs_{t}\sim d^{t}_{\vtheta}, \rva_{t}\sim\boldsymbol{\pi}_{\vtheta}}\left[ \rvg_{\text{C}, t}^{i}(b) \right] = \mathbf{Var}_{\rs_{t}\sim d^{t}_{\vtheta}}\left[ \E_{\rva_{t}\sim\boldsymbol{\pi}_{\vtheta}}\left[ \rvg_{\text{C}, t}^{i}(b) \right] \right] + \E_{\rs_{t}\sim d^{t}_{\vtheta}}\left[ \mathbf{Var}_{\rva_{t}\sim\boldsymbol{\pi}_{\vtheta}}\left[ \rvg_{\text{C}, t}^{i}(b) \right] \right] \\
    &=  \mathbf{Var}_{\rs_{t}\sim d^{t}_{\vtheta}}\left[ \E_{\rva_{t}\sim\boldsymbol{\pi}_{\vtheta}}\left[ \rvg_{\text{C}, t}^{i}(b) \right] \right] + \E_{\rs_{t}\sim d^{t}_{\vtheta}}\left[ \mathbf{Var}_{\rva_{t}^{-i}\sim\boldsymbol{\pi}^{-i}_{\vtheta}}\left[ \E_{\ra_{t}^{i}\sim\pi^{i}_{\vtheta}}\left[ \rvg^{i}_{\text{C}, t}(b) \right] \right]  +
    \E_{\rva_{t}^{-i}\sim\boldsymbol{\pi}^{-i}_{\vtheta}}\left[ \mathbf{Var}_{\ra_{t}^{i}\sim\pi^{i}_{\vtheta}}\left[ \rvg^{i}_{\text{C}, t}(b) \right] \right]  \right] 
    \nonumber\\
    &= \underbrace{\mathbf{Var}_{\rs_{t}\sim d^{t}_{\vtheta}}\left[ \E_{\rva_{t}\sim\boldsymbol{\pi}_{\vtheta}}\left[ \rvg_{\text{C}, t}^{i}(b) \right] \right]}_{\text{\footnotesize  Variance from state}} + 
    \underbrace{\E_{\rs_{t}\sim d^{t}_{\vtheta}}\left[ \mathbf{Var}_{\rva_{t}^{-i}\sim\boldsymbol{\pi}^{-i}_{\vtheta}}\left[ \E_{\ra_{t}^{i}\sim\pi^{i}_{\vtheta}}\left[ \rvg^{i}_{\text{C}, t}(b) \right] \right] \right]}_\text{\footnotesize Variance from other agents' actions} +  \underbrace{\E_{\rs_{t} \sim d^{t}_{\vtheta}, \rva_{t}^{-i}\sim\boldsymbol{\pi}^{-i}_{\vtheta}}\Big[ \mathbf{Var}_{\ra_{t}^{i}\sim\pi^{i}_{\vtheta}}\big[ \rvg^{i}_{\text{C}, t}(b) \big] \Big]}_\text{\footnotesize Variance from agent $i$'s action}. \nonumber
\end{smalleralign}
Thus, in a CTDE estimator, there are three main sources of variance, which are: state $\rs_{t}$, other agents' joint action $\rva_{t}^{-i}$, and the agent $i$'s action $\ra^{i}_{t}$. The first two terms of the right-hand side of the above equation, which are those involving variance coming from $\rs_{t}$ and $\rva_{t}^{-i}$, remain constant for all $b$. However, the baseline subtraction influences the local variance of the agent, $\mathbf{Var}_{\ra_{t}^{i}\sim\pi^{i}_{\vtheta}}\big[ \rvg^{i}_{\text{C}, t}(b) \big] $, and therefore  minimising it for every  $(s, \va^{-i})$ pair minimises the third term, which is equivalent to minimising the entire variance. In this subsection, we describe how to perform this minimisation.

\begin{restatable}[Optimal baseline for MAPG]{theorem}{obtheorem}
\label{theorem:ob-marl}
The optimal  baseline (OB) for the  MAPG estimator is
\begin{align}
    b^{\text{optimal}}(\rs, \rva^{-i}) = \frac{ \E_{\ra^{i}\sim\pi^{i}_{\vtheta}}\left[ \hat{Q}(\rs, \rva^{-i}, \ra^{i}) \left|\left| 
    \nabla_{\theta^{i}}\log \pi^{i}_{\vtheta}(\ra^{i}|\rs) \right|\right|^{2} \right] }{  
    \E_{\ra^{i}\sim\pi^{i}_{\vtheta}}\Big[ \left|\left| \nabla_{\theta^{i}}\log \pi^{i}_{\vtheta}(\ra^{i}|\rs) \right|\right|^{2} \Big]
    }
    \label{eq:ob}
\end{align}
\end{restatable}
For proof see Appendix \ref{proof:ob-marl}. Albeit elegant, OB in Equation \ref{eq:ob} is computationally challenging to estimate due to the fact that 
%Bad news following this result are that the computation, or estimation, of this baseline would 
it requires  a repeated computation of the norm of the gradient $\nabla_{\theta^{i}}\log \pi^{i}_{\vtheta}(\ra^{i}|\rs)$, which can have dimension of order $\sim 10^{4}$ when the policy is parametrised by a neural network (e.g., see \cite[Appendix C.2]{son2019qtran}). %, which is computationally expensive and thus, impractical. 
Furthermore, in continuous action spaces, as in principle the $Q$-function does not have a simple analytical form, this baseline cannot be computed exactly, and instead it must be approximated. This is problematic, too, because the huge dimension of the gradients may induce large variance in the approximation of OB in addition to the variance of the policy gradient estimation. To make OB computable and applicable, we formulate in the next section a surrogate variance-minimisation objective, whose solution is much more tractable.

%\vspace{-5pt}
\subsection{ Optimal Baselines for Deep Neural Networks}
\label{subsec:deep-ob}
%\vspace{-5pt}
Recall that in deep MARL, the policy $\pi^{i}_{\vtheta}$ is assumed to be a member of a specific family of distributions, and the network $\theta^{i}$ only computes its parameters, which we can refer to as $\psi^{i}_{\vtheta}$. In the discrete MARL, it can be the last layer before $\softmax$, and in the continuous MARL, $\psi^{i}_{\vtheta}$ can be the mean and the standard deviation of a Gaussian distribution \cite{foerster2018counterfactual, mappo}. We can then write $\pi^{i}_{\vtheta}(\ra^{i}|\rs) = \pi^{i}\left(\ra^{i}|\psi^{i}_{\vtheta}(\rs)\right)$, and factorise the gradient {\small $\nabla_{\theta^{i}}\log\pi^{i}_{\vtheta}$} with the chain rule. This allows us to rewrite the local variance as
\begin{smalleralign}
    \label{eq:local-var}
    &\mathbf{Var}_{\ra^{i}\sim\pi^{i}_{\vtheta}}\left[  \nabla_{\theta^{i}}\log\pi^{i}_{\vtheta}\left(\ra^{i}|\psi^{i}_{\vtheta}(\rs)\right)
    \left(\hat{Q}(\rs, \rva^{-i}, \ra^{i}) - b(\rs, \rva^{-i}) \right)\right]
    \nonumber\\
    =& \mathbf{Var}_{\ra^{i}\sim\pi^{i}_{\vtheta}}
    \left[ 
    \nabla_{\theta^{i}}\psi^{i}_{\vtheta}(\rs)  \nabla_{\psi^{i}_{\vtheta}(\rs)}\log\pi^{i}
    \left(\ra^{i}|\psi^{i}_{\vtheta}(\rs) \right)
    \left(\hat{Q}(\rs, \rva^{-i}, \ra^{i}) - b(\rs, \rva^{-i})\right)
    \right]
    \nonumber\\
    =& 
    \nabla_{\theta^{i}}\psi^{i}_{\vtheta}(\rs)
    \mathbf{Var}_{\ra^{i}\sim\pi^{i}_{\vtheta}}
    \left[
    \nabla_{\psi^{i}_{\vtheta}(\rs)}\log\pi^{i}
    \left(\ra^{i}|\psi^{i}_{\vtheta}(\rs)\right)
    \left(\hat{Q}(\rs, \rva^{-i}, \ra^{i}) - b(\rs, \rva^{-i})\right)
    \right]
    \nabla_{\theta^{i}}\psi^{i}_{\vtheta}(\rs)^{T}.
\end{smalleralign}
This allows us to formulate a surrogate minimisation objective, which is the variance term from Equation \ref{eq:local-var}, which we refer to as \textsl{surrogate local variance}. The optimal baseline for this objective comes as a corollary to the proof of Theorem \ref{theorem:ob-marl}. 
\begin{corollary}
\label{corollary:surrogate-ob}
The optimal baseline for the surrogate local variance in Equation \ref{eq:local-var} is
    \vspace{-5pt}
    \begin{smalleralign}
    \label{eq:surrob}
        &b^{*}(\rs, \rva^{-i}) 
        = \frac{ \E_{\ra^{i}\sim\pi^{i}_{\vtheta}}\left[ \hat{Q}(\rs, \rva^{-i}, \ra^{i})
        \left|\left|
        \nabla_{\psi^{i}_{\vtheta}(\rs)}\log\pi^{i}\left(\ra^{i}|\psi^{i}_{\vtheta}(\rs)\right)
        \right|\right|^{2} \right]  }{ \E_{\ra^{i}\sim\pi^{i}_{\vtheta}}\left[ \left|\left|
        \nabla_{\psi^{i}_{\vtheta}(\rs)}\log\pi^{i}\left(\ra|\psi^{i}_{\vtheta}(\rs)\right)
        \right|\right|^{2} \right]  }.
    \end{smalleralign}
\end{corollary} 
Note that the vector $\nabla_{\psi^{i}_{\vtheta}(\rs)}\log\pi^{i}\left(\ra^{i}|\psi^{i}_{\theta}(\rs)\right)$ can be computed without backpropagation when the family of distributions to which $\pi^{i}_{\vtheta}$ belongs is known, which is fairly common in deep MARL. Additionally,  the dimension of this   vector is of the same size as the size of the action space, which is  in the order  $\sim 10$ in many cases (e.g.,  \cite{foerster2018counterfactual, sunehag2018value})  %\yaodong{again, i don;t know when the evidecen come from}\Kuba{When they describe experiments, they describe action space, which implicitly tells you size of softmax. }, 
which makes computations  tractable. Equation \ref{eq:surrob} essentially allows us to incorporate the OB in any existing (deep) MAPG methods, accouting for both continuous and discrete-action taks. 
Hereafater, we refer to the surrogate OB in Equation \ref{eq:surrob} as the OB and apply it in the later experiment section. 

%by its approximation in continuous action spaces, and exact computation in discrete action spaces %\yaodong{explain why it is approxiamtion in continuous and exact in discrete} \Kuba{I did already, just above the section 3.3}. %, where it can be done in parallel with other steps of the associated algorithm. 

%\vspace{-5pt}
\subsection{Excess Variance of  MAPG/COMA \emph{vs.} OB}
\label{subsec:comparison}
\vspace{-5pt}
%\Kuba{Here, the first two sentences seem not to be linked.}
We notice that for a probability measure ${\small x^{i}_{{\psi}^{i}_{\vtheta}}(\ra^{i}|s) = \frac{\pi^i_{\vtheta}(\ra^i|s) \big|\big|\nabla_{\psi^i_\vtheta(s)}\log\pi^i(\ra^i|\psi^i_\vtheta(s))\big|\big|^2}{\E_{\ra^i\sim\pi^i_\vtheta}\big[ \big|\big|\nabla_{\psi^i_\vtheta(s)}\log\pi^i(\ra^i|\psi^i_\vtheta(s))\big|\big|^2 \big] } }$, the OB takes the form of 
{\small $b^{*}(\rs, \rva^{-i}) = \E_{\ra^{i}\sim x^{i}_{\psi^{i}_{\vtheta}}}\big[ \hat{Q}(\rs, \rva^{-i}, \ra^{i}) \big]$}.
It is then instructive to look at a  practical deep MARL example, which is that of a discrete actor with policy $\pi^{i}_{\vtheta}(\ra^{i}|\rs) = \softmax(\psi_{\vtheta}^{i}(\rs))(\ra^{i})$.
In this case, we can derive that 
%\vspace{-5pt}
\begin{align}
    \label{eq:x-measure-softmax}
    x^{i}_{\psi^{i}_{\vtheta}}\left( \ra^{i}|\rs \right)
    \propto
    \pi^{i}_{\vtheta}\left(\ra^{i}|\rs\right)
    \left(1 + \left|\left|\pi_{\vtheta}^{i}(\rs)\right|\right|^{2} - 2\pi_{\vtheta}^{i}\left(\ra^{i}|\rs\right)\right) 
\end{align}
(Full derivation is shown in Appendix \ref{proof:x-measure}).
This measure, in contrast to COMA, scales up the weight of actions with small weight in $\pi^{i}_{\vtheta}$, and scales down the weight of actions with large $\pi^{i}_{\vtheta}$, while COMA's baseline simply takes each action with weight determined by $\pi^{i}_{\vtheta}$, which has an opposite effect to OB. 
Let \textsl{the excess surrogate local variance} of a CTDE MAPG estimator $\rvg^{i}_{\text{C}}(b)$ of agent $i$ be defined as Equation \ref{eq:exss}. We analyse the excess variance in Theorem \ref{thm:excess}. 
\begin{smalleralign}
\label{eq:exss}
    \Delta \mathbf{Var}(b) \triangleq \mathbf{Var}_{\ra^{i}\sim \pi_{\vtheta}^{i}}\left[ \rvg^{i}_{\text{C}}(b) \right] - \mathbf{Var}_{\ra^{i}\sim \pi^{i}_{\vtheta}}\left[ \rvg^{i}_{\text{C}}(b^{*}) \right]
\end{smalleralign}

\begin{restatable}{theorem}{esvtheorem}
\label{thm:excess}
The excess surrogate local variance for baseline $b$ satisfies
\vspace{-5pt}
\begin{smalleralign}
    \Delta \mathbf{Var}(b) = \left(b - b^{*}(s, \va^{-i})\right)^{2}
    \E_{\ra^{i}\sim\pi^{i}_{\vtheta}}\left[ \left|\left|
    \nabla_{\psi^{i}_{\vtheta}}\log\pi^{i}\left(\ra^{i}\big|\psi^{i}_{\vtheta}(s)\right)
    \right|\right|^{2} 
    \right]\nonumber
    \nonumber
\end{smalleralign}
In particular, the excess variance of the vanilla MAPG and COMA estimators satisfy 
\vspace{-5pt}
\begin{smalleralign}
    &\Delta \mathbf{Var}_{\text{MAPG}} \ \leq \  D_{i}^{2}\left(\mathbf{Var}_{\ra^{i}\sim\pi^{i}_{\vtheta}}\big[ A^{i}_{\vtheta}(s, \va^{-i}, \ra^{i}) \big] + Q_{\vtheta}^{-i}(s, \va^{-i})^{2} \right) \ \leq \ D_{i}^{2}\ \left(\epsilon_{i}^{2} + \left[ \frac{\beta}{1-\gamma}\right]^{2} \right) \nonumber\\
    & \Delta \mathbf{Var}_{\text{COMA}}  \ \leq \ 
    D_{i}^{2}\ \mathbf{Var}_{\ra^{i}\sim\pi^{i}_{\theta}}\big[ A^{i}_{\vtheta}(s, \va^{-i}, \ra^{i}) \big]
    \leq \ \left(\epsilon_{i}D_{i} \right)^{2} \nonumber
\end{smalleralign}
where 
    {\small 
    $D_{i} = \sup_{a^{i}}\left|\left|\nabla_{\psi_{\vtheta}^{i}}\log\pi^{i}_{\vtheta}\left(a^{i}|\psi^{i}_{\vtheta}(s)\right)\right|\right|,
    \ \text{and} \
    \epsilon_{i} = \sup_{s, \va^{-i}, a^{i}}\left|A^{i}_{\vtheta}(s, \va^{-i}, a^{i})\right| 
    $}.
\end{restatable}

For proof see Appendix \ref{proof:comparison}.
This theorem implies that OB is particularly helpful in situations when the value of $Q^{-i}$ function is large, or when agent $i$'s local advantage has large variance. In these scenarios, a baseline like COMA might fail because when certain actions have large local advantage,  we would want the agent to learn them, although its gradient estimate may be  inaccurate, disabling the agent to learn to take the action efficiently.

\subsection{Implementation of the Optimal Baseline}
\label{subsec:implementation}
%\vspace{-5pt}
Our OB technique is a general method to any MAPG methods with a joint critic. It can be seamlessly integrated into any existing MARL algorithms that require MAPG estimators. One only needs to 
%It is straightforward to apply OB to any method with a joint critic. 
replace the algorithm's state-action value signal (either state-action value or advantage function) with
%\vspace{-5pt}
\begin{align}
    \label{eq:x-value-def}
    \hat{X}^{i}(\rs, \rva^{-i}, \ra^{i}) = \hat{Q}(\rs, \rva^{-i}, \ra^{i}) - b^{*}(\rs, \rva^{-i}).
\end{align}
This gives us an estimator of 
\textbf{COMA with OB} 
%\vspace{-5pt}
\begin{align}
    \rvg^{i}_{\text{X}} = \sum_{t=0}^{\infty}\gamma^{t}\hat{X}^{i}(\rs_{t}, \rva^{-i}_{t}, \ra^{i}_{t})\nabla_{\theta^{i}}\log\pi^{i}_{\vtheta}(\ra^{i}_{t}|\rs_{t}),\nonumber
\end{align}
and a variant of 
\textbf{Multi-agent PPO with OB}, which maximises the objective of 
\begin{align}
\vspace{-5pt}
    \sum_{i=1}^{n}\E_{\rs\sim\rho_{\vtheta_{old}}, \rva\sim\boldsymbol{\pi}_{\vtheta_{old}}}
    \left[ 
    \min\left(\frac{\pi_{\vtheta}(\ra^{i}|\rs)}{\pi_{\vtheta_{old}}(\ra^{i}|\rs)  }\hat{X}^{i}(\rs, \rva), \ \text{clip}\left(
    \frac{\pi_{\vtheta}(\ra^{i}|\rs)}{\pi_{\vtheta_{old}}(\ra^{i}|\rs)  }, 1-\epsilon, 1+ \epsilon
    \right)\hat{X}^{i}(\rs, \rva)\right)
    \right]\nonumber . 
    \vspace{-5pt}
\end{align}

In order to compute OB, agent $i$ follows these two steps: firstly, it evaluates the probability measure {\small $x^{i}_{\psi^{i}_{\vtheta}}$}, and then computes the expectation of {\small $\hat{Q}(s, \va^{-i}, \ra^{i})$} over it with a dot product. Such a protocol allows for exact computation of OB when the action space is discrete.
When it comes to continuous action space, the first step of evaluating  {\small $x^{i}_{\psi^{i}_{\vtheta}}$} relies on sampling actions from agent's policy, which gives us the approximation of OB.  
%, and for approximation of it when the action space is continuous, as in that case, the first step relies on a finite number of action samples. 
To make it clear, 
we provide PyTorch implmentations of  OB in both discrecte and continuous settings in Appendix  \ref{appendix:implementation}. 

%For implementation of the two OB versions in Pytorch \cite{paszke2019pytorch}, see Appendix \ref{appendix:implementation}. 

\captionsetup[table]{font=normalsize}
% NUMERICAL TOY EXAMPLE
\begin{table}[t!]
    \vspace{-0pt}
    \begin{adjustbox}{width=.97\columnwidth, center}
    \begin{threeparttable}
        \vspace{-0pt}
                \caption{ A numerial toy exmaple that shows the effectiveness of OB. 
                For all actions in column $\ra^{i}$,  agent $i$ is provided with the last layer before $\softmax$ of its actor network and the actions' values (columns {\small $\psi^{i}_{\vtheta}(\ra^{i})$} and {\small $\hat{Q}(\va^{-i}, \ra^{i})$}). It computes the remaining quantities in the table, which are used to derive the three gradient estimators, whose variance is summarised in the right part of the table.
        The full calculations of the below values are stored in Appendix \ref{proof:toy-example}.}
             \begin{tabular}{c | c c c c c c| c c } \toprule
             \\[-1em]
             $\ra^{i}$ 
             & $\psi^{i}_{\vtheta}(\ra^{i})$ 
             & $\pi_{\vtheta}^{i}(\ra^{i})$ 
             & $x_{\psi^{i}_{\vtheta}}^{i}(\ra^{i})$ 
             & $\hat{Q}( \va^{-i}, \ra^{i})$ 
             & $\hat{A}^{i}(\va^{-i}, \ra^{i})$
             & $\hat{X}^{i}(\va^{-i}, \ra^{i})$
             &  Method & Variance \\ [0.6ex]
             \midrule
             1 & $\log 8$ & $0.8$ & 0.14 & $2$ & $-9.7$ & $-41.71$ & MAPG & \textbf{1321}\\ 
             2 & $0$ & $0.1$ &  0.43 & $1$ &$-10.7$ & $-42.71$ & COMA & \textbf{1015}\\
             3 & $0$ & $0.1$ &  0.43 & $100$ & $88.3$ & $56.29$ & OB & \textbf{\color{red}{673}}\\
             \bottomrule
             \end{tabular}
            \label{table:toy-example}
    \end{threeparttable}
    \end{adjustbox}
    \vspace{-5pt}
\end{table}

\captionsetup[figure]{font=normalsize}
% COMA EXPERIMENTS
\begin{figure*}[t!]
 	\centering
 	 			\vspace{-0pt}
 			\centering
 			\includegraphics[width=.99\textwidth]{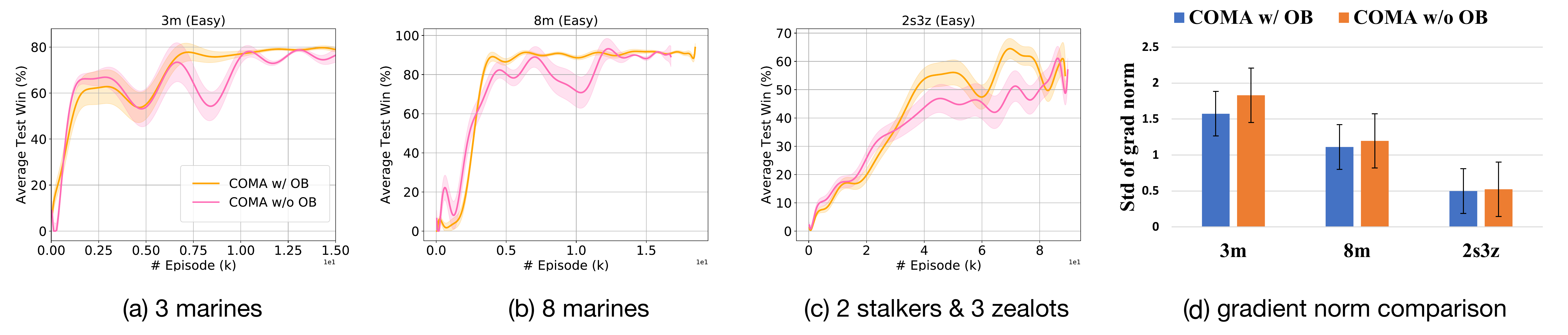}
% 	\begin{subfigure}{0.33\linewidth}
% 			\centering
% 			\includegraphics[width=1.4in]{images/figs/smac/win_rate_8m.pdf} 
% 			 			\vspace{-3pt}
% 			\caption{8 marines}
% 			\label{fig:exp_coma_8m_win_rate}
% 		\end{subfigure}
% 	\begin{subfigure}{0.33\linewidth}
% 			\centering
% 			\includegraphics[width=1.4in]{images/figs/smac/win_rate_2s3z_0.005.pdf} 
% 			 			\vspace{-3pt}
% 			\caption{2 stalkers \& 3 zealots}
% 			\label{fig:exp_coma_2s3z_win_rate}
% 		\end{subfigure}
 	%}%
 	\vspace{-5pt}
 	\caption{Performance comparisons between COMA with and without OB on three  SMAC challenges. %: (a) the map of \emph{3m}, (b) the map of \emph{8m}, and (c) the map of \emph{2s3z}.
 	}
 	\label{fig:matching_coma}
 	\vspace{-10pt}
\end{figure*}

%\vspace{-5pt}
\section{Experiments}
\label{section:experiments}
%\vspace{-5pt}
On top of  theoretical proofs, in this section, 
we demonstrate empircial evidence that OB can decrease the variance of MAPG estimators, stabilise training, and most importantly,  lead to better performance. 
To verify the adaptability of OB, we  apply  OB on both COMA and multi-agent PPO methods as described in Section \ref{subsec:implementation}. 
%Both of the methods can deal with either continuous or discrete control tasks. 
%We use OB to augment COMA~\cite{foerster2018counterfactual}, as to demonstrate the superiority of OB over COMA's baseline. We also add OB to MAPPO~\cite{mappo}, which can be applied to both discrete and continuous action spaces, and allows us to confirm that both versions of OB greatly reduce variance. 
We benchmark the OB-modified algorithms against existing  state-of-the-art (SOTA) methods, which include  COMA \cite{foerster2018counterfactual}  and MAPPO \cite{mappo}, and value-based methods such as QMIX \cite{rashid2018qmix} and COMIX \cite{peng2020facmac}, and a deterministic PG method, ie., MADDPG \cite{maddpg}. 
Notably, since OB relies on the  Q-function critics, we did not apply the GAE \cite{schulman2015high} estimator that builds only on the state value function when implementing multi-agent PPO for fair comparisons. % because of its theoretical bias, which is in breach our assumptions (see Appendix \ref{assutions}). 
For each of the baseline on each task, we report the results of five random seeds. 
 We refer to Appendix \ref{section:details-experiments} for the detailed  hyper-parameter settings for baselines.

%
%\begin{wrapfigure}{l}{0.35\textwidth}
%  \vspace{-20pt}
%  \begin{center}
%    \includegraphics[width=0.32\textwidth]{images/figs/smac/std_of_grad_norm.pdf}
%  \end{center}
%  \vspace{-5pt}
%  \caption{Standard deviation of gradient norm of COMA (with and without OB) on three SMAC maps. }
%  \vspace{-18pt}
%  \label{fig:exp_coma_variance}
%\end{wrapfigure}
\textbf{Numerical Toy Example}. We first offer a numerical toy example to demonstrate how the subtraction of OB alters the state-action value signal in the estimator, as well as how this technique performs, against vanilla MAPG and COMA. %, on the benchmark of the surrogate local variance. 
We assume a stateless setting, and a given joint action of other agents.  
%Agent $i$ computes the $\psi^{i}_{\vtheta}$ vector, and has access to $Q$-values from the joint critic. The rest of the quantities are computed from them. 
The results in Table \ref{table:toy-example} show that the measure $x^i_{\psi^{i}_{\vtheta}}$ puts more weight on actions neglected by $\pi^{i}_{\vtheta}$, lifting the value of OB beyond the COMA baseline, as suggested by Equation \ref{eq:x-measure-softmax}.
The resulting $X^{i}$ function penalises the sub-optimal actions  more heavily. Most importantly, OB provides a MAPG estimator with far lower variance as expected.

\textbf{StarCraft Multi-Agent Challenge (SMAC)}~\cite{samvelyanstarcraft}. 
In SMAC, each individual unit is controlled by a learning agent, which has finitely many possible actions to take.
The units cooperate to defeat enemy bots across scenarios of different levels of difficulty. 
%The experiment results on SMAC reveal that OB decreases the variance of COMA and MAPPO gradient estimates (see Figure \ref{fig:matching_coma}(d) and Table \ref{table:exp_mappo_variance}). 
Based on   Figure  \ref{fig:matching_coma}(a-d), we can tell that OB provides more accurate MAPG estimates and stabilises training of COMA across all three maps.  
 Importantly,  COMA with OB learns policies that achieve higher rewards than the classical COMA. %, as it shown on Figures \ref{fig:exp_coma_3m_win_rate} and \ref{fig:exp_coma_2s3z_win_rate}.
%COMA with OB achieves lower variance of performance across different random seeds, as all these three figures show. 
Since COMA perform badly on hard and super-hard maps in SMAC \cite{papoudakiscomparative}, we only report their  results on easy maps. 
On the hard and super-hard maps in Figures \ref{fig:exp_3s_vs_5z_win_rate}, \ref{fig:exp_5m_vs_6m_win_rate}, and \ref{fig:exp_6h_vs_8z_win_rate}, OB improves the performance of multi-agent PPO. 
Surprisingly on Figure \ref{fig:exp_3s_vs_5z_win_rate}, OB improves the winning rate of multi-agent PPO from zero to $90$\%. 
Moreover, with an increasing number of agents, the effectiveness of OB increases in terms of offering a low-variance MAPG estimator.  
According to Table \ref{table:exp_mappo_variance}, when the tasks involves $27$ learning agents, OB offers a $40$\% reduction in the variance of the gradient norm. 

%In certain discrete scenarios, it is necessary to use OB in MAPPO in order to make training sucessful, as it is shown on Figure \ref{fig:exp_3s_vs_5z_win_rate}. In principle, OB enables discrete MAPPO to learn better policies, as Figures \ref{fig:exp_3s_vs_5z_win_rate}, \ref{fig:exp_5m_vs_6m_win_rate}, and \ref{fig:exp_6h_vs_8z_win_rate} show. Figure \ref{fig:exp_27m_vs_30m_win_rate} presents a super hard SMAC task where SOTA methods, including MAPPO with OB, perform comparably. 

\textbf{Multi-Agent MuJoCo}~\cite{de2020deep}. 
SMAC are discrete control tasks; here we study the performance of OB when the action space is continuous. 
In each environment of Multi-Agent MuJoCo, each individual agent controls a part of a shared robot (e.g., a leg of a Hopper), and all agents maximise a shared reward function.  %We use these tasks to evaluate the performance of OB approximated by sampling, and applied to MAPPO. 
Results are consistent with the findings on SMAC in Table \ref{table:exp_mappo_variance}; on all tasks, OB helps decrease the variance of gradient norm. 
Moreover, 
 OB can  improve the performance of multi-agent PPO on most MuJoCo tasks in Figure \ref{fig:matching_mappo} (top row), and decreases its variance in all of them.   %HalfCheetach and Swimmer in Figures \ref{fig:exp_swimmer_v2} and \ref{fig:exp_halfcheetah_v2}. 
 In particular, multi-agent PPO with OB performs the best on Walker (i.e., Figure {\ref{fig:exp_walker_v2}}) and HalfCheetah (i.e., Figure \ref{fig:exp_halfcheetah_v2}) robots, and % in Figures \ref{fig:exp_hopper_v2}, \ref{fig:exp_halfcheetah_v2}, and \ref{fig:exp_walker_v2}. 
%The classical and the OB version of MAPPO tend to outperform other continuous SOTA baselines, as Figures \ref{fig:exp_hopper_v2}, \ref{fig:exp_halfcheetah_v2}, and \ref{fig:exp_walker_v2} show.  The approximated OB can drastically improve the prformace of MAPPO, as it shown on Figures \ref{fig:exp_swimmer_v2} and \ref{fig:exp_halfcheetah_v2}. 
with the increasing number of agents, the effectiveness OB becomes apparent; $6$-agent HalfCheetah achieves the largest performance gap.  
%OB is particularly useful when the number of agents is large, as Theorem \ref{fig:exp_halfcheetah_v2} suggests and demonstrates. 
%Figure \ref{fig:exp_walker_v2} is an example of comparable performance of these SOTA methods.  The approximated OB stabilises the MAPPO gradient thourghout training in continuous tasks, as seen in Table \ref{table:exp_mappo_variance}.

\captionsetup[figure]{font=footnotesize}
% MAPPO EXPERIMENTS
\begin{figure*}[t!]
 	\centering
 	\vspace{-5pt}
 	 %\subfigure[]{
 		\begin{subfigure}{0.24\linewidth}
 			\centering
 			\includegraphics[width=1.4in]{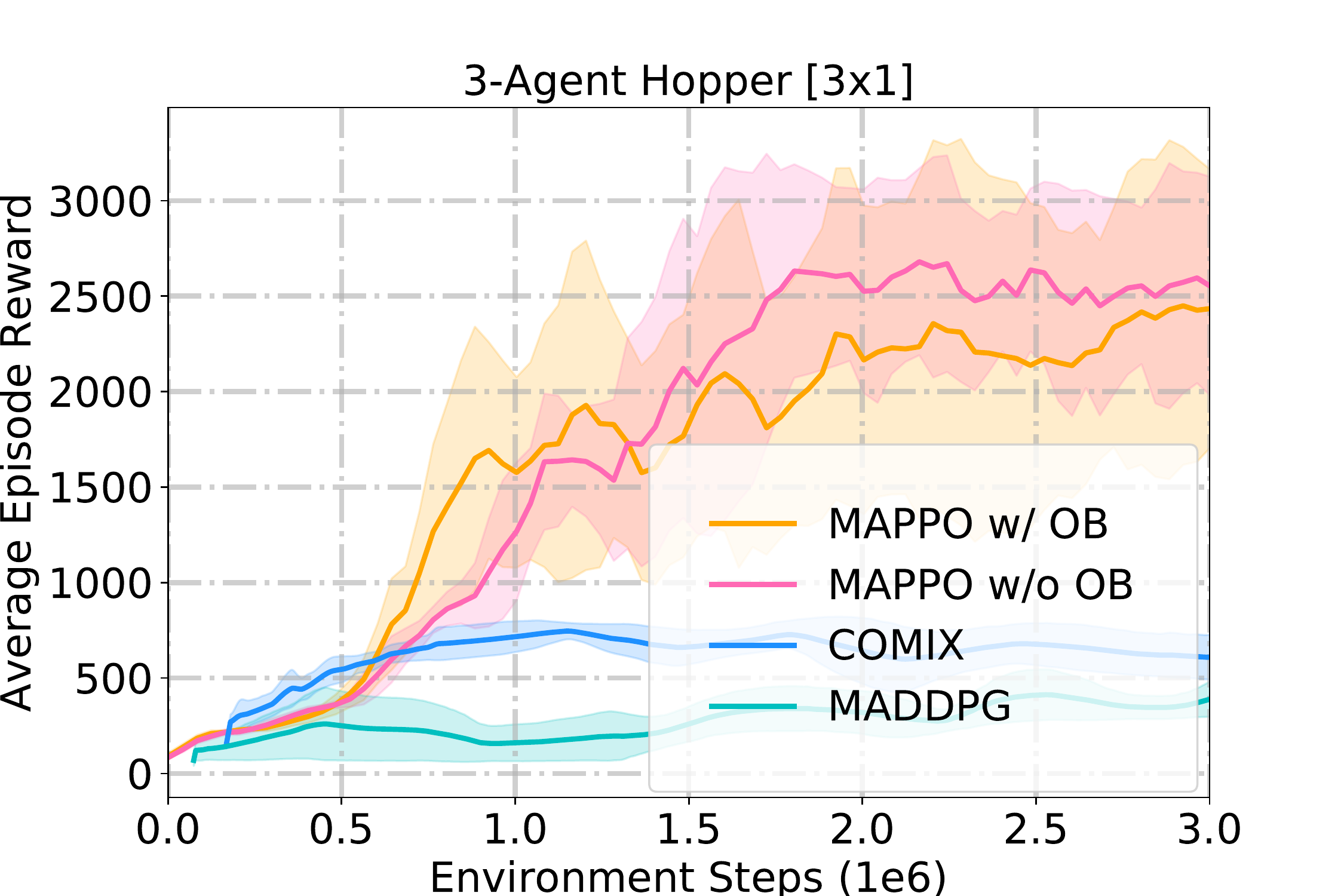}
 			\caption{3-Agent Hopper}
 			\label{fig:exp_hopper_v2}
 		\end{subfigure}%
 		\begin{subfigure}{0.24\linewidth}
 			\centering
 			\includegraphics[width=1.4in]{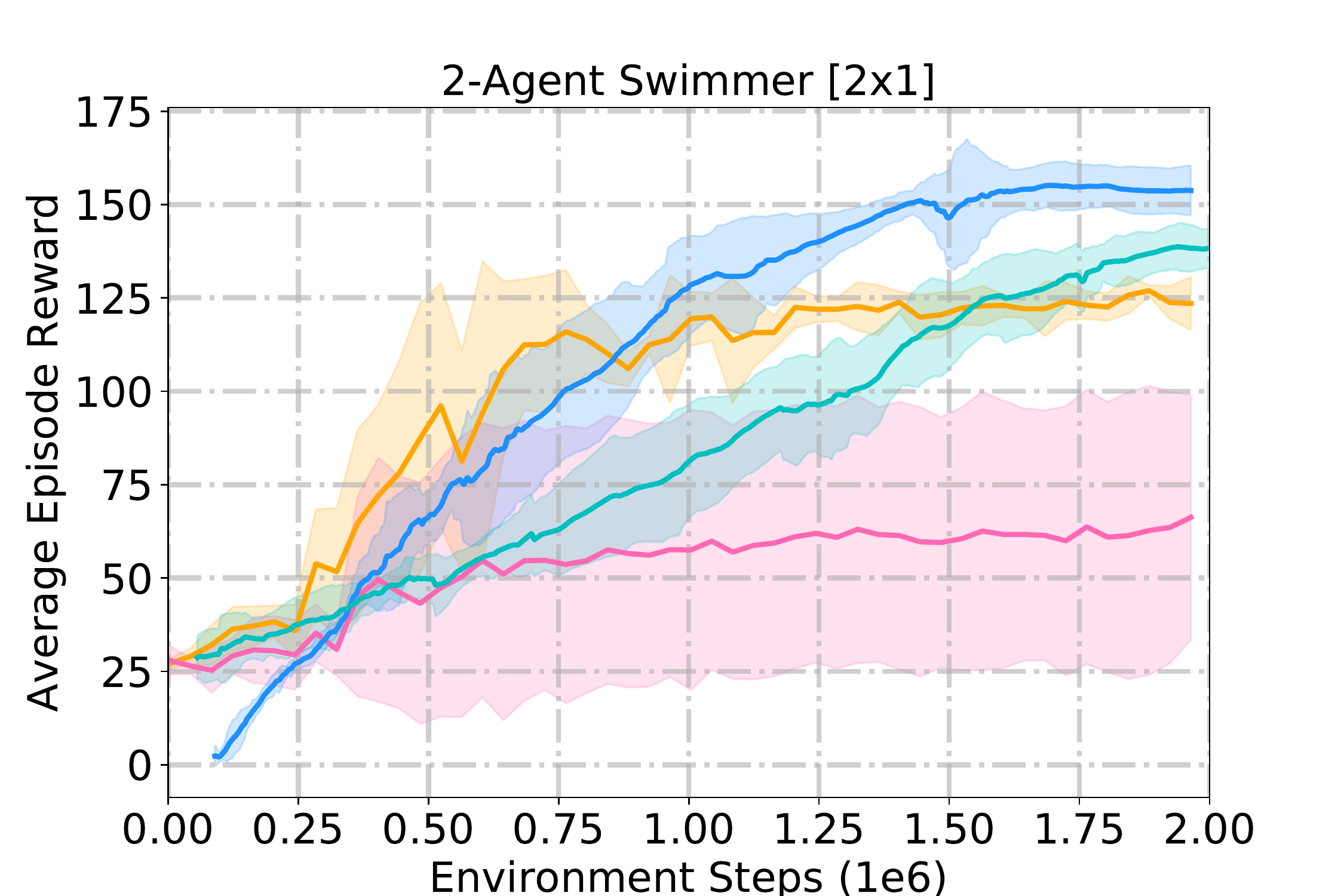} 
 			\caption{2-Agent Swimmer}
 			\label{fig:exp_swimmer_v2}
 		\end{subfigure}
 		\begin{subfigure}{0.24\linewidth}
 			\centering
 			\includegraphics[width=1.4in]{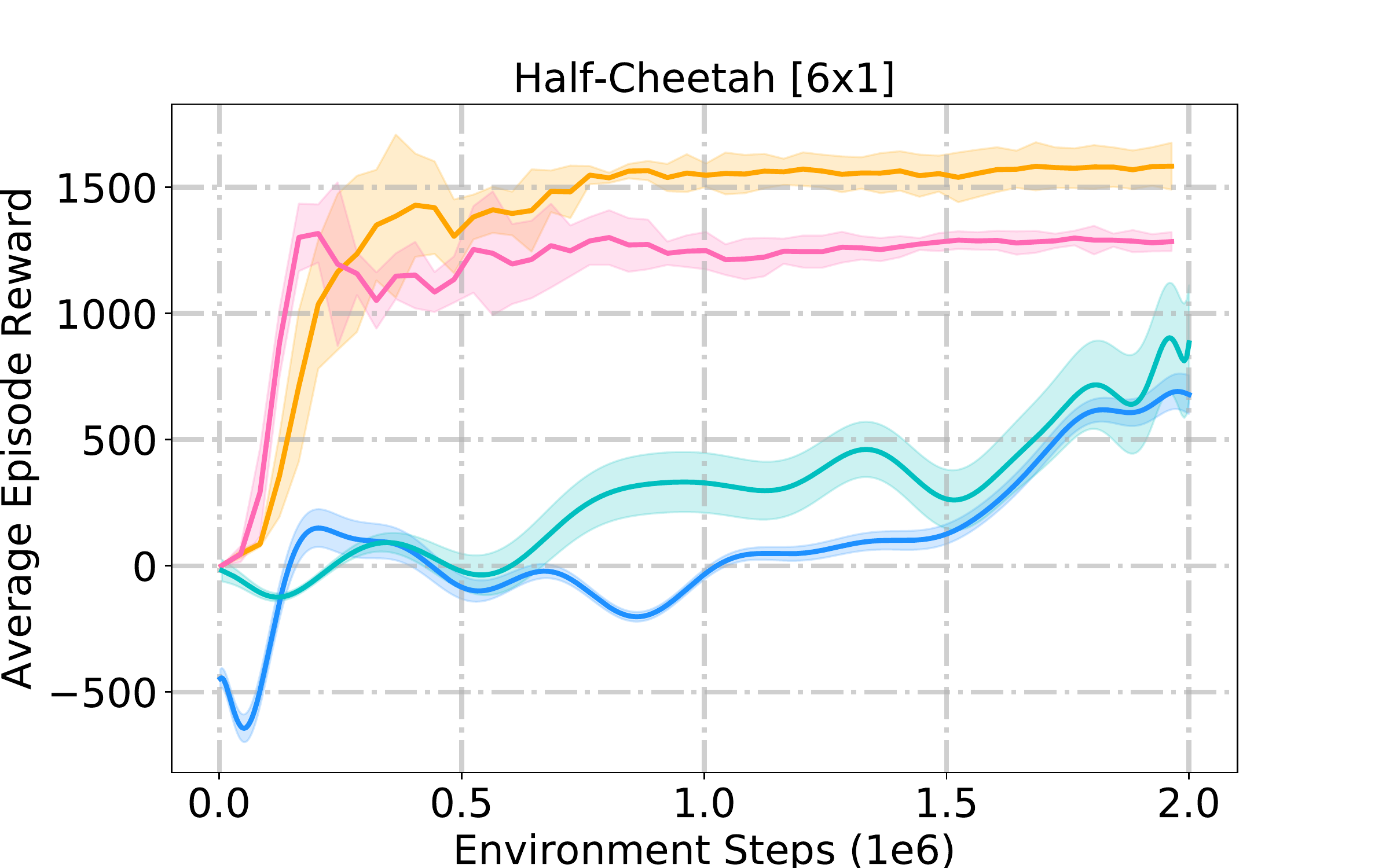}
 			\caption{6-Agent HalfChetaah}
 			\label{fig:exp_halfcheetah_v2}
 		\end{subfigure}%
 		\begin{subfigure}{0.24\linewidth}
 			\centering
 			\includegraphics[width=1.4in]{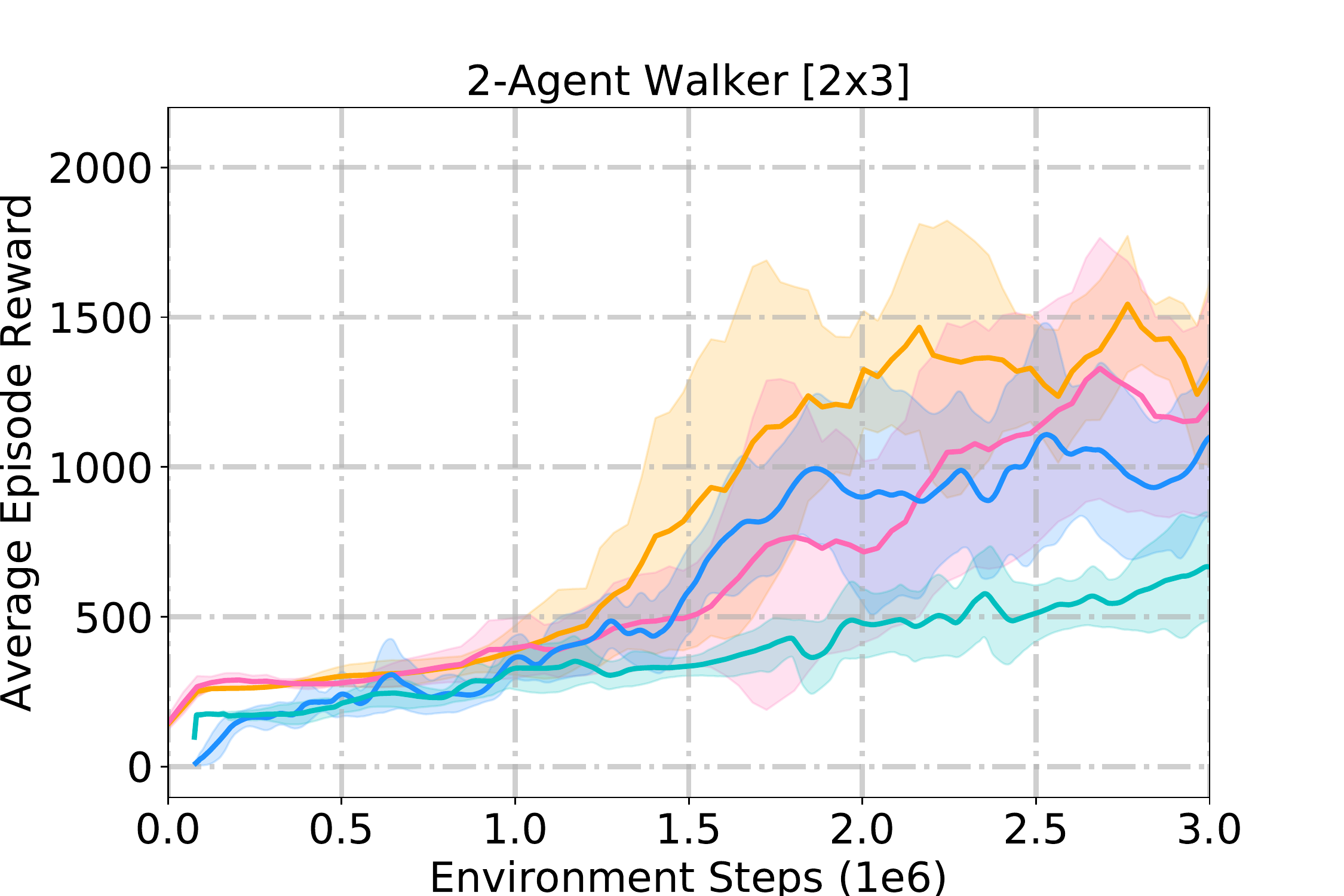} 
 			\caption{2-Agent Walker}
 			\label{fig:exp_walker_v2}
 		\end{subfigure}
 	\quad
 		\begin{subfigure}{0.24\linewidth}
 			\centering
 			\includegraphics[width=1.4in]{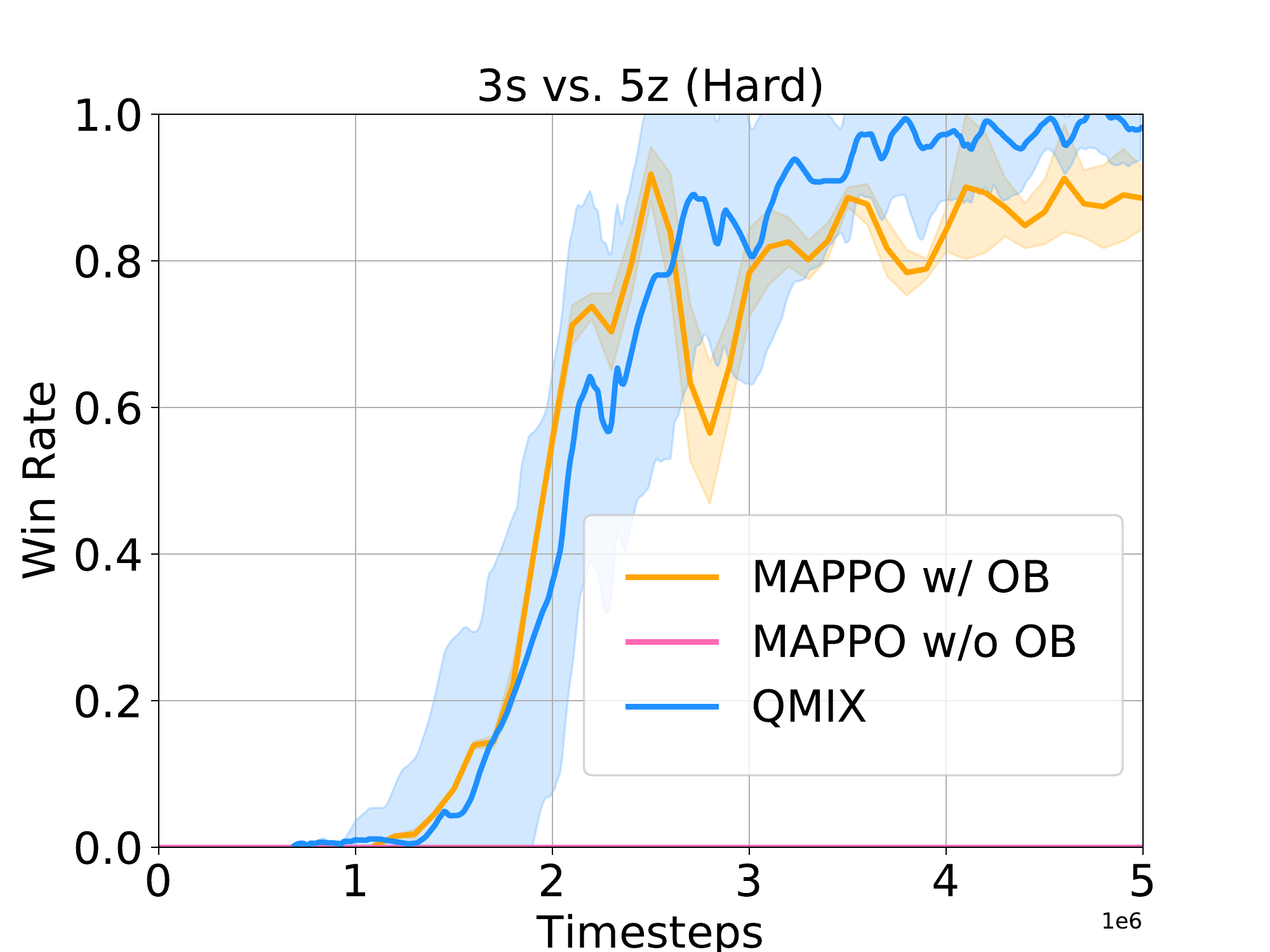}
 			\caption{3 stalkers vs. 5 zealots}
 			\label{fig:exp_3s_vs_5z_win_rate}
 		\end{subfigure}%
 	\begin{subfigure}{0.24\linewidth}
 			\centering
 			\includegraphics[width=1.4in]{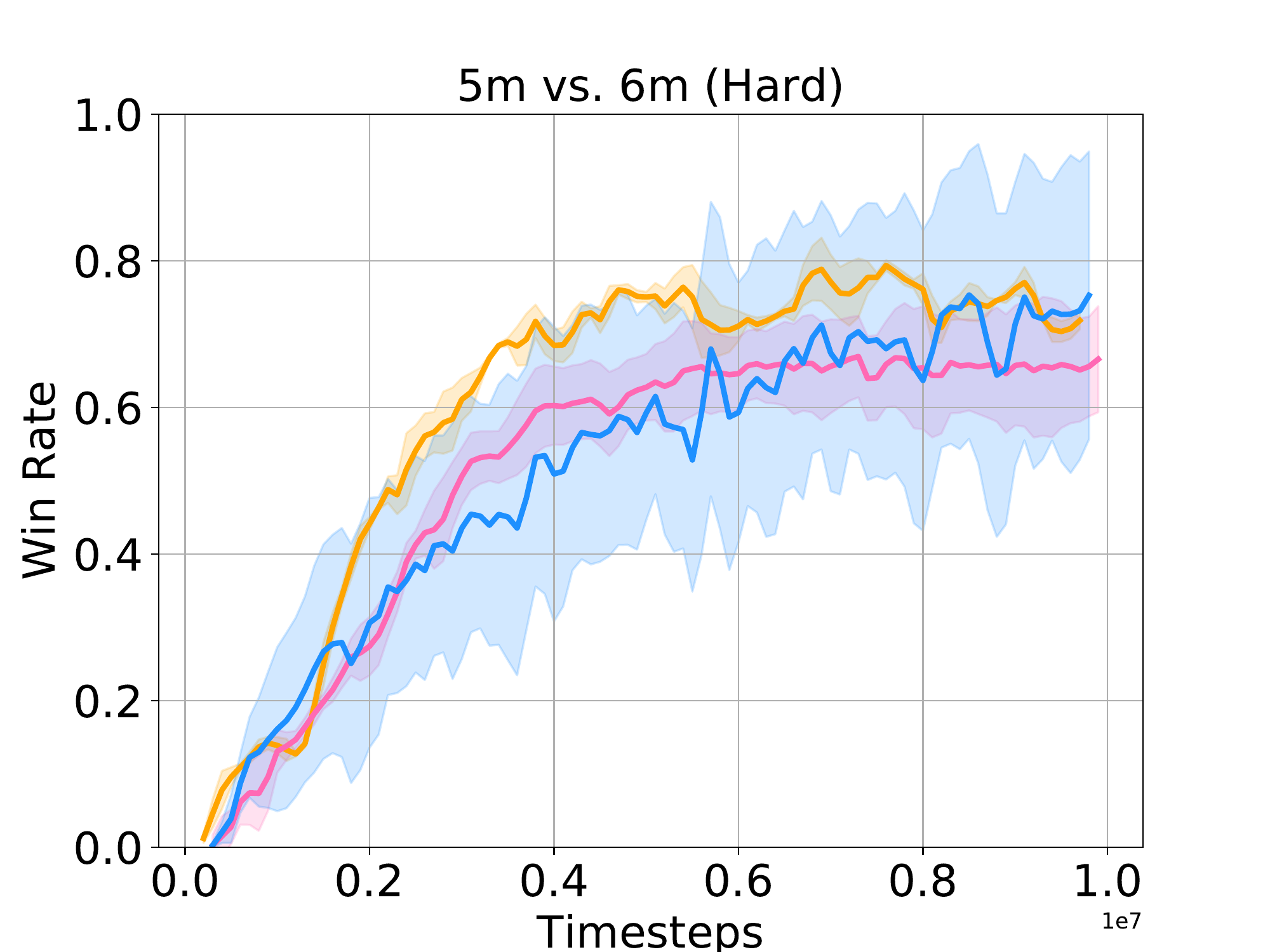} 
 			\caption{5 marines vs. 6 marines}
 			\label{fig:exp_5m_vs_6m_win_rate}
 		\end{subfigure}
 	\begin{subfigure}{0.24\linewidth}
 			\centering
 			\includegraphics[width=1.4in]{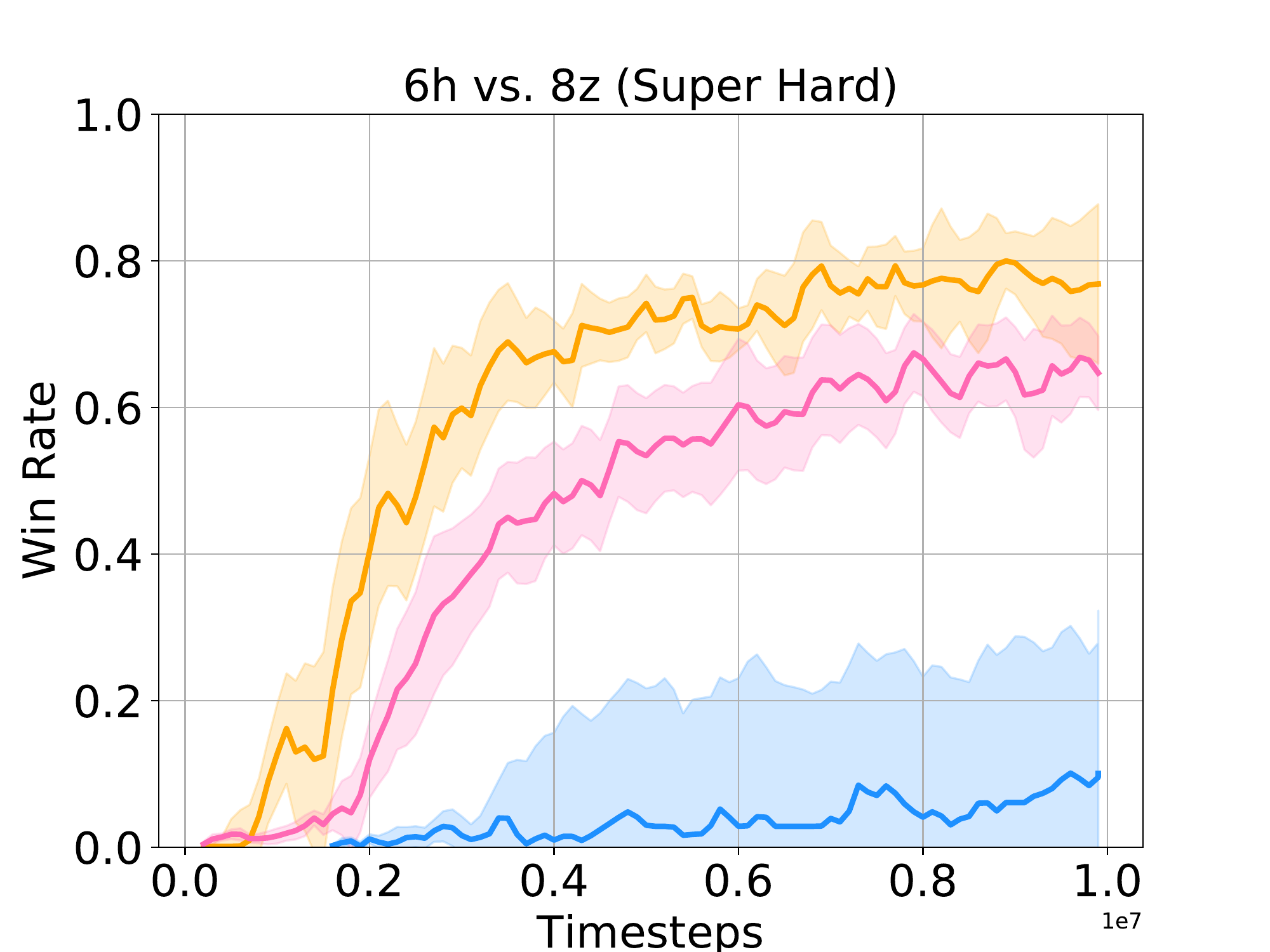}
 			\caption{6 hydralisks vs. 8 zealots}
 			\label{fig:exp_6h_vs_8z_win_rate}
 		\end{subfigure}
 	\begin{subfigure}{0.24\linewidth}
 			\centering
 			\includegraphics[width=1.4in]{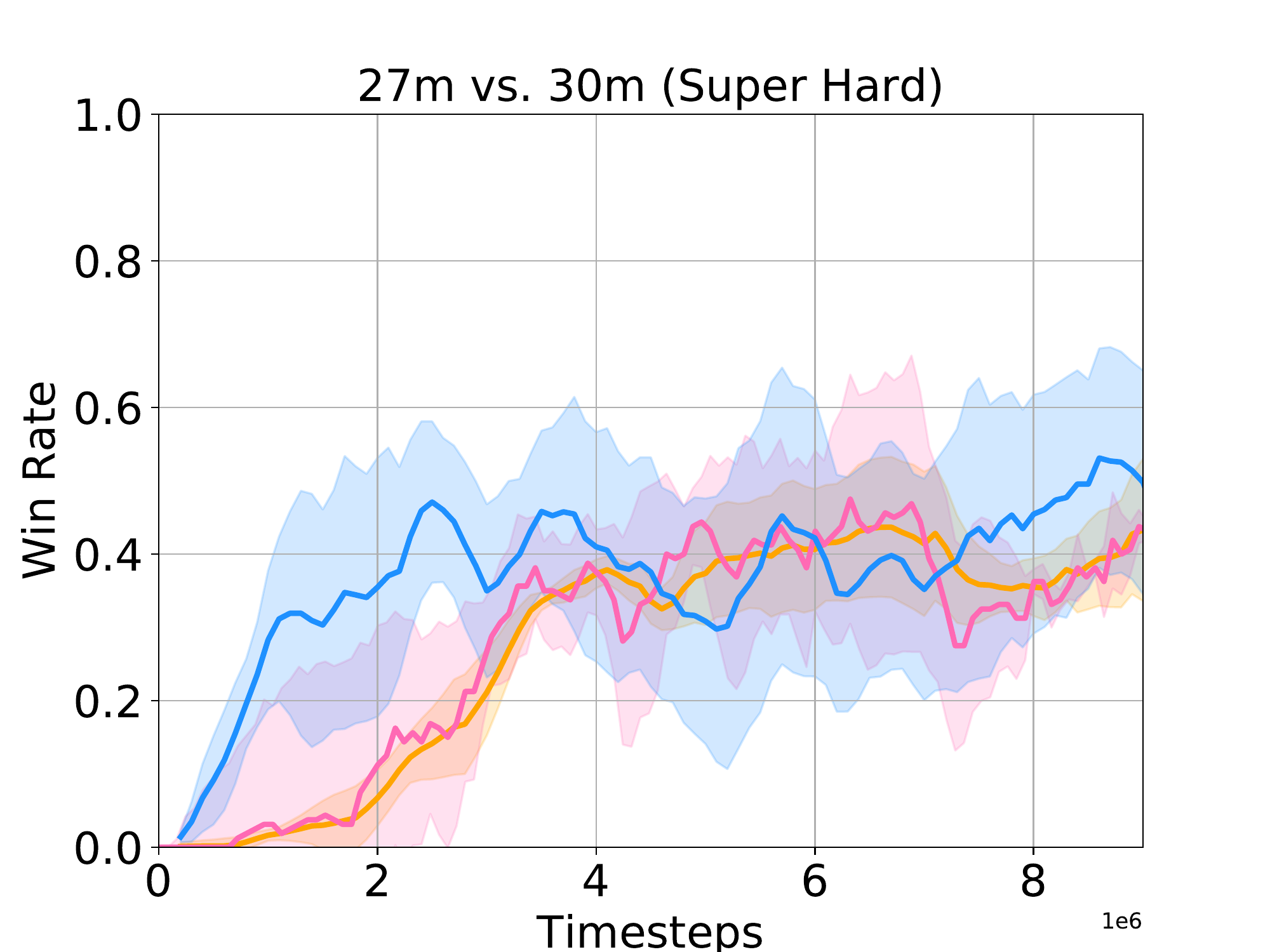} 
 			\caption{27 marines  vs. 30 marines}
 			\label{fig:exp_27m_vs_30m_win_rate}
 		\end{subfigure}
 		%\label{fig:matching:b}
 	%}%
    \vspace{-0pt}
 	\caption{\normalsize Performance comparisons between multi-agent PPO method with and without OB on four  multi-agent MuJoCo tasks and four  (super-)hard SMAC challenges.
% 	 The initial 5 random seeds are used to validate the performance of MAPPO with OB on four challenging Multi-Agent MuJoCo tasks:
% 	(a) 6-Agent \emph{Half-Cheetah},
% 	(b) 3-Agent \emph{Hopper-v2}, (c) 2-Agent \emph{Swimmer-v2}, (d) 2-Agent \emph{Walker2d-v2};
% 	and three hard and super hard SMAC:
% 	(d) the map of \emph{3s\_vs\_5z}, (e) the map of \emph{5m\_vs\_6m}, (f) the map of \emph{6h\_vs\_8z}, 
% 	(g) the map of \emph{27m\_vs\_30m}.
 	} 
 	\label{fig:matching_mappo}
 	\vspace{-15pt}
 \end{figure*}

\captionsetup[table]{font=Large}
% MAPPO VARIANCE in 1e+2, with VAR in 1e+4
\begin{table}[t!]
\vspace{-0pt}
    \begin{adjustbox}{width=\columnwidth, center}
    \begin{threeparttable}
        \vspace{-0pt}
                \caption{\Large Comparisons on the standard deviation of the gradient norm of multi-agent PPO method with and without OB. All quantities are provided in scale $0.01$. Standard errors are provided in brackets. Results suggest OB consistently reduces gradient norms across all eight tasks. }
            \label{table:exp_mappo_variance}
             \begin{tabular}{l | c c c c | c c c c } \toprule
             \\[-1em]
             Method\ / \ Task
             & 3s vs. 5z 
             & 5m vs. 6m 
             & 6h vs. 8z 
             & 27m vs. 30m 
             & 6-Agent HalfCheetah
             & 3-Agent Hopper
             & 2-Agent Swimmer
             & 2-Agent Walker \\ [0.6ex]
             \midrule
             MAPPO w/ OB & \textbf{2.20 (0.17)} 
             & \textbf{1.50 (0.02)} 
             & \textbf{2.33 (0.03)} 
             & \textbf{11.55 (4.80)} 
             & \textbf{30.64 (0.50)}
             & \textbf{79.67 (3.79)} 
             & \textbf{66.72 (3.67)} 
             & \textbf{372.03 (10.83)}\\ 
             MAPPO w/o OB 
             & 6.67 (0.35) 
             & 2.17 (0.07) 
             & 2.54 (0.09) 
             & 19.62 (6.90)
             & 33.65 (2.04)
             & 82.45 (2.79) 
             & 73.54 (11.98) 
             & 405.66 (18.34)\\
             \bottomrule
             \end{tabular}
    \end{threeparttable}
    \end{adjustbox}
    \vspace{-15pt}
\end{table}

\section{Conclusion}
\vspace{-5pt}
In this paper, we try to settle the variance of multi-agent policy gradient (MAPG) estimators. 
 We start our contribution by quantifying, for the first time, the variance of the  MAPG estimator and revealing the  key influencial factors.   
Specifically, we prove that 
the excess variance that a centralised estimator has over its decentralised counterpart grows linearly with the number of agents, and quadratically with agents' local advantages. 
A natural outcome of our analysis is the optimal baseline  (OB) technique. 
We adapt OB to exsiting deep MAPG methods and demonstrate its empirical effectiveness on challenging benchmarks against strong baselines.
In the future, we plan to study other variance reduction techniques that can apply without requiring Q-function critics.    
%
%We solved the large variance issue by deriving the optimal baseline (OB), and turining it into a practical method, whose strength we demonstrated emprirically.

\section*{Author Contributions}
\vspace{-5pt}
We summarise the main contributions from each of the authors as follows: 

\textbf{Jakub Grudzien Kuba}: Theoretical results, algorithm design, code implementation (discrete OB, continuous OB) and paper writing.
\newline
\textbf{Muning Wen}: Algorithm design, code implementation (discrete OB, continuous OB) and  experiments running.
\newline
\textbf{Linghui Meng}: Code implementation (discrete OB) and experiments running.  
\newline
\textbf{Shangding Gu}:  Code implementation (continuous OB) and experiments running.
\newline
\textbf{Haifeng Zhang}: Computational resources and project discussion.
\newline
\textbf{David Mguni}: Project  discussion.
\newline
\textbf{Jun Wang}: Project discussion and overall technical supervision. 
\newline 
\textbf{Yaodong Yang}: Project lead, idea proposing, theory development and experiment design supervision, and whole manuscript writing. Work were done at King's College London.
\section*{Acknowledgements}
The authors from CASIA thank the Strategic Priority Research Program of Chinese Academy of Sciences, Grant No. XDA27030401.

\bibliographystyle{plain}
\bibliography{sample.bib}
\clearpage

\appendices
\startcontents
\printcontents{}{1}{}
\clearpage

\section{Preliminary Remarks}
\label{appendix:remarks}

\begin{remark}
    The multi-agent state-action value function obeys the bounds
    \begin{align}
        \left| Q_{\vtheta}^{i_{1}, \dots, i_{k}}\left(s, \va^{(i_{1}, \dots, i_{k})}\right)
        \right|\leq \frac{\beta}{1-\gamma},  \ \ \  \text{for all } s\in\mathcal{S}, \ \va^{(i_{1}, \dots, i_{k})}\in\mathbfcal{A}^{(i_{1}, \dots, i_{k})}.\nonumber
    \end{align}
\end{remark}

\begin{proof}
It suffices to prove that, for all $t$, the total reward satisfies $\left|R_{t}\right|\leq \frac{\beta}{1-\gamma}$, as the value functions are expectations of it. We have
\begin{align}
    &\left| R_{t} \right| = \left| \sum_{k=0}^{\infty}\gamma^{k}\rr_{t+k} \right| \leq
    \sum_{k=0}^{\infty}\left|\gamma^{k}\rr_{t+k} \right| \leq
    \sum_{k=0}^{\infty}\gamma^{k}\beta 
    = \frac{\beta}{1-\gamma} \nonumber
\end{align}
\end{proof}

\begin{remark}
\label{remark:bounded-maad}
    The multi-agent advantage function is bounded.
\end{remark}
\begin{proof}
    We have
    \begin{align}
        &\left|A_{\vtheta}^{i_{1}, \dots, i_{k}}\left(s,
        \va^{(j_{1}, \dots, j_{m})},
        \va^{(i_{1}, \dots, i_{k})}\right) \right| \nonumber\\
        &= \left|Q_{\vtheta}^{j_{1}, \dots, j_{m}, i_{1}, \dots, i_{k}}\left(s, \va^{(j_{1}, \dots, j_{m}, i_{1}, \dots, i_{k})}\right)  - 
        Q_{\vtheta}^{j_{1}, \dots, j_{m}}\left(s, \va^{(j_{1}, \dots, j_{m}}\right)
        \right|\nonumber\\
        &\leq \left|Q_{\vtheta}^{j_{1}, \dots, j_{m}, i_{1}, \dots, i_{k}}\left(s, \va^{(j_{1}, \dots, j_{m}, i_{1}, \dots, i_{k})}\right)\right| \ 
        + \ \left| Q_{\vtheta}^{j_{1}, \dots, j_{m}}\left(s, \va^{(j_{1}, \dots, j_{m}}\right) \right| \ \leq \ \frac{2\beta}{1-\gamma}\nonumber
    \end{align}
\end{proof}

\begin{remark}
Baselines in MARL have the following property
\begin{align}
    \E_{\rs\sim d^{t}_{\vtheta}, \rva\sim\boldsymbol{\pi_{\vtheta}}}\left[ b\left(\rs, \rva^{-i}\right)\nabla_{\theta^{i}}\log\pi^{i}_{\vtheta}(\ra^{i}|\rs)\right] = \mathbf{0}. \nonumber
\end{align}
\end{remark}

\begin{proof}
\label{proof:baseline-no-bias}
We have 
\begin{align}
    \E_{\rs\sim d^{t}_{\vtheta}, \rva\sim\boldsymbol{\pi_{\vtheta}}}\left[ b\left(\rs, \rva^{-i}\right)\nabla_{\theta^{i}}\log\pi^{i}_{\vtheta}(\ra^{i}|\rs)\right]
    =
    \E_{\rs\sim d^{t}_{\vtheta}, \rva^{-i}\sim\boldsymbol{\pi}^{-i}_{\vtheta}}\left[ 
    b\left(\rs, \rva^{-i}\right)
    \E_{\ra^{i}\sim\pi^{i}_{\vtheta}}
    \left[
    \nabla_{\theta^{i}}\log\pi^{i}_{\vtheta}(\ra^{i}|\rs)\right]
    \right],\nonumber
\end{align}
which means that it suffices to prove that for any $s\in\mathcal{S}$
\begin{align}
    \E_{\ra^{i}\sim\pi^{i}_{\vtheta}}
    \left[
    \nabla_{\theta^{i}}\log\pi^{i}_{\vtheta}(\ra^{i}|s)\right] = \bold{0}.\nonumber
\end{align}
We prove it for continuous $\mathcal{A}^{i}$. The discrete case is analogous.
\begin{align}
    &\E_{\ra^{i}\sim\pi^{i}_{\vtheta}}
    \left[
    \nabla_{\theta^{i}}\log\pi^{i}_{\vtheta}(\ra^{i}|s)\right] = \int_{\mathcal{A}^{i}}\pi^{i}_{\vtheta}\left(
    a^{i}|s
    \right) \nabla_{\theta^{i}}\log\pi^{i}_{\vtheta}(\ra^{i}|s)\ da^{i}\nonumber\\
    &= \int_{\mathcal{A}^{i}}\nabla_{\theta^{i}}\pi^{i}_{\vtheta}\left(
    a^{i}|s\right)\ da^{i} = \nabla_{\theta^{i}}\int_{\mathcal{A}^{i}}\pi^{i}_{\vtheta}\left(
    a^{i}|s\right)\ da^{i} = \nabla_{\theta^{i}}(1) = \bold{0}\nonumber
\end{align}
\end{proof}

\newpage

\section{Proofs of the Theoretical Results}
\label{appendix:proof-of-results}
\subsection{Proofs of Lemmas \ref{lemma:ma-advantage-decomp}, \ref{lemma:advantage-variance-decomp},
and \ref{lemma:advantage-var-upper-bound}}
In this subsection, we prove the lemmas stated in the paper. We realise that their application to other, very complex, proofs is not always immediately clear. To compensate for that, we provide the stronger versions of the lemmas; we give a detailed proof of the strong version of Lemma \ref{lemma:ma-advantage-decomp} which is supposed to demonstrate the equivalence of the normal and strong versions, and prove the normal versions of Lemmas \ref{lemma:advantage-variance-decomp} \& \ref{lemma:advantage-var-upper-bound}, and state their stronger versions as remarks to the proofs.

\maadlemma*

\begin{proof}
    \label{proof:ma-advantage-decomp}
    We prove a slightly \textbf{stronger}, but perhaps less telling, version of the lemma, which is
    \begin{align}
        \label{lemma:maad-stronger}
        A_{\vtheta}^{k+1, \dots, m}\left(s, \va^{(1, \dots, k)} , \va^{(k+1, \dots, m)}\right) = \sum_{i=k+1}^{m}A_{\vtheta}^{i}\left(s, \va^{(1, \dots, i-1)}, a^{i}\right).
    \end{align}
    The original form of the lemma will follow from the above by taking $k=0$. \\
    By the definition of the multi-agent advantage, we have
    \begin{align}
        &A^{k+1, \dots, m}_{\vtheta}\left(s, \va^{(1, \dots, k)}, \va^{(k+1, \dots,m)}\right)\nonumber \\
        &= 
        Q^{1, \dots,k, k+1, \dots, m}_{\vtheta}\left(s, \va^{(1, \dots, k, k+1, \dots,  m)}\right) 
        - Q^{1, \dots, k}_{\vtheta}\left(s, \va^{(1, \dots, k)}\right)
        \nonumber
    \end{align}
    which can be written as a telescoping sum
    \begin{align}
        & Q^{1, \dots,k, k+1, \dots, m}_{\vtheta}\left(s, \va^{(1, \dots, k, k+1, \dots,  m)}\right) 
        - Q^{1, \dots, k}_{\vtheta}\left(s, \va^{(1, \dots, k)}\right) \nonumber\\
        &= \sum_{i=k+1}^{m}\left[ Q_{\vtheta}^{(1, \dots, i)}\left(s, \va^{(1, \dots, i)}\right) - Q_{\vtheta}^{(1, \dots, i-1)}\left(s, \va^{(1, \dots, i-1)}\right)\right] \nonumber\\
        &= \sum_{i=k+1}^{m}A_{\theta}^{i}\left(s, \va^{(1, \dots, i-1)}, a^{i}\right) \nonumber
    \end{align}
\end{proof}

\avdlemma*

\begin{proof}
    \label{proof:advantage-variance-decomp}
    The trick of this proof is to develop a relation on the variance of multi-agent advantage which is recursive over the number of agents. We have
    \begin{align}
        &\mathbf{Var}_{\rva\sim\boldsymbol{\pi}_{\vtheta}}\left[ A_{\vtheta}(s, \rva) \right] = 
        \mathbf{Var}_{
        \ra^{1}\sim\pi^{1}_{\vtheta}, \dots, \ra^{v}\sim\pi_{\vtheta}^{v}
        }\left[ A^{1, \dots, n}_{\vtheta}\left(s, \ra^{(1, \dots, n)}\right) \right] \nonumber\\
        &= \E_{\ra^{1}\sim\pi^{1}_{\vtheta}, \dots, \ra^{n}\sim\pi^{n}_{\vtheta}}\left[ A^{1, \dots, n}_{\vtheta}\left(s, \rva^{(1, \dots, n)}\right)^{2} \right] \nonumber\\
        &= \E_{\ra^{1}\sim\pi^{1}_{\vtheta}, \dots, \ra^{n-1}\sim\pi^{n-1}_{\vtheta}}\Bigg[ \E_{\ra^{n}\sim\pi^{n}_{\vtheta}}\left[
        A^{1, \dots, n}_{\vtheta}\left(s, \rva^{(1, \dots, n)}\right)^{2}
        \right] \nonumber\\
        &\ \ \ \ \ \ \ \ \ \ \ \ - \E_{\ra^{n}\sim\pi^{n}_{\theta}}\left[A^{1, \dots, n}_{\vtheta}\left(s, \rva^{(1, \dots, n)}\right)\right]^{2}
        +
        \E_{\ra^{n}\sim\pi^{n}_{\theta}}\left[A^{1, \dots, n}_{\vtheta}\left(s, \rva^{(1, \dots, n)}\right)\right]^{2} \Bigg]\nonumber\\
        &= \E_{\ra^{1}\sim\pi^{1}_{\vtheta}, \dots, \ra^{n-1}\sim\pi^{n-1}_{\vtheta}}\left[ 
        \mathbf{Var}_{\ra^{n}\sim\pi^{n}_{\theta}}\left[ 
        A^{1, \dots, n}_{\vtheta}\left(s, \rva^{(1, \dots, n)}\right)
        \right] \right]\nonumber\\
        &\ \ \ + 
        \E_{\ra^{1}\sim\pi^{1}_{\vtheta}, \dots, \ra^{n-1}\sim\pi^{n-1}_{\vtheta}}\left[ 
        A^{1, \dots, n-1}_{\vtheta}\left(s, \rva^{(1, \dots, n-1)}\right)^{2}
        \ \right]
        \nonumber
    \end{align}
    which, by the stronger version of Lemma \ref{lemma:ma-advantage-decomp}, given by Equation \ref{lemma:maad-stronger}, applied to the first term, equals
    \begin{align}
        &\E_{\ra^{1}\sim\pi^{1}_{\vtheta}, \dots, \ra^{n-1}\sim\pi^{n-1}_{\vtheta}}\left[ 
        \mathbf{Var}_{\ra^{n}\sim\pi^{n}_{\theta}}\left[ 
        A^{n}_{\vtheta}\left(s, \rva^{(1, \dots, n-1)}, \ra^{n}\right)
        \right] \right]\nonumber\\
        &\ \ \ + 
        \E_{\ra^{1}\sim\pi^{1}_{\vtheta}, \dots, \ra^{n-1}\sim\pi^{n-1}_{\vtheta}}\left[ 
        A^{1, \dots, n-1}_{\vtheta}\left(s, \rva^{(1, \dots, n-1)}\right)^{2}
        \ \right]
        \nonumber
    \end{align}
    Hence, we have a recursive relation
    \begin{align}
        &\mathbf{Var}_{
        \ra^{1}\sim\pi^{1}_{\vtheta}, \dots, \ra^{v}\sim\pi_{\vtheta}^{v}
        }\left[ A^{1, \dots, n}_{\vtheta}\left(s, \ra^{(1, \dots, n)}\right) \right] \nonumber\\
        &=\E_{\ra^{1}\sim\pi^{1}_{\vtheta}, \dots, \ra^{n-1}\sim\pi^{n-1}_{\vtheta}}\left[ 
        \mathbf{Var}_{\ra^{n}\sim\pi^{n}_{\vtheta}}\left[ 
        A^{n}_{\vtheta}\left(s, \rva^{(1, \dots, n-1)}, \ra^{n}\right)
        \right] \right]\nonumber\\
        &\ \ \ + 
        \mathbf{Var}_{\ra^{1}\sim\pi^{1}_{\vtheta}, \dots, \ra^{n-1}\sim\pi^{n-1}_{\vtheta}}\left[ 
        A^{1, \dots, n-1}_{\vtheta}\left(s, \rva^{(1, \dots, n-1)}\right)
        \ \right]
        \nonumber
    \end{align}
    from which we can obtain
    \begin{align}
        &\mathbf{Var}_{
        \ra^{1}\sim\pi^{1}_{\vtheta}, \dots, \ra^{v}\sim\pi_{\vtheta}^{v}
        }\left[ A^{1, \dots, n}_{\vtheta}\left(s, \ra^{(1, \dots, n)}\right) \right] \nonumber\\
        &= \sum_{i=1}^{n}\E_{\ra^{1}\sim\pi^{1}_{\vtheta}, \dots, \ra^{i-1}\sim\pi^{i-1}_{\vtheta}}\left[ 
        \mathbf{Var}_{\ra^{i}\sim\pi^{i}_{\vtheta}}\left[ 
        A^{i}_{\vtheta}\left(s, \rva^{(1, \dots, i-1)}, \ra^{i}\right)
        \right] \right]\nonumber
    \end{align}

\end{proof}
\begin{remark}
    Lemma \ref{lemma:advantage-variance-decomp} has a \textbf{stronger} version, coming as a corollary to the above proof; that is
    \begin{align}
        &\mathbf{Var}_{\ra^{k+1}\sim\pi^{k+1}_{\vtheta}, \dots, \ra^{n}\sim\pi^{n}_{\vtheta}}
        \left[ A^{k+1, \dots, n}\left(s, \va^{(1, \dots, k)}, \rva^{(k+1, \dots, n)}  \right) \right] \nonumber \\
        &= \sum_{i=k+1}^{n}\E_{\ra^{k+1}\sim\pi^{k+1}_{\vtheta}, \dots, \ra^{i-1}\sim\pi^{i-1}_{\vtheta}}\left[ 
        \mathbf{Var}_{\ra^{i}\sim\pi^{i}_{\vtheta}}
        \left[ A^{i}_{\vtheta}\left(s, \va^{1, \dots, k}, \rva^{k+1, \dots, i-1} , \ra^{i} \right) \right]
        \right].
    \end{align}
    We think of it as a corollary to the proof of the lemma, as the fixed joint action $\va^{1, \dots, k}$ has the same algebraic properites, throghout the proof, as state $s$.
    
\end{remark}

\avublemma*

\begin{proof}
    \label{proof:advantage-var-upper-bound}
    By Lemma \ref{lemma:advantage-variance-decomp}, we have
    \begin{align}
        \label{eq:ref-var-decomp}
        \mathbf{Var}_{\rva\sim\boldsymbol{\pi}_{\vtheta}}\big[ A_{\vtheta}(s, \rva) \big] =
        \sum_{i=1}^{n}\E_{\ra^{1}\sim\pi^{1}_{\vtheta}, \dots, \ra^{i-1}\sim\pi_{\vtheta}^{i-1}}\left[ \mathbf{Var}_{\ra^{i}\sim\pi^{i}_{\vtheta}}\left[ A_{\vtheta}^{i}\left(s, \rva^{(1, \dots, i-1)}, \ra^{i}\right) \right] \right] 
    \end{align}
    Take an arbitrary $i$. We have
    \begin{align}
        &\E_{\ra^{1}\sim\pi^{1}_{\vtheta}, \dots, \ra^{i-1}\sim\pi_{\vtheta}^{i-1}}\left[ \mathbf{Var}_{\ra^{i}\sim\pi^{i}_{\vtheta}}\left[ A_{\vtheta}^{i}\left(s, \rva^{(1, \dots, i-1)}, \ra^{i}\right) 
        \right] \right] \nonumber\\
        &= \ \E_{\ra^{1}\sim\pi^{1}_{\vtheta}, \dots, \ra^{i-1}\sim\pi_{\vtheta}^{i-1}}\left[ \E_{\ra^{i}\sim\pi^{i}_{\vtheta}}\left[ A_{\vtheta}^{i}\left(s, \rva^{(1, \dots, i-1)}, \ra^{i}\right)^{2} 
        \right] \right] \nonumber\\
        &= \ \E_{\ra^{1}\sim\pi^{1}_{\vtheta}, \dots, \ra^{i-1}\sim\pi_{\vtheta}^{i-1}}\left[ \E_{\ra^{i}\sim\pi^{i}_{\vtheta}}\left[ \E_{\ra^{i+1}\sim\pi^{i+1}_{\vtheta}, \dots, \ra^{n}\sim\pi^{n}_{\vtheta}}\left[
        A_{\vtheta}^{i, \dots, n}
        \left(s, \rva^{(1, \dots, i-1)}, \rva^{(i, \dots, n)}\right)
        \right]^{2} 
        \right] \right]
        \nonumber\\
        &\leq  \ 
        \E_{\ra^{1}\sim\pi^{1}_{\vtheta}, \dots, \ra^{i-1}\sim\pi_{\vtheta}^{i-1}}\left[ \E_{\ra^{i}\sim\pi^{i}_{\vtheta}}\left[ \E_{\ra^{i+1}\sim\pi^{i+1}_{\vtheta}, \dots, \ra^{n}\sim\pi^{n}_{\vtheta}}\left[
        A_{\vtheta}^{i, \dots, n}
        \left(s, \rva^{(1, \dots, i-1)}, \rva^{(i, \dots, n)}\right)^{2}
        \right]
        \right] \right]
        \nonumber\\
        &= \
        \E_{\ra^{1}\sim\pi^{1}_{\vtheta}, \dots, \ra^{i-1}\sim\pi^{i-1}_{\vtheta},\ra^{i+1}\sim\pi^{i+1}_{\vtheta}, \dots, \ra^{n}\sim\pi^{n}_{\vtheta}}
        \left[  
        \E_{\ra^{i}\sim\pi^{i}_{\vtheta}}
        \left[
         A_{\vtheta}^{i, \dots, n}
        \left(s, \rva^{(1, \dots, i-1)}, \rva^{(i, \dots, n)}\right)^{2}
        \right]  \right]
        \nonumber
    \end{align}
    The above can be equivalently, but more tellingly, rewritten after permuting (cyclic shift) the labels of agents, in the following way
    \begin{align}
        &\E_{\ra^{-i}\sim\pi^{-i}_{\vtheta}}
        \left[  
        \E_{\ra^{i}\sim\pi^{i}_{\vtheta}}
        \left[
        A_{\vtheta}^{i+1, \dots, n, i}\left
        (s, \rva^{(1, \dots, i-1)}, \rva^{(i+1, \dots, n, i)}\right)^{2}
        \right]  \right]
        \nonumber\\
        &= \E_{\ra^{-i}\sim\pi^{-i}_{\vtheta}}
        \left[  
        \mathbf{Var}_{\ra^{i}\sim\pi^{i}_{\vtheta}}
        \left[
        A_{\vtheta}^{i+1, \dots, n, i}\left
        (s, \rva^{(1, \dots, i-1)}, \rva^{(i+1, \dots, n, i)}\right)
        \right]  \right]
        \nonumber
    \end{align}
    which, by the strong version of Lemma \ref{lemma:ma-advantage-decomp}, equals
    \begin{align}
        \E_{\rva^{-i}\sim\pi^{-i}_{\vtheta}}
        \left[  
        \mathbf{Var}_{\ra^{i}\sim\pi^{i}_{\theta}}
        \left[
        A_{\vtheta}^{i}\left(s, \rva^{-i}, \ra^{i}\right)
        \right]  
        \right]
        \nonumber
    \end{align}
    which can be further simplified by
    \begin{align}
        &\E_{\rva^{-i}\sim\pi^{-i}_{\vtheta}}
        \left[  
        \mathbf{Var}_{\ra^{i}\sim\pi^{i}_{\vtheta}}
        \left[
        A_{\vtheta}^{i}\left(s, \rva^{-i}, \ra^{i}\right)
        \right]  
        \right] = 
        \E_{\rva^{-i}\sim\pi^{-i}_{\vtheta}}
        \left[  
        \E_{\ra^{i}\sim\pi^{i}_{\vtheta}}
        \left[
        A_{\vtheta}^{i}\left(s, \rva^{-i}, \ra^{i}\right)^{2}
        \right]  
        \right]\nonumber\\
        &= 
        \E_{\rva\sim\boldsymbol{\pi}_{\vtheta}}
        \left[
        A_{\vtheta}^{i}\left(s, \rva^{-i}, \ra^{i}\right)^{2}
        \right] 
         = \mathbf{Var}_{\rva\sim\boldsymbol{\pi}_{\vtheta}}
        \left[
        A_{\vtheta}^{i}\left(s, \rva^{-i}, \ra^{i}\right)
        \right] 
         \nonumber
    \end{align}
    which, combined with Equation \ref{eq:ref-var-decomp}, finishes the proof.
\end{proof}
\begin{remark}
Again, subsuming a joint action $\va^{(1, \dots, k)}$ into state in the above proof, we can have a \textbf{stronger} version of Lemma \ref{lemma:advantage-var-upper-bound}, 
\begin{align}
    \label{eq:advantage-var-upper-bound-stronger}
    &\mathbf{Var}_{\ra^{k+1}\sim\pi^{k+1}_{\vtheta}, \dots, \ra^{n}\sim\pi^{n}_{\vtheta}}\left[ 
    A^{k+1, \dots, n}_{\vtheta}\left(s, \va^{(1, \dots, k)}, \rva^{(k+1, \dots, n)} \right)   \right] \nonumber \\
    &\leq \sum_{i=k+1}^{n}\mathbf{Var}_{\ra^{k+1}\sim\pi^{k+1}_{\vtheta}, \dots,  \ra^{n}\sim\pi^{n}_{\vtheta}}
    \left[ A^{i}\left( s, \va^{(k+1, \dots, i-1, i+1, \dots, n)}, \ra^{i}\right) \right]
\end{align}
\end{remark}

\subsection{Proofs of Theorems \ref{theorem:excess-variance} and \ref{theorem:coma-dt}}
Let us recall the two assumptions that we make in the paper.

\assumespace*

\assumedifferen*

These assumptions assure that the supremum $\sup_{s, a^{i}}\left|\left| \nabla_{\theta^{i}}\log\pi^{i}_{\vtheta}(a^{i}|s)\right|\right|$ exists for every agent $i$. We notice that the supremum $\sup_{s, \va^{-i}, a^{i}}\left|A^{i}(s, \va^{-i}, a^{i})\right|$ exists regardless of assumptions, as by Remark \ref{remark:bounded-maad}, the multi-agent advantage is bounded from both sides.

\evtheorem*

\begin{proof}
    \label{proof:grad-var-bound}
    It suffices to prove the first inequality, as the second one is a trivial upper bound. Let's consider an arbitrary time step $t\geq 0$. Let 
    \begin{align}
        &\rvg^{i}_{\text{C}, t} = \hat{Q}(\rs_{t}, \rva_{t})\nabla_{\theta^{i}}\log\pi^{i}_{\vtheta}(\rs_{t}, \ra^{i}_{t})\nonumber\\
        &\rvg^{i}_{\text{D}, t} = \hat{Q}^{i}(\rs_{t}, \ra^{i}_{t})\nabla_{\theta^{i}}\log\pi^{i}_{\vtheta}(\rs_{t}, \ra^{i}_{t})\nonumber
    \end{align}
    be the contributions to the centralised and decentralised gradient estimators coming from sampling $\rs_{t}\sim d_{\vtheta}^{t}$, $\rva_{t}\sim\boldsymbol{\pi}_{\vtheta}$. Note that
    \begin{align}
        &\rvg^{i}_{\text{C}} = \sum_{t=0}^{\infty}\gamma^{t}\rvg^{i}_{\text{C}, t} \ \ \ \ \ \ \ \ \text{and}  \ \ \ \ \ \ \ \ 
        \rvg^{i}_{\text{D}} = \sum_{t=0}^{\infty}\gamma^{t}\rvg^{i}_{\text{D}, t}\nonumber
    \end{align}
    Moreover, let $\rg^{i}_{\text{C}, t, j}$ and $\rg^{i}_{\text{D}, t, j}$ be the $j^{th}$ components of $\rvg^{i}_{\text{C}, t}$ and $\rvg^{i}_{\text{D}, t}$, respectively.
    Using the law of total variance, we have
    \begin{align}
        \label{eq:var-to-exp}
        &\mathbf{Var}_{\rs\sim d_{\vtheta}^{t}, \rva\sim\boldsymbol{\pi}_{\vtheta}}\left[ 
        \rg^{i}_{\text{C}, t, j} \right] 
        - 
        \mathbf{Var}_{\rs\sim d_{\vtheta}^{t}, \rva\sim\boldsymbol{\pi}_{\vtheta}}\left[ \rg^{i}_{\text{D}, t, j} \right]
        \nonumber\\
        &= \left( \mathbf{Var}_{\rs\sim d_{\vtheta}^{t}}\left[ \E_{\rva\sim\boldsymbol{\pi}_{\vtheta}}\left[ \rg^{i}_{C, t, j} \right] \right] 
        +
        \E_{\rs\sim d_{\vtheta}^{t}}\left[ \mathbf{Var}_{\rva\sim\boldsymbol{\pi}_{\vtheta}}
        \left[ \rg^{i}_{\text{C}, t, j} \right] 
        \right] \right) 
        \nonumber\\
        & \ \ \ \ \ \ \ - \left( \mathbf{Var}_{\rs\sim d_{\vtheta}^{t}}\left[ \E_{\rva\sim\boldsymbol{\pi}_{\vtheta}}\left[ \rg^{i}_{\text{D}, t, j} \right] \right] +
        \E_{\rs\sim d_{\vtheta}^{t}}\left[ \mathbf{Var}_{\rva\sim\boldsymbol{\pi}_{\vtheta}}\left[ \rg^{i}_{\text{D}, t, j} \right] \right] \right)
    \end{align}
    Noting that $\rvg^{i}_{\text{C}}$ and $\rvg^{i}_{\text{D}}$ have the same expectation over $\rva\sim\boldsymbol{\pi}_{\vtheta}$, the above simplifies to 
    \begin{align}
        \label{eq:exp-to-state}
        &\E_{\rs\sim d_{\vtheta}^{t}}\left[ \mathbf{Var}_{\rva\sim\boldsymbol{\pi}_{\vtheta}}\left[ \rg^{i}_{\text{C}, t, j} \right] 
        \right] - 
        \E_{\rs\sim d_{\vtheta}^{t}}\left[ \mathbf{Var}_{\rva\sim\boldsymbol{pi}_{\vtheta}}\left[ 
        \rg^{i}_{\text{D}, t, j} 
        \right] \right] \nonumber\\
        &= \E_{\rs\sim d_{\vtheta}^{t}}\left[ \mathbf{Var}_{\rva\sim\boldsymbol{\pi}_{\vtheta}}\left[ \rg^{i}_{\text{C}, t, j} \right] - \mathbf{Var}_{\rva\sim\pi_{\theta}}\left[ \rg^{i}_{\text{D}, t, j} \right] 
        \right]
    \end{align}
    Let's fix a state $s$. Using (again) the fact that the expectations of the two gradients are the same, we have
    \begin{align}
        &\mathbf{Var}_{\rva\sim\boldsymbol{\pi}_{\vtheta}}\left[ \rg^{i}_{\text{C}, t, j} \right] - \mathbf{Var}_{\rva\sim\boldsymbol{\pi}_{\vtheta}}\left[ \rg^{i}_{\text{D}, t, j} \right]\nonumber\\
        &= \left( 
        \E_{\rva\sim\boldsymbol{\pi}_{\vtheta}}\left[ \left(\rg^{i}_{\text{C}, t, j}\right)^{2} \right]
        -
        \E_{\rva\sim\boldsymbol{\pi}_{\vtheta}}\left[ \rg^{i}_{\text{C}, t, j} \right]^{2}
        \right)
        -
        \left( 
        \E_{\rva\sim\boldsymbol{\pi}_{\vtheta}}\left[ \left(\rg^{i}_{\text{D}, t, j}\right)^{2} \right]
        -
        \E_{\rva\sim\boldsymbol{\pi}_{\vtheta}}\left[ \rg^{i}_{\text{D}, t, j} \right]^{2}
        \right)\nonumber\\
        &= \E_{\rva\sim\boldsymbol{\pi}_{\vtheta}}\left[ \left(\rg^{i}_{\text{C}, t, j}\right)^{2}\right] -  \E_{\rva\sim\boldsymbol{\pi}_{\vtheta}}\left[ \left(\rg^{i}_{\text{D}, t, j}\right)^{2} \right]
        \nonumber\\
        &= \E_{\rva\sim\boldsymbol{\pi}_{\vtheta}}\left[ \left(\rg^{i}_{\text{C}, t, j}\right)^{2} - \left(\rg^{i}_{\text{D}, t, j}\right)^{2} \right]\nonumber\\
        & = \E_{\rva\sim\boldsymbol{\pi}_{\vtheta}}\left[ \left(
        \frac{\partial\log\pi^{i}_{\vtheta}(\ra^{i}|s)}{\partial \theta^{i}}\hat{Q}(s, \rva)
        \right)^{2} - 
        \left(
        \frac{\partial\log\pi^{i}_{\vtheta}(\ra^{i}|s)}{\partial \theta^{i}}\hat{Q}^{i}(s, \ra^{i})
        \right)^{2} 
        \right]\nonumber\\
        &= \E_{\ra^i\sim\pi^i_\vtheta}\left[ \left(
        \frac{\partial\log\pi^{i}_{\vtheta}(\ra^{i}|s)}{\partial \theta^{i}}\right)^2 \E_{\rva^{-i}\sim\boldsymbol{\pi}^{-i}_\vtheta}\left[ \hat{Q}(s, \ra^i, \rva^{-i})^2 - \hat{Q}^i(s, \ra^i)^2\right]
        \right].\nonumber\\
        &\text{The inner expectation is the variance of $\hat{Q}(s, \ra^i, \rva^{-i})$, given $\ra^i$. We rewrite it as}\nonumber\\
        &= \E_{\ra^i\sim\pi^i_\vtheta}\left[ \left(
        \frac{\partial\log\pi^{i}_{\vtheta}(\ra^{i}|s)}{\partial \theta^{i}}\right)^2 \E_{\rva^{-i}\sim\boldsymbol{\pi}^{-i}_\vtheta}\left[ \left( \hat{Q}(s, \ra^i, \rva^{-i}) - \hat{Q}^i(s, \ra^i) \right)^2\right]
        \right]\nonumber\\
        & = \E_{\rva\sim\boldsymbol{\pi}_{\vtheta}}\left[ \left(
        \frac{\partial\log\pi^{i}_{\vtheta}(\ra^{i}|s)}{\partial \theta^{i}}\right)^{2}\left(\hat{Q}(s, \rva)
        - \hat{Q}^{i}(s, \ra^{i})
        \right)^{2} 
        \right].\nonumber
    \end{align}
    Now, recalling that the variance of the total gradient is the sum of variances of the gradient components, we have 
    \begin{align}
        &\mathbf{Var}_{\rva\sim\boldsymbol{\pi}_{\vtheta}}\left[ \rvg^{i}_{\text{C}, t} \right] - \mathbf{Var}_{\rva\sim\boldsymbol{\pi}_{\vtheta}}\left[ \rvg^{i}_{\text{D}, t} \right] 
        = \E_{\rva\sim\boldsymbol{\pi}_{\vtheta}}\left[ \left|\left|
        \nabla_{\theta^{i}}\log\pi^{i}_{\vtheta}(\ra^{i}|s)
        \right|\right|^{2}
        \left(\hat{Q}(s, \rva)
        - \hat{Q}^{i}(s, \ra^{i})
        \right)^{2} 
        \right] \nonumber\\
        &  \leq B_{i}^{2} \ \E_{\rva\sim\boldsymbol{\pi}_{\vtheta}}\left[ \left(\hat{Q}(s, \rva)
        - \hat{Q}^{i}(s, \ra^{i})
        \right)^{2} 
        \right]
        = B_{i}^{2} \ \E_{\ra^{i}\sim\pi^{i}_{\vtheta}}\left[
        \E_{\rva^{-i}\sim\boldsymbol{\pi}^{-i}_{\vtheta}}
        \left[
        \left(\hat{Q}(s, \rva)
        - \hat{Q}^{i}(s, \ra^{i})
        \right)^{2} 
        \right]
        \right] \nonumber\\
        &= B_{i}^{2} \ \E_{\ra^{i}\sim\pi^{i}_{\vtheta}}\left[
        \E_{\rva^{-i}\sim\boldsymbol{\pi}^{-i}_{\vtheta}}
        \left[
        \left(\hat{Q}(s, \ra^{i}, \rva^{-i})
        - \hat{Q}^{i}(s, \ra^{i})
        \right)^{2} 
        \right]
        \right]\nonumber\\
        &= B_{i}^{2} \ \E_{\ra^{i}\sim\pi^{i}_{\vtheta}}\left[
        \E_{\rva^{-i}\sim\boldsymbol{\pi}^{-i}_{\vtheta}}
        \left[
        \hat{A}^{-i}(s, \ra^{i}, \rva^{-i})^{2}
        \right]
        \right] =
        B_{i}^{2} \ \E_{\ra^{i}\sim\pi^{i}_{\vtheta}}
        \left[
        \mathbf{Var}_{\rva^{-i}\sim\boldsymbol{\pi}^{-i}_{\vtheta}}
        \left[
        \hat{A}^{-i}(s, \ra^{i}, \rva^{-i})
        \right]
        \right]\nonumber
    \end{align}
    which by the strong version of Lemma \ref{lemma:advantage-var-upper-bound}, given in Equation \ref{eq:advantage-var-upper-bound-stronger}, can be upper-bounded by
    \begin{align}
        &B_{i}^{2} \ \E_{\ra^{i}\sim\pi^{i}_{\vtheta}}
        \left[
        \sum_{j\neq i}\mathbf{Var}_{\rva^{-i}\sim\boldsymbol{\pi}^{-i}_{\vtheta}}\left[ 
        \hat{A}^{j}(s, \rva^{-j}, \ra^{j}) \right]
        \right]
        =
        B_{i}^{2}
        \sum_{j\neq i}
        \ \E_{\ra^{i}\sim\pi^{i}_{\vtheta}}
        \left[
        \mathbf{Var}_{\rva^{-i}\sim\boldsymbol{\pi}^{-i}_{\vtheta}}\left[ 
        \hat{A}^{j}(s, \rva^{-j}, \ra^{j}) \right]
        \right]
        \nonumber
    \end{align}
    Notice that, for any $j\neq i$, we have
    \begin{align}
        &\E_{\ra^{i}\sim\pi^{i}_{\vtheta}}
        \left[
        \mathbf{Var}_{\rva^{-i}\sim\boldsymbol{\pi}^{-i}_{\vtheta}}\left[ 
        \hat{A}^{j}(s, \rva^{-j}, \ra^{j}) \right]
        \right]\nonumber\\
        &= \E_{\ra^{i}\sim\pi^{i}_{\vtheta}}
        \left[
        \E_{\rva^{-i}\sim\boldsymbol{\pi}^{-i}_{\vtheta}}\left[ 
        \hat{A}^{j}(s, \rva^{-j}, \ra^{j})^{2} \right]
        \right]\nonumber\\
        &= \E_{\rva\sim\boldsymbol{\pi}_{\vtheta}}
        \left[ \hat{A}^{j}(s, \rva^{-j}, \ra^{j})^{2} \right] \ \leq \ \epsilon_{j}^{2}
        \nonumber
    \end{align}
    This gives
    \begin{align}
        &\mathbf{Var}_{\rva\sim\boldsymbol{\pi}_{\vtheta}}\left[ \rvg^{i}_{\text{C}, t} \right] 
        - 
        \mathbf{Var}_{\rva\sim\boldsymbol{\pi}_{\vtheta}}\left[ \rvg^{i}_{\text{D}, t} \right] \ \leq \ B_{i}^{2}\sum_{j\neq i}\epsilon_{j}^{2}\nonumber
    \end{align}
    and combining it with Equations \ref{eq:var-to-exp} and \ref{eq:exp-to-state} for entire gradient vectors, we get
    \begin{align}  
        \label{eq:per-step-bound}
         &\mathbf{Var}_{\rs\sim d^{t}_{\vtheta}, \rva\sim\boldsymbol{\pi}_{\vtheta}}\left[ \rvg^{i}_{\text{C}, t} \right] 
        - 
        \mathbf{Var}_{\rs\sim d^{t}_{\vtheta}, \rva\sim\boldsymbol{\pi}_{\vtheta}}\left[ \rvg^{i}_{\text{D}, t} \right] \ \leq \ B_{i}^{2}\sum_{j\neq i}\epsilon_{j}^{2}
    \end{align}
    Noting that 
    \begin{align}
    &\mathbf{Var}_{\rs_{0:\infty} \sim d_{\vtheta}^{0:\infty}, \rva_{0:\infty}\sim\boldsymbol{\pi}_{\vtheta} }\left[ \rvg_{\cdot}^{i} \right]
    =
    \mathbf{Var}_{\rs_{0:\infty} \sim d_{\vtheta}^{0:\infty}, \rva_{0:\infty}\sim\boldsymbol{\pi}_{\vtheta} }
    \left[ \sum_{t=0}^{\infty}\gamma^{t}\rvg_{\cdot, t}^{i} \right]\nonumber\\
    &= \sum_{t=0}^{\infty}\mathbf{Var}_{\rs_{t} \sim d_{\vtheta}^{t}, \rva\sim\boldsymbol{\pi}_{\vtheta}}
    \left[ \gamma^{t}\rvg_{\cdot, t}^{i} \right] 
    =
    \sum_{t=0}^{\infty}\gamma^{2t}\mathbf{Var}_{\rs_{t} \sim d_{\vtheta}^{t}, \rva\sim\boldsymbol{\pi}_{\vtheta}}\left[\rvg_{\cdot, t}^{i} \right]\nonumber
    \end{align}
    Combining this series expansion with the estimate from Equation \ref{eq:per-step-bound}, we finally obtain
    \begin{align}
    &\mathbf{Var}_{\rs_{0:\infty} \sim d_{\vtheta}^{0:\infty}, \rva_{0:\infty}\sim\boldsymbol{\pi}_{\vtheta} }\left[ \rvg^{i}_{\text{C}} \right] 
    - 
    \mathbf{Var}_{\rs_{0:\infty} \sim d_{\vtheta}^{0:\infty}, \rva_{0:\infty}\sim\boldsymbol{\pi}_{\vtheta} }\left[ \rvg^{i}_{\text{D}} \right]\nonumber\\ 
    &\ \ \ \ \ \ \ \ \ \ \ \ \ \ \ \leq \ \sum_{t=0}^{\infty}\gamma^{2t}\left(B_{i}^{2}\sum_{j\neq i}\epsilon_{j}^{2}\right)
    \leq \ \frac{B_{i}^{2}}{1-\gamma^{2}}\sum_{j\neq i}\epsilon_{j}^{2}\nonumber
    \end{align}
\end{proof}

\comadttheorem*

\begin{proof}
    \label{proof:coma-de}
    Just like in the proof of Theorem \ref{theorem:excess-variance}, we start with the difference
    \begin{align}
        \mathbf{Var}_{\rs\sim d^{t}_{\vtheta}, \ra\sim\boldsymbol{\pi}_{\vtheta}}\left[ \rg^{i}_{\text{COMA}, t, j}\right]
        -
        \mathbf{Var}_{\rs\sim d^{t}_{\vtheta}, \ra\sim\boldsymbol{\pi}_{\vtheta}}\left[ \rg^{i}_{\text{D}, t, j}\right]\nonumber
    \end{align}
    which we transform to an analogue of Equation \ref{eq:exp-to-state}:
    \begin{align}
        \E_{\rs\sim d^{t}_{\vtheta}}\left[ 
        \mathbf{Var}_{\rva\sim\boldsymbol{\pi}_{\vtheta}}\left[ \rg^{i}_{\text{COMA}, t, j}\right] - \mathbf{Var}_{\rva\sim\boldsymbol{\pi}_{\vtheta}}\left[\rg^{i}_{\text{D}, t, j}\right]
        \right]\nonumber
    \end{align}
    which is trivially upper-bounded by
    \begin{align}
        \E_{\rs\sim d^{t}_{\vtheta}}\left[ 
        \mathbf{Var}_{\rva\sim\boldsymbol{\pi}_{\vtheta}}\left[ \rg^{i}_{\text{COMA}, t, j}\right] 
        \right]\nonumber
    \end{align}
    Now, let us fix a state $s$. We have
    \begin{align}
        &\mathbf{Var}_{\rva\sim\boldsymbol{\pi}_{\vtheta}}
        \left[ \rg^{i}_{\text{COMA}, t, j}\right]
        =
        \mathbf{Var}_{\rva^{-i}\sim\boldsymbol{\pi}^{-i}_{\vtheta}, \ra^{i}\sim\pi^{i}_{\vtheta}}
        \left[ \frac{\partial \log\pi^{i}_{\vtheta}(\ra^{i}|s) }{\partial \theta^{i}_{j}} A^{i}(s, \rva^{-i}, \ra^{i}) \right]
        \nonumber\\
        &\leq \ \E_{\rva^{-i}\sim\boldsymbol{\pi}^{-i}_{\vtheta}, \ra^{i}\sim\pi^{i}_{\vtheta}}
        \left[ 
        \left(\frac{\partial \log\pi^{i}_{\vtheta}(\ra^{i}|s) }{\partial \theta^{i}_{j}} \right)^{2}A^{i}(s, \rva^{-i}, \ra^{i})^{2}\right]\nonumber\\
        &\leq \epsilon^{2}_{i} \ 
        \E_{\rva^{-i}\sim\boldsymbol{\pi}^{-i}_{\vtheta}, \ra^{i}\sim\pi^{i}_{\vtheta}}
        \left[ 
        \left(\frac{\partial \log\pi^{i}_{\vtheta}(\ra^{i}|s) }{\partial \theta^{i}_{j}} \right)^{2}\right]
    \end{align}
which summing over all components of $\theta^{i}$ gives
\begin{align}
    &\mathbf{Var}_{\rva\sim\boldsymbol{\pi}_{\vtheta}}
    \left[ \rvg^{i}_{\text{COMA}, t}\right] \ \leq \ \left( \epsilon_{i}B_{i}\right)^{2}\nonumber
\end{align}
Now, applying the reasoning from Equation \ref{eq:per-step-bound} until the end of the proof of Theorem \ref{theorem:excess-variance}, we arrive at the result
\begin{align}
         \mathbf{Var}_{\rs_{0:\infty} \sim d_{\vtheta}^{0:\infty}, \rva_{0:\infty}\sim\boldsymbol{\pi}_{\vtheta} }\left[ \rvg^{i}_{\text{COMA}} \right] 
        - 
        \mathbf{Var}_{\rs_{0:\infty} \sim d_{\vtheta}^{0:\infty}, \rva_{0:\infty}\sim\boldsymbol{\pi}_{\vtheta} }\left[ \rvg^{i}_{\text{D}} \right] 
        \ \leq \ \frac{\left(\epsilon_{i}B_{i}\right)^{2}}{1-\gamma^{2}} \nonumber
     \end{align}
\end{proof}

\section{Proofs of the Results about Optimal Baselines}
\label{appendix:proofs-ob}

In this section of the Appendix we prove the results about optimal baselines, which are those that minimise the CTDE MAPG estimator's variance. We rely on the following variance decomposition
\begin{smalleralign}[\scriptsize]
    \label{eq:grad-var-decomp}
    &\mathbf{Var}_{\rs_{t}\sim d^{t}_{\vtheta}, \rva_{t}\sim\boldsymbol{\pi}_{\vtheta}}\left[ \rvg_{\text{C}, t}^{i}(b) \right] = \mathbf{Var}_{\rs_{t}\sim d^{t}_{\vtheta}}\left[ \E_{\rva_{t}\sim\boldsymbol{\pi}_{\vtheta}}\left[ \rvg_{\text{C}, t}^{i}(b) \right] \right] + \E_{\rs_{t}\sim d^{t}_{\vtheta}}\left[ \mathbf{Var}_{\rva_{t}\sim\boldsymbol{\pi}_{\vtheta}}\left[ \rvg_{\text{C}, t}^{i}(b) \right] \right] \\
    &=  \mathbf{Var}_{\rs_{t}\sim d^{t}_{\vtheta}}\left[ \E_{\rva_{t}\sim\boldsymbol{\pi}_{\vtheta}}\left[ \rvg_{\text{C}, t}^{i}(b) \right] \right] + \E_{\rs_{t}\sim d^{t}_{\vtheta}}\left[ \mathbf{Var}_{\rva_{t}^{-i}\sim\boldsymbol{\pi}^{-i}_{\vtheta}}\left[ \E_{\ra_{t}^{i}\sim\pi^{i}_{\vtheta}}\left[ \rvg^{i}_{\text{C}, t}(b) \right] \right]  +
    \E_{\rva_{t}^{-i}\sim\boldsymbol{\pi}^{-i}_{\vtheta}}\left[ \mathbf{Var}_{\ra_{t}^{i}\sim\pi^{i}_{\vtheta}}\left[ \rvg^{i}_{\text{C}, t}(b) \right] \right]  \right] 
    \nonumber\\
    &= \underbrace{\mathbf{Var}_{\rs_{t}\sim d^{t}_{\vtheta}}\left[ \E_{\rva_{t}\sim\boldsymbol{\pi}_{\vtheta}}\left[ \rvg_{\text{C}, t}^{i}(b) \right] \right]}_{\text{\footnotesize  Variance from state}} + 
    \underbrace{\E_{\rs_{t}\sim d^{t}_{\vtheta}}\left[ \mathbf{Var}_{\rva_{t}^{-i}\sim\boldsymbol{\pi}^{-i}_{\vtheta}}\left[ \E_{\ra_{t}^{i}\sim\pi^{i}_{\vtheta}}\left[ \rvg^{i}_{\text{C}, t}(b) \right] \right] \right]}_\text{\footnotesize Variance from other agents' actions} +  \underbrace{\E_{\rs_{t} \sim d^{t}_{\vtheta}, \rva_{t}^{-i}\sim\boldsymbol{\pi}^{-i}_{\vtheta}}\Big[ \mathbf{Var}_{\ra_{t}^{i}\sim\pi^{i}_{\vtheta}}\big[ \rvg^{i}_{\text{C}, t}(b) \big] \Big]}_\text{\footnotesize Variance from agent $i$'s action}. \nonumber
\end{smalleralign}
This decomposition reveals that baselines impact the variance via the local variance $\mathbf{Var}_{\ra^{i}_{t}\sim\pi^{i}_{\vtheta}}\left[ \rvg^{i}_{\text{C}, t}(b) \right]$. We rely on this fact in the proofs below.

\subsection{Proof of Theorem \ref{theorem:ob-marl}}

\obtheorem*

\begin{proof}
    \label{proof:ob-marl}
    From the decomposition of the estimator's variance, we know that minimisation of the variance is equivalent to minimisation of the local variance 
    \begin{align}
    \mathbf{Var}_{\ra^{i}\sim\pi^{i}_{\vtheta}}
    \left[ 
    \left(\hat{Q}(s, \rva^{-i}, \ra^{i}) - b\right)
    \nabla_{\theta^{i}}\log\pi^{i}_{\vtheta}(\ra^{i}|s)
    \right]\nonumber
    \end{align}
    For a baseline $b$, we have
    \begin{align}
        &\mathbf{Var}_{\ra^{i}\sim\pi^{i}_{\vtheta}}
        \left[ 
        \left(
        \hat{Q}(s, \va^{-i}, \ra^{i}) - b)\right)
        \left(\frac{\partial\log\pi^{i}_{\vtheta}(\ra^{i}|s)}{\partial \theta^{i}_{j}}
        \right)
        \right] \nonumber\\
        &= \E_{\ra^{i}\sim\pi^{i}_{\vtheta}}
        \left[ 
        \left(
        \hat{Q}(s, \va^{-i}, \ra^{i}) - b)\right)^{2}
        \left(\frac{\partial\log\pi^{i}_{\vtheta}(\ra^{i}|s)}{\partial \theta^{i}_{j}}
        \right)^{2}
        \right]\nonumber\\
        &\ \ \ \ \ \ \ \ \ - 
        \E_{\ra^{i}\sim\pi^{i}_{\vtheta}}
        \left[ 
        \left(
        \hat{Q}(s, \va^{-i}, \ra^{i}) - b)\right)
        \left(\frac{\partial\log\pi^{i}_{\vtheta}(\ra^{i}|s)}{\partial \theta^{i}_{j}}
        \right)
        \right]^{2}\nonumber\\
        \label{term:quadratic-to-minimise}
        &= \E_{\ra^{i}\sim\pi^{i}_{\vtheta}}
        \left[ 
        \left(
        \hat{Q}(s, \va^{-i}, \ra^{i}) - b)\right)^{2}
        \left(\frac{\partial\log\pi^{i}_{\vtheta}(\ra^{i}|s)}{\partial \theta^{i}_{j}}
        \right)^{2}
        \right]\\
        &\ \ \ \ \ \ \ \ \ - 
        \E_{\ra^{i}\sim\pi^{i}_{\vtheta}}
        \left[ 
        \hat{Q}(s, \va^{-i}, \ra^{i}) 
        \left(\frac{\partial\log\pi^{i}_{\vtheta}(\ra^{i}|s)}{\partial \theta^{i}_{j}}
        \right)
        \right]^{2}\nonumber
    \end{align}
    as $b$ is a baseline. So in order to minimise variance, we shall minimise the term \ref{term:quadratic-to-minimise}.
    \begin{align}
        &\E_{\ra^{i}\sim\pi^{i}_{\vtheta}}
        \left[ 
        \left(
        \hat{Q}(s, \va^{-i}, \ra^{i}) - b)\right)^{2}
        \left(\frac{\partial\log\pi^{i}_{\vtheta}(\ra^{i}|s)}{\partial \theta^{i}_{j}}
        \right)^{2}
        \right]\nonumber\\
        &= \E_{\ra^{i}\sim\pi^{i}_{\vtheta}}
        \left[ 
        \left(b^{2} - 
        2b\ \hat{Q}(s, \va^{-i}, \ra^{i}) + \hat{Q}(s, \va^{-i}, \ra^{i})^{2})\right)
        \left(\frac{\partial\log\pi^{i}_{\vtheta}(\ra^{i}|s)}{\partial \theta^{i}_{j}}
        \right)^{2}
        \right]\nonumber\\
        &= b^{2} \ \E_{\ra^{i}\sim\pi^{i}_{\vtheta}}
        \left[ 
         \left(\frac{\partial\log\pi^{i}_{\vtheta}(\ra^{i}|s)}{\partial \theta^{i}_{j}}\right)^{2}     
        \right] - 2b \ \E_{\ra^{i}\sim\pi^{i}_{\vtheta}}\left[ 
         \hat{Q}(s, \va^{-i}, \ra^{i})
         \left(\frac{\partial\log\pi^{i}_{\vtheta}(\ra^{i}|s)}{\partial \theta^{i}_{j}}\right)^{2}     
        \right] \nonumber\\
        & \ \ \ \ \ \ \ \ \ \ \ \ \ \ \ \ \ \ + \E_{\ra^{i}\sim\pi^{i}_{\vtheta}}\left[ 
         \hat{Q}(s, \va^{-i}, \ra^{i})^{2}
         \left(\frac{\partial\log\pi^{i}_{\vtheta}(\ra^{i}|s)}{\partial \theta^{i}_{j}}\right)^{2}     
        \right] \nonumber
    \end{align}
    which is a quadratic in $b$. The last term of the quadratic does not depend on $b$, and so it can be treated as a constant. Recalling that the variance of the whole gradient vector $\rvg^{i}(b)$ is the sum of variances of its components $\rg^{i}_{j}(b)$, we obtain it by summing over $j$
    
    \begin{align}
        \label{eq:quadratic-of-ob}
        &\mathbf{Var}_{\ra^{i}\sim\pi^{i}_{\vtheta}}
        \left[ 
        \left(\hat{Q}(s, \va^{-i}, \ra^{i}) - b\right)
        \nabla_{\theta^{i}}\log\pi^{i}_{\vtheta}(\ra^{i}|s)
        \right]\nonumber\\
        &= \sum_{j}\Bigg(b^{2} \ \E_{\ra^{i}\sim\pi^{i}_{\vtheta}}
        \left[ 
         \left(\frac{\partial\log\pi^{i}_{\vtheta}(\ra^{i}|s)}{\partial \theta^{i}_{j}}\right)^{2}     
        \right]\nonumber\\
        &\ \ \ \ \ \ \ \ \ \ \ \ \ \ \ \ \ \ \ \ \ \ \ \ \ \ \ \ \ \ - 2b \  \E_{\ra^{i}\sim\pi^{i}_{\vtheta}}\left[ 
         \hat{Q}(s, \va^{-i}, \ra^{i})\left(
         \frac{\partial\log\pi^{i}_{\vtheta}(\ra^{i}|s)}{\partial \theta^{i}_{j}}
         \right)^{2}     
        \right] + const \Bigg)\nonumber\\
        &= b^{2} \ \E_{\ra^{i}\sim\pi^{i}_{\vtheta}}\left[ 
         \left|\left| \nabla_{\theta^{i}}\log\pi^{i}_{\vtheta}(\ra^{i}|s)\right|\right|^{2}   
        \right]\nonumber\nonumber \\
        &\ \ \ \ \ \ \ \ \ \ \ - 2b \  \E_{\ra^{i}\sim\pi^{i}_{\vtheta}}\left[ 
         \hat{Q}(s, \va^{-i}, \ra^{i})\left|\left| \nabla_{\theta^{i}}\log\pi^{i}_{\vtheta}(\ra^{i}|s)\right|\right|^{2}
        \right] + const
    \end{align}
    As the leading coefficient is positive, the quadratic achieves the minimum at
    \begin{align}
        b^{\text{optimal}} = \frac{ \E_{\ra^{i}\sim\pi^{i}_{\vtheta}}\left[ 
        \hat{Q}(\rs, \va^{-i}, \ra^{i}) \left|\left| \nabla_{\theta^{i}}\log \pi^{i}_{\vtheta}(\ra^{i}|s) \right|\right|^{2} \right] }{  
        \E_{\ra^{i}\sim\pi^{i}_{\vtheta}}\left[ 
        \left|\left| 
        \nabla_{\theta^{i}}\log \pi^{i}_{\vtheta}(\ra^{i}|\rs) 
        \right|\right|^{2} \right]}\nonumber
    \end{align}
\end{proof}

\subsection{Remarks about the surrogate optimal baseline}
In the paper, we discussed the impracticality of the above baseline. To handle this, we noticed that the policy $\pi^{i}_{\vtheta}(\ra^{i}|s)$, at state $s$, is determined by the output layer, $\psi^{i}_{\vtheta}(s)$, of an actor neural network. With this representation, in order to handle the impracticality of the above optimal baseline, we considered a minimisation objective, the \textsl{surrogate local variance}, given by
\begin{smalleralign}
    \mathbf{Var}_{\ra^{i}\sim\pi^{i}_{\vtheta}}\left[ 
    \nabla_{\psi^{i}_{\vtheta}}\log\pi^{i}_{\vtheta}\left(\ra^{i}|\psi^{i}_{\vtheta}(s)\right)
    \left(\hat{Q}(s, \va^{-i}, \ra^{i}) - b(s, \va^{-i})\right)
    \right]\nonumber
\end{smalleralign}
As a corollarly to the proof, the surrogate version of the optimal baseline (which we refer to as OB) was proposed, and it is given by
\begin{smalleralign}
    &b^{*}(\rs, \rva^{-i}) 
    = \frac{ \E_{\ra^{i}\sim\pi^{i}_{\vtheta}}\left[ \hat{Q}(\rs, \rva^{-i}, \ra^{i})
    \left|\left|
    \nabla_{\psi^{i}_{\vtheta}(\rs)}\log\pi^{i}\left(\ra^{i}|\psi^{i}_{\vtheta}(\rs)\right)
    \right|\right|^{2} \right]  }{ \E_{\ra^{i}\sim\pi^{i}_{\vtheta}}\left[ \left|\left|
    \nabla_{\psi^{i}_{\vtheta}(\rs)}\log\pi^{i}\left(\ra|\psi^{i}_{\vtheta}(\rs)\right)
    \right|\right|^{2} \right]  }.\nonumber
\end{smalleralign}

\begin{remark}
The 
$x^{i}_{\psi^{i}_{\vtheta}}$ measure, for which {\small $b^{*}(s, \va^{-i}) 
= \E_{\ra^{i}\sim\pi^{i}_{\vtheta}}\left[
\hat{Q}(s, \va^{-i}, \ra^{i}) \right]$}, is generally given by
\begin{align}
    \label{eq:x-measure-def}
    x^{i}_{\psi^{i}_{\vtheta}}\left(\ra^{i}|\rs\right) = \frac{ \pi^{i}_{\vtheta}\left( \ra^{i}|\rs \right) \left|\left| \nabla_{\theta^{i}}\log\pi^{i}_{\vtheta}\left( \ra^{i}|\rs \right) \right|\right|^{2} }{
    \E_{\ra^{i}\sim\pi^{i}_{\vtheta}}\left[ 
    \left|\left| \nabla_{\theta^{i}}\log\pi^{i}_{\vtheta}\left( \ra^{i}|\rs \right) \right|\right|^{2} 
    \right]
    }
\end{align}
\end{remark}

Let us introduce the definition of the $\softmax$ function, which is the subject of the next definition. For a vector $\vz\in\mathbb{R}^{d}$, we have $\softmax(\vz) = \left( \frac{e^{z_{1}}}{\eta}, \dots, \frac{e^{z_{d}}}{\eta}\right)$, where $\eta = \sum_{j=1}^{d}e^{z_{j}}$. We write
$\softmax\left( \psi^{i}_{\vtheta}(s)\right)\left( a^{i} \right) = \frac{\exp\left(\psi^{i}_{\vtheta}(s)(a^{i})\right)  }{ \sum_{\tilde{a}^{i}}\exp\left(\psi^{i}_{\vtheta}(s)(\tilde{a}^{i})\right) }$.

\begin{remark}
\label{remark:x-measure-formula}
When the action space is discrete, and the actor's policy is $\pi^{i}_{\vtheta}\left( \ra^{i}| \rs\right) = \softmax\left( \psi^{i}_{\vtheta}(s) \right)\left(\ra^{i}\right)$, then the $x^{i}_{\psi^{i}_{\vtheta}}$ measure is given by
\begin{align}
    x^{i}_{\psi^{i}_{\vtheta}}\left( \ra^{i}|\rs \right)
    =
    \frac{\pi^{i}_{\vtheta}\left(\ra^{i}|\rs\right)
    \left(1 + \left|\left|\pi_{\vtheta}^{i}(\rs)\right|\right|^{2} - 2\pi_{\vtheta}^{i}\left(\ra^{i}|\rs\right)\right)}{ 1 - ||\pi_{\vtheta}^{i}(\rs)||^{2}}\nonumber
\end{align}
\end{remark}
\begin{proof}
    \label{proof:x-measure}
    As we do not vary states $\rs$ and parameters $\vtheta$ in this proof, let us drop them from the notation for $\pi^{i}_{\vtheta}$, and $\psi^{i}_{\vtheta}(\rs)$, hence writing $\pi^{i}\left(\ra^{i}\right) = \softmax\left( \psi^{i} \right)\left(\ra^{i}\right)$. Let us compute the partial derivatives:
    \begin{align}
        &\frac{ \partial\log \pi^{i}\left(a^{i}\right) }{ \partial \psi^{i}\left( \tilde{a}^{i} \right)}
        = 
        \frac{\partial \log \softmax\left( \psi^{i}\right)\left(a^{i}\right) }
        { \partial \psi^{i}\left( \tilde{a}^{i} \right) } 
        =
        \frac{\partial}{\partial \psi^{i}\left( \tilde{a}^{i} \right)}\left[ \log
        \frac{ \exp\left( \psi^{i}\left(a^{i}\right) \right) }
        {
        \sum_{\hat{a}^{i}}\exp\left( \psi^{i}\left(\hat{a}^{i}\right)\right)
        } \right]\nonumber\\
        &=
        \frac{\partial}{\partial \psi^{i}\left( \tilde{a}^{i} \right)}\left[ \psi^{i}\left(a^{i}\right) 
        -
        \log\sum_{\hat{a}^{i}}\exp\left( \psi^{i}\left(\hat{a}^{i}\right)\right) \right]\nonumber\\
        &= \bold{I}\left(a^{i} = \tilde{a}^{i}\right)
        -
        \frac{\exp\left( \psi^{i}\left(
        \tilde{a}^{i}\right)\right) }{\sum_{\hat{a}^{i}}\exp\left( \psi^{i}\left(\hat{a}^{i}\right)\right)} = 
        \bold{I}\left(a^{i} = \tilde{a}^{i}\right)
        - \pi^{i}\left(\tilde{a}^{i}\right)\nonumber
    \end{align}
    where $\mathbf{I}$ is the indicator function, taking value $1$ if the stetement input to it is true, and $0$ otherwise. Taking $\ve_{k}$ to be the standard normal vector with $1$ in $k^{\text{th}}$ entry, we have the gradient
    \begin{align}
        \label{eq:grad-of-log}
        \nabla_{\psi^{i}}\log\pi^{i}\left( a^{i} \right) = \ve_{a^{i}} - \pi^{i} 
    \end{align}
    which has the squared norm
    \begin{align}
        & \left|\left|\nabla_{\psi^{i}}\log\pi^{i}\left( a^{i} \right)\right|\right|^{2} = \left|\left|\ve_{a^{i}} - \pi^{i}\right|\right|^{2} = \left(1 - \pi^{i}\left(a^{i}\right) \right)^{2} + \sum_{\tilde{a}^{i}\neq a^{i}}\left(-\pi^{i}\left( \tilde{a}^{i} \right)\right)^{2} \nonumber\\
        &= 1 + \sum_{\tilde{a}^{i}}\left(-\pi^{i}\left( \tilde{a}^{i} \right)\right)^{2} - 2\pi^{i}\left( a^{i} \right) = 1 + \left|\left| \pi^{i} \right|\right|^{2} - 2\pi^{i}\left( a^{i} \right)\nonumber.
    \end{align}
    The expected value of this norm is 
    \begin{align}
        &\E_{\ra^{i}\sim\pi^{i}}\left[  
        1 + \left|\left| \pi^{i} \right|\right|^{2} - 2\pi^{i}\left( a^{i} \right)
        \right] = 
        1 + \left|\left| \pi^{i} \right|\right|^{2} - 
        \E_{\ra^{i}\sim\pi^{i}}\left[  
        2\pi^{i}\left( a^{i} \right)
        \right]\nonumber\\
        &= 1 + \left|\left| \pi^{i} \right|\right|^{2}
        - 2\sum_{\tilde{a}^{i}}\left( \pi^{i}\left(a^{i}\right) \right)^{2} = 1  - \left|\left| \pi^{i} \right|\right|^{2}\nonumber
    \end{align}
    which combined with Equation \ref{eq:x-measure-def} finishes the proof.
\end{proof}

\subsection{Proof of Theorem \ref{thm:excess}}

\esvtheorem*

\begin{proof}
    \label{proof:comparison}
    The first part of the theorem (the formula for excess variance) follows from Equation \ref{eq:quadratic-of-ob}.
    For the rest of the statements, it suffices to show the first of each inequalities, as the later ones follow directly from the fact that $|Q_{\vtheta}(s, \va)| \leq \frac{\beta}{1-\gamma}$, $\mathbf{Var}_{\ra^{i}\sim\pi^{i}_{\vtheta}}\left[ A_{\vtheta}^{i}(s, \va^{-i}, \ra^{i}) \right] = \E_{\ra^{i}\sim\pi^{i}_{\vtheta}}\left[ A_{\vtheta}^{i}(s, \va^{-i}, \ra^{i})^{2} \right] $, and the definition of $\epsilon_{i}$.
    Let us first derive the bounds for $\Delta \mathbf{Var}_{\text{MAPG}}$. Let us, for short-hand, define
    \begin{align}
        &c_{\vtheta}^{i} := \E_{\ra^{i}\sim\pi^{i}_{\vtheta}}\left[ 
        \left|\left|
        \nabla_{\psi^{i}_{\vtheta}}\log \pi^{i}_{\vtheta}\left(\ra^{i}|\psi^{i}_{\vtheta}\right)
        \right|\right|^{2} \right]
        \nonumber
    \end{align}
    We have
    \begin{align}
        &\Delta\mathbf{Var}_{\text{MAPG}} = \Delta \mathbf{Var}(0) 
        = c_{\vtheta}^{i} \ b^{*}(s, \va^{-i})^{2} 
        = c_{\vtheta}^{i} \ \E_{\ra^{i}\sim x_{\psi^{i}_{\vtheta}}^{i}}
        \left[ \hat{Q}(s, \va^{-i}, \ra^{i}) \right]^{2} \nonumber\\
        &\leq \ c_{\vtheta}^{i} \ \E_{\ra^{i}\sim x_{\psi^{i}_{\vtheta}}^{i}}
        \left[ \hat{Q}(s, \va^{-i}, \ra^{i})^{2} \right]
        = c_{\vtheta}^{i}\E_{\ra^{i}\sim \pi_{\vtheta}^{i}}
        \left[
        \frac{\hat{Q}(s, \va^{-i}, \ra^{i})^{2}\left|\left|
        \nabla_{\psi^{i}_{\vtheta}}\log \pi^{i}_{\vtheta}\left(\ra^{i}|\psi^{i}_{\vtheta}\right)\right|\right|^{2} }{c_{\theta}^{i}} \right]\nonumber \\
        &= \ \E_{\ra^{i}\sim \pi_{\vtheta}^{i}}\left[
        \hat{Q}(s, \va^{-i}, \ra^{i})^{2}
        \left|\left|
        \nabla_{\psi^{i}_{\vtheta}}\log \pi^{i}_{\vtheta}\left(\ra^{i}|\psi^{i}_{\vtheta}(s)\right)\right|\right|^{2} 
        \right]\nonumber\\
        & \ \ \ \ \ \ \ \ \ \ \ \ \leq \E_{\ra^{i}\sim \pi_{\vtheta}^{i}}\left[
        \hat{Q}(s, \va^{-i}, \ra^{i})^{2} D_{i}^{2}\right] \nonumber\\ 
        &= \ D_{i}^{2}\left(
        \E_{\ra^{i}\sim \pi_{\vtheta}^{i}}\left[
        \hat{Q}(s, \va^{-i}, \ra^{i})^{2} \right] - \E_{\ra^{i}\sim \pi_{\vtheta}^{i}}\left[
        \hat{Q}(s, \va^{-i}, \ra^{i}) \right]^{2} + \E_{\ra^{i}\sim \pi_{\vtheta}^{i}}\left[
        \hat{Q}(s, \va^{-i}, \ra^{i}) \right]^{2} \right)\nonumber\\
        &= \ D_{i}^{2}\left( \mathbf{Var}_{\ra^{i}\sim\pi^{i}_{\vtheta}}
        \left[ \hat{Q}(s, \va^{-i}, \ra^{i}) \right] + \hat{Q}^{-i}(s, \va^{-i})^{2} \right)\nonumber\\
        & = D_{i}^{2}\left( \mathbf{Var}_{\ra^{i}\sim\pi^{i}_{\vtheta}}
        \left[ \hat{A}^{i}(s, \va^{-i}, \ra^{i}) \right] + \hat{Q}^{-i}(s, \va^{-i})^{2} 
        \right)\nonumber
    \end{align}
    which finishes the proof for MAPG. For COMA, we have
    \begin{align}
        &\Delta \mathbf{Var}_{\text{COMA}} \ = \ \Delta\mathbf{Var}\left( \hat{Q}^{-i}(s, \va^{-i})\right) =  c_{\vtheta}^{i}
        \left( b^{*}(s, \va^{-i}) - \hat{Q}^{-i}(s, \va^{-i})\right)^{2} \nonumber\\
        &= \ c_{\vtheta}^{i}\left( \E_{\ra^{i}\sim x_{\psi^{i}_{\vtheta}}^{i}}\left[ \hat{Q}(s, \va^{-i}, \ra^{i}) \right] - \hat{Q}^{-i}(s, \va^{-i}) \right)^{2} 
        \nonumber\\
        &= c_{\vtheta}^{i} \E_{\ra^{i}\sim x^{i}_{\psi^{i}_{\vtheta}} }
        \left[ \hat{Q}(s, \va^{-i}, \ra^{i})  - \hat{Q}^{-i}(s, \va^{-i}) \right]^{2} 
        \nonumber\\
        &= \ c_{\vtheta}^{i} \E_{\ra^{i}\sim x^{i}_{{\psi}^{i}_{\vtheta}} }\left[ \hat{A}^{i}(s, \va^{-i}, \ra^{i})  \right]^{2}
        \nonumber\\
        &\leq \ c_{\vtheta}^{i} \E_{\ra^{i}\sim x^{i}_{ {\psi}^{i}_{\vtheta} } }\left[ \hat{A}^{i}(s, \va^{-i}, \ra^{i})^{2}  \right] 
        \nonumber\\
        &= \ c_{\vtheta}^{i}\E_{\ra^{i}\sim \pi_{\vtheta}^{i}}\left[ \hat{A}^{i}(s, \va^{-i}, \ra^{i})^{2}\frac{\left|\left|\nabla_{\psi^{i}_{\vtheta}}\log \pi^{i}_{\vtheta}\left(\ra^{i}|\psi^{i}_{\vtheta}(s)\right)\right|\right|^{2} }{c_{\vtheta}^{i}}  \right] 
        \nonumber\\
        &= \ \E_{\ra^{i}\sim \pi_{\vtheta}^{i}}\left[ \hat{A}^{i}(s, \va^{-i}, \ra^{i})^{2}\left|\left|\nabla_{\psi^{i}_{\vtheta}}\log \pi^{i}_{\vtheta}\left(\ra^{i}|\psi^{i}_{\vtheta}(s)\right)\right|\right|^{2}  
        \right]
        \nonumber\\
        &\leq \ 
        \E_{\ra^{i}\sim \pi_{\vtheta}^{i}}\left[ D_{i}^{2}\hat{A}^{i}(s, \va^{-i}, \ra^{i})^{2} 
        \right] \ \leq \ \left(\epsilon_{i}D_{i}\right)^{2}
        \nonumber
    \end{align}
    which finishes the proof.
\end{proof}

\newpage

\section{Pytorch Implementations of the Optimal Baseline}
\label{appendix:implementation}

First, we import necessary packages, which are \textbf{PyTorch} \cite{paszke2019pytorch}, and its \textbf{nn.functional} sub-package. These are standard Deep Learning packages used in RL \cite{SpinningUp2018}.
\label{code:imports}
\begin{lstlisting}[language=Python]
import torch, torch.nn.functional as F
\end{lstlisting}
\vspace{-5pt}
We then implement a simple method that normalises a row vector, so that its (non-negative) entries sum up to $1$, making the vector a probability distribution.
\label{code:normalise}
\begin{lstlisting}[language=Python]
# x: batch of row vectors to normalise to probability mass
normalize = lambda x: F.normalize(x, p=1, dim=-1)
\end{lstlisting}
\vspace{-5pt}
The \textbf{discrete} OB is an exact dot product between the measure $x^{i}_{\psi^{i}_{\vtheta}}$, and available values of $\hat{Q}$.
\label{code:discrete-ob}
\begin{lstlisting}[language=Python]
# q: Q values of actions of agent i
# pi: policy of agent i
def optimal_baseline(q, pi):
    M = torch.norm(pi, dim=-1, keepdim = True) ** 2 + 1
    xweight = normalize( (M - 2 * pi) * pi )
    return (xweight * q).sum(-1)
\end{lstlisting}
\vspace{-5pt}
In the \textbf{continuous} case, the measure $x^{i}_{\psi^{i}_{\vtheta}}$ and $Q$-values can only be sampled at finitely many points.
\label{code:continuous-ob}
\begin{lstlisting}[language=Python]
# a: sampled actions of agent i
# q: Q values of the sampled actions
# mu, std: parameters of the Gaussian policy of agent i
def optimal_baseline(a, q, mu, std):
    mu_term = torch.norm((a - mu)/std**2, dim=-1)
    std_term = torch.norm(((a - mu)**2 - std**2)/std**3, dim=-1)
    xweight = normalize(mu_term**2 + std_term**2)
    return (xweight * q).sum(-1)
\end{lstlisting}
\vspace{-5pt}
We can incorporate it into our MAPG algorith by simply replacing the values of advantage with the values of X, in the buffer. Below, we present a discrete example
\label{code:apply-discrete-ob}
\begin{lstlisting}[language=Python]
# compute the policy and sample an action from it
a, pi = actor(obs)
q = critic(obs)

#compute OB
ob = optimal_baseline(q, pi)

# use OB to construct the loss
q = q.gather(-1, a)
pi = pi.gather(-1, a)
X = q - ob
loss = -(X * torch.log(pi)).mean()
\end{lstlisting}
\vspace{-5pt}
and a continuous one
\label{code:apply-continuous-ob}
\begin{lstlisting}[language=Python]
# normal sampling step, where log_pi is the log probability of a
a, log_pi = actor(obs, deterministic=False)
q = critic(obs, a)

# resample m (e.g., m=1000) actions for the observation
obs_m = obs.unsqueeze(0).repeat(m, 1)
a_m, mu_m, std_m = actor(obs, deterministic=False)

# approximate OB
q_m = critic(obs, a_m)
ob = optimal_baseline(a_m, q_m, mu_m, std_m)

# use OB to construct the loss
X = q - ob
loss = -(X * log_pi).mean()
\end{lstlisting}

\section{Computations for the Numerical Toy Example}

Here we prove that the quantities in table are filled properly.
% NUMERICAL TOY EXAMPLE
\begin{table}[h!]
    \begin{adjustbox}{width=\columnwidth, center}
    \begin{threeparttable}
        \vspace{-10pt}
             \begin{tabular}{c | c c c c c c| c c } \toprule
             \\[-1em]
             $\ra^{i}$ 
             & $\psi^{i}_{\vtheta}(\ra^{i})$ 
             & $\pi_{\vtheta}^{i}(\ra^{i})$ 
             & $x_{\psi^{i}_{\vtheta}}^{i}(\ra^{i})$ 
             & $\hat{Q}( \va^{-i}, \ra^{i})$ 
             & $\hat{A}^{i}(\va^{-i}, \ra^{i})$
             & $\hat{X}^{i}(\va^{-i}, \ra^{i})$
             &  Method & Variance \\ [0.6ex]
             \midrule
             1 & $\log 8$ & $0.8$ & 0.14 & $2$ & $-9.7$ & $-41.71$ & MAPG & \textbf{1321}\\ 
             2 & $0$ & $0.1$ &  0.43 & $1$ &$-10.7$ & $-42.71$ & COMA & \textbf{1015}\\
             3 & $0$ & $0.1$ &  0.43 & $100$ & $88.3$ & $56.29$ & OB & \textbf{\color{red}{673}}\\
             \bottomrule
             \end{tabular}
    \end{threeparttable}
    \end{adjustbox}
    \vspace{-15pt}
\end{table}

\begin{proof}
    \label{proof:toy-example}
    In this proof, for convenience, the multiplication and exponentiation of vectors is element-wise.
    Firstly, we trivially obtain the column $\pi^{i}_{\vtheta}(\ra^{i})$, by taking $\softmax$ over $\psi^{i}_{\vtheta}(\ra^{i})$. This allows us to compute the counterfactual baseline of COMA, which is
    \begin{smalleralign}
        \label{eq:compute-coma}
        &\hat{Q}^{-i}(\va^{-i})
        = \E_{\ra^{i}\sim\pi^{i}_{\vtheta}}\left[ \hat{Q}\left(
        \va^{-i}, \ra^{i} \right) \right]
        = \sum_{a^{i}=1}^{3}\pi^{i}_{\vtheta}(a^{i})\hat{Q}\left(\va^{-i}, a^{i}\right)\nonumber\\
        &= 0.8\times 2 + 0.1 \times 1 + 0.1 \times 100 = 1.6 + 0.1 + 10 = 11.7\nonumber
    \end{smalleralign}
    By subtracting this value from the column $\hat{Q}(\va^{-i}, \ra^{i})$, we obtain the column $\hat{A}^{i}(\va^{-i}, \ra^{i})$.
    
    Let us now compute the column of $x_{\psi^{i}_{\vtheta}}^{i}$. For this, we use Remark \ref{remark:x-measure-formula}. We have
    \begin{align}
        &\left|\left| \pi^{i}_{\vtheta} \right|\right|^{2} 
        = 0.8^{2} + 0.1^{2} + 0.1^{2} = 0.66\nonumber
    \end{align}
    and $1+ \left|\left| \pi^{i}_{\vtheta} \right|\right|^{2} - 2\pi^{i}_{\vtheta}(\ra^{i}) = 1.66 - 2\pi^{i}_{\vtheta}(\ra^{i})$, which is $0.06$ for $\ra^{i}=1$, and $1.46$ when $\ra^{i}=2, 3$.
    For $\ra^{i}=1$, we have that
    \begin{align}
        \pi^{i}_{\vtheta}(\ra^{i})\left( 1+ \left|\left| \pi^{i}_{\vtheta} \right|\right|^{2} - 2\pi^{i}_{\vtheta}(\ra^{i})\right) = 0.8\times0.06 = 0.048\nonumber
    \end{align}
    and for $\ra^{i}=2, 3$, we have
    \begin{align}
        \pi^{i}_{\vtheta}(\ra^{i})\left( 1+ \left|\left| \pi^{i}_{\vtheta} \right|\right|^{2} - 2\pi^{i}_{\vtheta}(\ra^{i})\right) = 0.1\times1.46 = 0.146\nonumber
    \end{align}
    We obtain the column $x^{i}_{\vtheta}(\ra^{i})$ by normalising the vector $(0.048, 0.146, 0.146)$.
    Now, we can compute OB, which is the dot product of the columns $x^{i}_{\psi^{i}_{\vtheta}}(\ra^{i})$ and $\hat{Q}(\va^{-i}, \ra^{i})$
    \begin{align}
        b^{*}(\va^{-i}) = 0.14\times2 + 0.43\times1 + 0.43\times100 = 0.28 + 0.43 + 43 = 43.71\nonumber
    \end{align}
    We obtain the column $\hat{X}^{i}(\va^{-i}, \ra^{i})$ after subtracting $b^{*}(\va^{-i})$ from the column $\hat{Q}(\va^{-i}, \ra^{i})$.
    
    Now, we can compute and compare the variances of vanilla MAPG, COMA, and OB. 
    The surrogate local variance of an MAPG estimator $\rvg^{i}(b)$ is
    \begin{smalleralign}
        &\mathbf{Var}_{\ra^{i}\sim\pi^{i}_{\vtheta}}\left[
        \rvg^{i}(b)
        \right]
        =
        \mathbf{Var}_{\ra^{i}\sim\pi^{i}_{\vtheta}}\left[
        \left( \hat{Q}^{i}\left(\va^{-i}, \ra^{i}\right)
        - b(\va^{-i})\right)
        \nabla_{\psi^{i}_{\vtheta}}\log\pi^{i}_{\theta}(a^{i})
        \right]\nonumber\\
        &= \text{sum}\left(
        \E_{\ra^{i}\sim\pi^{i}_{\vtheta}}\left[
        \left(
        \left[\hat{Q}^{i}\left(\va^{-i}, \ra^{i}\right)
        - b(\va^{-i})\right]
        \nabla_{\psi^{i}_{\vtheta}}\log\pi^{i}_{\theta}(a^{i})
        \right)^{2}
        \right] 
        -    \E_{\ra^{i}\sim\pi^{i}_{\vtheta}}\left[
        \left( \hat{Q}^{i}\left(\va^{-i}, \ra^{i}\right)
        - b(\va^{-i})\right)
        \nabla_{\psi^{i}_{\vtheta}}\log\pi^{i}_{\theta}(a^{i})
        \right]^{2} \right)\nonumber\\
        &= \text{sum}\left(
        \E_{\ra^{i}\sim\pi^{i}_{\vtheta}}\left[
        \left(
        \left[\hat{Q}^{i}\left(\va^{-i}, \ra^{i}\right)
        - b(\va^{-i})\right]
        \nabla_{\psi^{i}_{\vtheta}}\log\pi^{i}_{\theta}(a^{i})
        \right)^{2}
        \right]\right)
        - 
        \text{sum}\left(
        \E_{\ra^{i}\sim\pi^{i}_{\vtheta}}\left[
        \hat{Q}^{i}\left(\va^{-i}, \ra^{i}\right)
        \nabla_{\psi^{i}_{\vtheta}}\log\pi^{i}_{\theta}(a^{i})
        \right]^{2} \right)\nonumber
    \end{smalleralign}
    where ``sum" is taken element-wise. The last equality follows be linearity of element-wise summing, and the fact that $b$ is a baseline.
    We compute the variance of vanilla MAPG ($\rvg^{i}_{\text{MAPG}}$), COMA ($\rvg^{i}_{\text{COMA}}$), and OB ($\rvg^{i}_{X}$).
    Let us derive the first moment, which is the same for all methods
    \begin{align}
        &\E_{\ra^{i}\sim\pi^{i}_{\vtheta}}\left[ \hat{Q}\left(\va^{-i}, \ra^{i}\right)\nabla_{\psi^{i}_{\vtheta}}\log \pi^{i}_{\vtheta}(\ra^{i}) \right] = \sum_{a^{i}=1}^{3}\pi^{i}_{\vtheta}(a^{i})\hat{Q}\left(\va^{-i}, a^{i}\right)\nabla_{\psi^{i}_{\vtheta}}\log\pi^{i}_{\theta}(a^{i})\nonumber\\
        &\text{recalling Equation \ref{eq:grad-of-log}}\nonumber\\
        &= \ \sum_{a^{i}=1}^{3}\pi^{i}_{\vtheta}(a^{i})\hat{Q}\left(\va^{-i}, a^{i}\right)\left( \ve_{a^{i}} - \pi^{i}_{\vtheta} \right)\nonumber \\
        & =
        \ 0.8\times 2 \times \begin{bmatrix}
        0.2\\
        -0.1\\
        -0.1
        \end{bmatrix} + 0.1\times 1 \times \begin{bmatrix}
        -0.8\\
        0.9\\
        -0.1
        \end{bmatrix} +0.1\times 100 \times \begin{bmatrix}
        -0.8\\
        -0.1\\
        0.9
        \end{bmatrix}\nonumber\\
        &= \begin{bmatrix}
        0.32\\
        -0.16\\
        -0.16\\
        \end{bmatrix} +
        \begin{bmatrix}
        -0.08\\
        0.09\\
        -0.01\\
        \end{bmatrix} +
        \begin{bmatrix}
        -8\\
        -1\\
        9\\
        \end{bmatrix} =
        \begin{bmatrix}
        -7.76\\
        -1.07\\
        8.83\\
        \end{bmatrix}\nonumber
    \end{align}
    Now, let's compute the second moment for each of the methods, starting from vanilla MAPG
    \begin{align}
        &\E_{\ra^{i}\sim\pi^{i}_{\vtheta}}\left[ \hat{Q}\left(\va^{-i}, \ra^{i}\right)^{2}\left(\nabla_{\psi_{\vtheta}^{i}}\log \pi^{i}_{\vtheta}(\ra^{i})\right)^{2} \right] \nonumber\\
        &= \sum_{a^{i}=1}^{3}\pi^{i}_{\vtheta}(a^{i})\hat{Q}\left(\va^{-i}, a^{i}\right)^{2}
        \left(\nabla_{\psi^{i}_{\vtheta}}\log\pi^{i}_{\vtheta}(a^{i})\right)^{2}\nonumber\\
        &= \ \sum_{a^{i}=1}^{3}\pi^{i}_{\vtheta}(a^{i})\hat{Q}\left(\va^{-i}, a^{i}\right)^{2}\left(\ve_{a^{i}} - \pi^{i}_{\vtheta} \right)^{2}
        \nonumber\\
        &= \ 0.8\times 2^{2} \times \begin{bmatrix}
        0.2\\
        -0.1\\
        -0.1
        \end{bmatrix}^{2} + 0.1\times 1^{2} \times \begin{bmatrix}
        -0.8\\
        0.9\\
        -0.1
        \end{bmatrix}^{2} +0.1\times 100^{2} \times \begin{bmatrix}
        -0.8\\
        -0.1\\
        0.9
        \end{bmatrix}^{2}\nonumber\\
        &= 0.8\times 4 \times \begin{bmatrix}
        0.04\\
        0.01\\
        0.01
        \end{bmatrix} + 0.1 \times \begin{bmatrix}
        0.64\\
        0.81\\
        0.01
        \end{bmatrix} +0.1\times 10000 \times \begin{bmatrix}
        0.64\\
        0.01\\
        0.81
        \end{bmatrix}\nonumber\\
        &= \begin{bmatrix}
        0.128\\
        0.032\\
        0.032
        \end{bmatrix} +
        \begin{bmatrix}
        0.064\\
        0.081\\
        0.001
        \end{bmatrix} +
        \begin{bmatrix}
        640\\
        10\\
        810
        \end{bmatrix} \ = \ \begin{bmatrix}
        640.192\\
        10.113\\
        810.033
        \end{bmatrix}
        \nonumber
    \end{align}
    We have
    \begin{align}
        &\mathbf{Var}_{\ra^{i}\sim\pi^{i}_{\vtheta}}\left[ \rvg^{i}_{\text{MAPG}} \right] \nonumber\\
        &=
        \E_{\ra^{i}\sim\pi^{i}_{\vtheta}}\left[ \hat{Q}\left(\va^{-i}, \ra^{i}\right)^{2}\left(\nabla_{\psi_{\vtheta}^{i}}\log \pi^{i}_{\vtheta}(a^{i})\right)^{2} 
        \right] \nonumber\\
        &\ \ \ \ \ \ \ \ \ \ \ \ \ \ \ \ \ \ \ \ \ \ \ \ \ - \E_{\ra^{i}\sim\pi^{i}_{\vtheta}}\left[ \hat{Q}\left(\va^{-i}, \ra^{i}\right)\nabla_{\psi_{\vtheta}^{i}}\log \pi^{i}_{\vtheta}(\ra^{i}) \right]^{2}
        \nonumber\\
        &= \begin{bmatrix}
        640.192\\
        10.113\\
        810.033
        \end{bmatrix} - \begin{bmatrix}
        -7.76\\
        -1.07\\
        8.83\\
        \end{bmatrix}^{2}
        \ = \ \begin{bmatrix}
        640.192\\
        10.113\\
        810.033
        \end{bmatrix} - 
        \begin{bmatrix}
        60.2176\\
        1.1449\\
        77.9689
        \end{bmatrix}\ = \
        \begin{bmatrix}
        579.9744\\
        8.968\\
        732.064
        \end{bmatrix}
        \nonumber
    \end{align}
    So the variance of vanilla MAPG in this case is 
    \begin{align}
        \mathbf{Var}_{\ra^{i}\sim\pi^{i}_{\vtheta}}\left[ \rvg^{i}_{\text{MAPG}} \right] = 1321.007
        \nonumber
    \end{align}
    
    Let's now deal with COMA
    \begin{align}
        &\E_{\ra^{i}\sim\pi^{i}_{\vtheta}}\left[ \hat{A}^{i}
        \left(\va^{-i}, \ra^{i}\right)^{2}
        \left(\nabla_{\psi_{\vtheta}^{i}}\log \pi^{i}_{\vtheta}(\ra^{i})\right)^{2} 
        \right] \nonumber\\
        &= \sum_{a^{i}=1}^{3}\pi^{i}_{\vtheta}(a^{i})\hat{A}^{i}\left(\va^{-i}, a^{i}\right)^{2}
        \left(\nabla_{\psi^{i}_{\vtheta}}\log\pi^{i}_{\vtheta}(a^{i})\right)^{2}\nonumber\\
        &= \sum_{a^{i}=1}^{3}\pi^{i}_{\vtheta}(a^{i})
        \hat{A}^{i}\big(\va^{-i}, a^{i}\big)^{2}\left( \ve_{a^{i}} - \pi^{i}_{\vtheta} \right)^{2}\nonumber\\
        &= \ 0.8\times (-9.7)^{2} \times \begin{bmatrix}
        0.2\\
        -0.1\\
        -0.1
        \end{bmatrix}^{2} + 0.1\times (-10.7)^{2} \times \begin{bmatrix}
        -0.8\\
        0.9\\
        -0.1
        \end{bmatrix}^{2} +0.1\times 88.3^{2} \times \begin{bmatrix}
        -0.8\\
        -0.1\\
        0.9
        \end{bmatrix}^{2}\nonumber\\
        &= 0.8\times 94.09 \times \begin{bmatrix}
        0.04\\
        0.01\\
        0.01
        \end{bmatrix} + 0.1 \times 114.49 \times \begin{bmatrix}
        0.64\\
        0.81\\
        0.01
        \end{bmatrix} +0.1\times 7796.89 \times \begin{bmatrix}
        0.64\\
        0.01\\
        0.81
        \end{bmatrix}\nonumber\\
        &= \begin{bmatrix}
        3.011\\
        0.753\\
        0.753
        \end{bmatrix} +
        \begin{bmatrix}
        2.327\\
        9.274\\
        0.114
        \end{bmatrix} +
        \begin{bmatrix}
        499.001\\
        7.797\\
        631.548
        \end{bmatrix} \ = \ \begin{bmatrix}
        504.339\\
        17.824\\
        632.415
        \end{bmatrix}
        \nonumber
    \end{align}
    We have
    \begin{align}
        &\E_{\ra^{i}\sim\pi^{i}_{\vtheta}}\left[ \rvg^{i}_{\text{COMA}}\right] \nonumber\\
        =& \E_{\ra^{i}\sim\pi^{i}_{\vtheta}}
        \left[ \hat{A}^{i}\left(\va^{-i}, \ra^{i}\right)^{2}\left(\nabla_{\theta^{i}}\log \pi^{i}_{\vtheta}(\ra^{i})\right)^{2} 
        \right]  -
        \E_{\ra^{i}\sim\pi^{i}_{\vtheta}}\left[ \hat{A}^{i}\left(\va^{-i}, \ra^{i}\right)\nabla_{\theta^{i}}\log \pi^{i}_{\vtheta}(\ra^{i}) \right]^{2}\nonumber\\
        &= \begin{bmatrix}
        504.339\\
        17.824\\
        632.415
        \end{bmatrix} - \begin{bmatrix}
        -7.76\\
        -1.07\\
        8.83\\
        \end{bmatrix}^{2}
        \ = \ \begin{bmatrix}
        504.339\\
        17.824\\
        632.415
        \end{bmatrix} - 
        \begin{bmatrix}
        60.2176\\
        1.1449\\
        77.9689
        \end{bmatrix} \ = \
        \begin{bmatrix}
        444.1214\\
        16.6791\\
        554.4461
        \end{bmatrix}
        \nonumber
    \end{align}
    and we have
    \begin{align}
        \mathbf{Var}_{\ra^{i}\sim\pi^{i}_{\vtheta}}\left[ \rvg^{i}_{\text{COMA}} \right] = 1015.2466
        \nonumber
    \end{align}
    Lastly, we figure out OB
    \begin{align}
        &\E_{\ra^{i}\sim\pi^{i}_{\vtheta}}\left[ \hat{X}^{i}\left(\va^{-i}, \ra^{i}\right)^{2}
        \left(\nabla_{\psi^{i}_{\vtheta}}\log \pi^{i}_{\vtheta}(\ra^{i})\right)^{2} \right] \nonumber\\
        &= \sum_{a^{i}=1}^{3}\pi^{i}_{\vtheta}(a^{i})\hat{X}^{i}\left(\va^{-i}, a^{i}\right)^{2}
        \left(\nabla_{\psi^{i}_{\vtheta}}\log\pi^{i}_{\vtheta}(a^{i})\right)^{2}\nonumber\\
        &= \ \sum_{a^{i}=1}^{3}\pi^{i}_{\vtheta}(a^{i})\hat{X}^{i}\left(\va^{-i}, a^{i}\right)^{2}\left( \ve_{a^{i}} - \pi^{i}_{\vtheta} \right)^{2}
        \nonumber\\
        &= \ 0.8\times (-41.71)^{2} \times \begin{bmatrix}
        0.2\\
        -0.1\\
        -0.1
        \end{bmatrix}^{2} + 0.1\times (-42.71)^{2} \times \begin{bmatrix}
        -0.8\\
        0.9\\
        -0.1
        \end{bmatrix}^{2} +0.1\times 56.29^{2} \times \begin{bmatrix}
        -0.8\\
        -0.1\\
        0.9
        \end{bmatrix}^{2}\nonumber\\
        &= 0.8\times 1739.724 \times \begin{bmatrix}
        0.04\\
        0.01\\
        0.01
        \end{bmatrix} + 0.1 \times 1824.144 \times \begin{bmatrix}
        0.64\\
        0.81\\
        0.01
        \end{bmatrix} +0.1\times 3168.564 \times \begin{bmatrix}
        0.64\\
        0.01\\
        0.81
        \end{bmatrix}\nonumber\\
        &= \begin{bmatrix}
        55.6712\\
        13.92\\
        13.92
        \end{bmatrix} +
        \begin{bmatrix}
        116.7452\\
        147.756\\
        1.824
        \end{bmatrix} +
        \begin{bmatrix}
        202.788\\
        3.169\\
        256.654
        \end{bmatrix} \ = \ \begin{bmatrix}
        375.2044\\
        164.845\\
        272.398
        \end{bmatrix}
        \nonumber
    \end{align}
    We have
    \begin{align}
        &\E_{\ra^{i}\sim\pi^{i}_{\vtheta}}\left[ \rvg^{i}_{X} \right]\nonumber\\
        &= \E_{\ra^{i}\sim\pi^{i}_{\vtheta}}\left[ \hat{X}^{i}
        \left(\va^{-i},\ra^{i}\right)^{2}
        \left(\nabla_{\psi_{\vtheta}^{i}}\log \pi^{i}_{\vtheta}(\ra^{i})\right)^{2} \right] 
        -
        \E_{\ra^{i}\sim\pi^{i}_{\vtheta}}
        \left[ \hat{X}^{i}\left(\va^{-i}, \ra^{i}\right)\nabla_{\psi_{\vtheta}^{i}}\log \pi^{i}_{\theta}(\ra^{i}) \right]^{2}
        \nonumber\\
        &= \begin{bmatrix}
        375.2044\\
        164.845\\
        272.398
        \end{bmatrix} - \begin{bmatrix}
        -7.76\\
        -1.07\\
        8.83\\
        \end{bmatrix}^{2}
        \ = \ \begin{bmatrix}
         375.2044\\
        164.845\\
        272.398
        \end{bmatrix} - 
        \begin{bmatrix}
        60.2176\\
        1.1449\\
        77.9689
        \end{bmatrix}
        \ = \
        \begin{bmatrix}
        314.987\\
        163.7\\
        194.429
        \end{bmatrix}
        \nonumber
    \end{align}
    and we have
    \begin{align}
        \mathbf{Var}_{\ra^{i}\sim\pi^{i}_{\vtheta}}
        \left[ \rvg^{i}_{X} \right] = 673.116
        \nonumber
    \end{align}
 
\end{proof}

\section{Detailed Hyper-parameter Settings for Experiments}
\label{section:details-experiments}
In this section, we include the details of our experiments. Their implementations can be found in the following codebase:

\begin{tcolorbox}[colback=blue!15, colframe=blue!75]
    \begin{center}
    {\url{https://github.com/morning9393/Optimal-Baseline-for-Multi-agent-Policy-Gradients}.}
    \end{center}
\end{tcolorbox}

In COMA experiments, we use the official implementation in their codebase \cite{foerster2018counterfactual}. The only difference between COMA with and without OB is the baseline introduced, that is, the OB or the counterfactual baseline of COMA \cite{foerster2018counterfactual}.

Hyper-parameters used for COMA in the SMAC domain.
\begin{table}[h!]
    \begin{adjustbox}{width=0.7\columnwidth, center}
        \begin{tabular}{c | c c c} 
            \toprule
            Hyper-parameters & 3m & 8m & 2s3z\\
            \midrule
            actor lr & 5e-3 & 1e-2 & 1e-2\\ 
            critic lr & 5e-4 & 5e-4 & 5e-4\\
            gamma & 0.99 & 0.99 & 0.99\\
            epsilon start & 0.5 & 0.5 & 0.5\\
            epsilon finish & 0.01 & 0.01 & 0.01\\
            epsilon anneal time & 50000 & 50000 & 50000\\
            batch size & 8 & 8 & 8\\
            buffer size & 8 & 8 & 8\\
            target update interval & 200 & 200 & 200\\
            optimizer & RMSProp & RMSProp & RMSProp\\
            optim alpha & 0.99 & 0.99 & 0.99\\
            optim eps & 1e-5 & 1e-5 & 1e-5\\
            grad norm clip & 10 & 10 & 10\\
            actor network & rnn & rnn & rnn\\
            rnn hidden dim & 64 & 64 & 64\\
            critic hidden layer & 1 & 1 & 1\\
            critic hiddem dim & 128 & 128 & 128\\
            activation & ReLU & ReLU & ReLU\\
            eval episodes & 32 & 32 & 32\\
            \bottomrule
        \end{tabular}
    \end{adjustbox}
\end{table}
\clearpage
As for Multi-agent PPO, based on the official implementation \cite{mappo}, the original V-based critic is replaced by Q-based critic for OB calculation. Simultaneously, we have not used V-based tricks like the GAE estimator, when either using OB or state value as baselines, for fair comparisons.

We provide the pseudocode of our implementation of Multi-agent PPO with OB. We highlight the novel components of it (those unpresent, for example, in \cite{mappo}) in \textcolor{teal}{colour}.
\begin{algorithm}
  \caption{Multi-agent PPO with Q-critic and OB}
  \label{alg1}
  \begin{algorithmic}[1]
  \STATE Initialize $\theta$ and $\phi$, the parameters for actor $\pi$ \textcolor{teal}{and critic $Q$}
  \STATE $episode_{max} \gets step_{max}/batch\_size$
  \WHILE{$episode \leq episode_{max}$}
  \STATE Set data buffer $\bm{D} = \{ \}$
  \STATE Get initial states $s_{0}$ and observations $\bm{o_{0}}$
  \FOR{$t=0$ \TO $batch\_size$}
  \FORALL{agents $i$}
  \IF{discrete action space}
  \STATE $a_{t}^{i}, p_{\pi,t}^{i} \gets \pi(o_{t}^{i};\theta)$  // where $p_{\pi,t}^{i}$ is the probability distribution of available actions 
  \ELSIF{continuous action space}
  \STATE $a_{t}^{i}, p_{a, t}^{i} \gets \pi(o_{t}^{i};\theta)$  // where $p_{a, t}^{i}$ is the probability density of action $a_{t}^{i}$
  \ENDIF
  \STATE \textcolor{teal}{$q_{t}^{i} \gets Q(s_{t}, i, a_{t}^{i};\phi)$}
  \ENDFOR
  \STATE $s_{t+1}, o_{t+1}^{n}, r_{t} \gets$ execute $\{a_{t}^{1}...a_{t}^{n}\}$
  \IF{discrete action space}
  \STATE \textcolor{teal}{Append $[s_{t}, \bm{o_{t}}, \bm{a_{t}}, r_{t}, s_{t+1}, \bm{o_{t+1}}, \bm{q_{t}}, \bm{p_{\pi,t}}]$ to $\bm{D}$}
  \ELSIF{continuous action space}
  \STATE \textcolor{teal}{Append $[s_{t}, \bm{o_{t}}, \bm{a_{t}}, r_{t}, s_{t+1}, \bm{o_{t+1}}, \bm{q_{t}}, \bm{p_{a,t}}]$ to $\bm{D}$}
  \ENDIF
  \ENDFOR
  
  // from now all agents are processed in parallel in $D$
  \IF{discrete action space}
  \STATE \textcolor{teal}{$\bm{ob} \gets \text{optimal\_baseline}(\bm{q}, \bm{p_{\pi}})$} // use data from $D$
  \ELSIF{continuous action space}
  \STATE \textcolor{teal}{Resample $\bm{a_{t, 1...m}}, \bm{q_{t, 1...m}} \sim \bm{\mu_{t}}, \bm{\sigma_{t}}$ for each $s_{t}, \bm{o_{t}}$}
  \STATE \textcolor{teal}{$\bm{ob} \gets \text{optimal\_baseline}(\bm{a}, \bm{q}, \bm{\mu}, \bm{\sigma})$} // use resampled data
  \ENDIF
  \STATE \textcolor{teal}{$\bm{X} \gets \bm{q} - \bm{ob}$}
  \STATE \textcolor{teal}{$\text{Loss}(\theta) \gets - \text{mean}(\bm{X} \cdot \log \bm{p_{a}})$}
  \STATE Update $\theta$ with Adam/RMSProp to minimise $\text{Loss}(\theta)$
  \ENDWHILE
  \end{algorithmic}
  The critic parameter $\phi$ is trained with TD-learning \cite{sutton2018reinforcement}.
\end{algorithm}
\clearpage

Hyper-parameters used for Multi-agent PPO in the SMAC domain.
\begin{table}[h!]
    \begin{adjustbox}{width=0.6\columnwidth, center}
        \begin{tabular}{c | c } 
            \toprule
            Hyper-parameters & 3s vs 5z / 5m vs 6m / 6h vs 8z / 27m vs 30m\\
            \midrule
            actor lr & 1e-3\\ 
            critic lr & 5e-4\\
            gamma & 0.99\\
            batch size & 3200\\
            num mini batch & 1\\
            ppo epoch & 10\\
            ppo clip param & 0.2\\
            entropy coef & 0.01\\
            optimizer & Adam\\
            opti eps & 1e-5\\
            max grad norm & 10\\
            actor network & mlp\\
            hidden layper & 1\\
            hidden layer dim & 64\\
            activation & ReLU\\
            gain & 0.01\\
            eval episodes & 32\\
            use huber loss & True\\
            rollout threads & 32\\
            episode length & 100\\
            \bottomrule
        \end{tabular}
    \end{adjustbox}
\end{table}

Hyper-parameters used for Multi-agent PPO in the Multi-Agent MuJoCo domain.
\begin{table}[h!]
    \begin{adjustbox}{width=0.8\columnwidth, center}
        \begin{tabular}{c | c c c c} 
            \toprule
            Hyper-parameters & Hopper(3x1) & Swimmer(2x1) & HalfCheetah(6x1) & Walker(2x3)\\
            \midrule
            actor lr & 5e-6 & 5e-5 & 5e-6 & 1e-5\\ 
            critic lr & 5e-3 & 5e-3 & 5e-3 & 5e-3\\
            lr decay & 1 & 1 & 0.99 & 1\\
            episode limit & 1000 & 1000 & 1000 & 1000\\
            std x coef & 1 & 10 & 5 & 5\\
            std y coef & 0.5 & 0.45 & 0.5 & 0.5\\
            ob n actions & 1000 & 1000 & 1000 & 1000\\
            gamma & 0.99 & 0.99 & 0.99 & 0.99\\
            batch size & 4000 & 4000 & 4000 & 4000\\
            num mini batch & 40 & 40 & 40 & 40\\
            ppo epoch & 5 & 5 & 5 & 5\\
            ppo clip param & 0.2 & 0.2 & 0.2 & 0.2\\
            entropy coef & 0.001 & 0.001 & 0.001 & 0.001\\
            optimizer & RMSProp & RMSProp & RMSProp & RMSProp\\
            momentum & 0.9 & 0.9 & 0.9 & 0.9\\
            opti eps & 1e-5 & 1e-5 & 1e-5 & 1e-5\\
            max grad norm & 0.5 & 0.5 & 0.5 & 0.5\\
            actor network & mlp & mlp & mlp & mlp\\
            hidden layper & 2 & 2 & 2 & 2\\
            hidden layer dim & 32 & 32 & 32 & 32\\
            activation & ReLU & ReLU & ReLU & ReLU\\
            gain & 0.01 & 0.01 & 0.01 & 0.01\\
            eval episodes & 10 & 10 & 10 & 10\\
            use huber loss & True & True & True & True\\
            rollout threads & 4 & 4 & 4 & 4\\
            episode length & 1000 & 1000 & 1000 & 1000\\
            \bottomrule
        \end{tabular}
    \end{adjustbox}
\end{table}

For QMIX and COMIX baseline algorithms, we use implementation from their official codebases and keep the performance consistent with the results reported in their original papers \cite{peng2020facmac, rashid2018qmix}. MADDPG is provided along with COMIX, which is derived from its official implementation as well \cite{maddpg}. 

\clearpage
Hyper-parameters used for QMIX baseline in the SMAC domain.
\begin{table}[h!]
    \begin{adjustbox}{center}
        \begin{tabular}{c | c } 
            \toprule
            Hyper-parameters & 3s vs 5z / 5m vs 6m / 6h vs 8z / 27m vs 30m\\
            \midrule
            critic lr & 5e-4\\
            gamma & 0.99\\
            epsilon start & 1\\
            epsilon finish & 0.05\\
            epsilon anneal time & 50000\\
            batch size & 32\\
            buffer size & 5000\\
            target update interval & 200\\
            double q & True\\
            optimizer & RMSProp\\
            optim alpha & 0.99\\
            optim eps & 1e-5\\
            grad norm clip & 10\\
            mixing embed dim & 32\\
            hypernet layers & 2\\
            hypernet embed & 64\\
            critic hidden layer & 1\\
            critic hiddem dim & 128\\
            activation & ReLU\\
            eval episodes & 32\\
            \bottomrule
        \end{tabular}
    \end{adjustbox}
\end{table}

Hyper-parameters used for COMIX baseline in the Multi-Agent MuJoCo domain.
\begin{table}[h!]
    \begin{adjustbox}{center}
        \begin{tabular}{c | c } 
            \toprule
            Hyper-parameters & Hopper(3x1) / Swimmer(2x1) / HalfCheetah(6x1) / Walker(2x3)\\
            \midrule
            critic lr & 0.001\\
            gamma & 0.99\\
            episode limit & 1000\\
            exploration mode & Gaussian\\
            start steps & 10000\\
            act noise & 0.1\\
            batch size & 100\\
            buffer size & 1e6\\
            soft target update & True\\
            target update tau & 0.001\\
            optimizer & Adam\\
            optim eps & 0.01\\
            grad norm clip & 0.5\\
            mixing embed dim & 64\\
            hypernet layers & 2\\
            hypernet embed & 64\\
            critic hidden layer & 2\\
            critic hiddem dim & [400, 300]\\
            activation & ReLU\\
            eval episodes & 10\\
            \bottomrule
        \end{tabular}
    \end{adjustbox}
\end{table}

\clearpage
Hyper-parameters used for MADDPG baseline in the Multi-Agent MuJoCo domain.
\begin{table}[h!]
    \begin{adjustbox}{center}
        \begin{tabular}{c | c } 
            \toprule
            Hyper-parameters & Hopper(3x1) / Swimmer(2x1) / HalfCheetah(6x1) / Walker(2x3)\\
            \midrule
            actor lr & 0.001\\
            critic lr & 0.001\\
            gamma & 0.99\\
            episode limit & 1000\\
            exploration mode & Gaussian\\
            start steps & 10000\\
            act noise & 0.1\\
            batch size & 100\\
            buffer size & 1e6\\
            soft target update & True\\
            target update tau & 0.001\\
            optimizer & Adam\\
            optim eps & 0.01\\
            grad norm clip & 0.5\\
            mixing embed dim & 64\\
            hypernet layers & 2\\
            hypernet embed & 64\\
            actor network & mlp\\
            hidden layer & 2\\
            hiddem dim & [400, 300]\\
            activation & ReLU\\
            eval episodes & 10\\
            \bottomrule
        \end{tabular}
    \end{adjustbox}
\end{table}

\end{document}

%% file: input_math.tex
\usepackage{amsmath,amsfonts,bm}

\def\eqref#1{equation~\ref{#1}}
\def\1{\bm{1}}

\def\ra{{\textnormal{a}}}

\def\rg{{\textnormal{g}}}

\def\rr{{\textnormal{r}}}
\def\rs{{\textnormal{s}}}

\def\rva{{\mathbf{a}}}

\def\rvg{{\mathbf{g}}}

\def\vtheta{{\bm{\theta}}}
\def\va{{\bm{a}}}

\def\ve{{\bm{e}}}

\def\vz{{\bm{z}}}

\def\vtheta{{\boldsymbol{\theta}}}

\DeclareMathAlphabet{\mathsfit}{\encodingdefault}{\sfdefault}{m}{sl}
\SetMathAlphabet{\mathsfit}{bold}{\encodingdefault}{\sfdefault}{bx}{n}

\newcommand{\E}{\mathbb{E}}

\newcommand{\softmax}{\mathrm{softmax}}